%% file: main.tex
\documentclass{article}

\usepackage[preprint]{arxiv}

\usepackage[utf8]{inputenc} 
\usepackage[T1]{fontenc} 

\usepackage{microtype}
\usepackage{graphicx}
\usepackage{subcaption}
\usepackage{booktabs} 
\usepackage{hyperref} 
\usepackage{url} 
\usepackage{booktabs} 
\usepackage{nicefrac} 
\usepackage{microtype} 
\usepackage{verbatim}
\usepackage{caption}
\usepackage{tikz}

\usepackage{hyperref}     
\hypersetup{
    colorlinks=true,      
    linkcolor=violet,       
    citecolor=BlueViolet,        
    urlcolor=violet         
}

\usepackage{listings}

\usepackage{amsmath, amsfonts, amssymb, amsthm}
\usepackage{mathtools}
\usepackage{thmtools, thm-restate}
\input{math_commands}

\usepackage{natbib}
\setcitestyle{authoryear, open={(},close={)}}
\usepackage[capitalize, noabbrev]{cleveref}

\usepackage{algorithm}
\usepackage{algpseudocode}

\usepackage{enumitem}
\setlist[itemize]{itemsep = 5pt, topsep=0pt, parsep=0pt, partopsep=0pt}
\setlist[enumerate]{itemsep = 5pt, topsep=0pt, parsep=0pt, partopsep=0pt}

\usepackage{xcolor}
\usepackage[dvipsnames]{xcolor} 

\theoremstyle{plain}
\newtheorem{theorem}{Theorem}[section]
\newtheorem{proposition}[theorem]{Proposition}
\newtheorem{lemma}[theorem]{Lemma}
\newtheorem{corollary}[theorem]{Corollary}
\theoremstyle{definition}
\newtheorem{definition}[theorem]{Definition}

\theoremstyle{definition}
\newtheorem{remark}[theorem]{Remark}
\newtheorem{example}[theorem]{Example}

\title{Robustness Verification of Polynomial Neural Networks}

\addauthor{Yulia Alexandr}{1,\dag}{UCLA}{yulia@math.ucla.edu}
\addauthor{Hao Duan}{2,\dag}{UCLA}{hduan7@ucla.edu}
\addauthor{Guido Mont\'ufar}{1,2,3}{UCLA \& MPI MiS}{montufar@math.ucla.edu}

\firstpagefootnotes{%
\textsuperscript{\dag}\,Equal contribution.\\[0.3ex] \textsuperscript{1} \, Department of Mathematics, University of California, Los Angeles, CA 90095, USA \\[0.1ex] \textsuperscript{2} \, Department of Statistics \& Data Science, University of California, Los Angeles, CA 90095, USA \\[0.1ex] \textsuperscript{3} \, Max Planck Institute for Mathematics in the Sciences, 04103 Leipzig, Germany}

\begin{document}

\maketitle

\begin{abstract}
  We study robustness verification of neural networks via metric algebraic geometry. 
  For polynomial neural networks, certifying a robustness radius amounts to computing the distance to the algebraic decision boundary. 
  We use the Euclidean distance (ED) degree as an intrinsic measure of the complexity of this problem, analyze the associated ED discriminant, and introduce a parameter discriminant that detects parameter values at which the ED degree drops. 
  We derive formulas for the ED degree for several network architectures and characterize the expected number of real critical points in the infinite-width limit. 
  We develop symbolic elimination methods to compute these quantities and homotopy-continuation methods for exact robustness certification. 
  Finally, experiments on lightning self-attention modules reveal decision boundaries with strictly smaller ED degree than generic cubic hypersurfaces of the same ambient dimension. 

  \smallskip 

  \smallskip 

  \emph{MSC classes:} 14Q65, 65H14, 14Q30, 90C23, 60G15, 68T07
\end{abstract}

\section{Introduction}
  The deployment of artificial neural networks in safety-critical settings has created a growing demand for formal robustness guarantees, ensuring that small perturbations of the input do not lead to incorrect or unsafe predictions. 
  Robustness verification, i.e., certifying that the output prediction remains invariant over a prescribed neighborhood of a given input, has therefore emerged as a central problem in modern applications. 
  In this paper, we study this problem for classifiers implemented as polynomial neural networks and, more generally, for classifiers whose decision boundaries are algebraic or semi-algebraic. 
  We aim to advance theoretical frameworks to better understand how robustness verification is influenced by
  (i) the architecture of the neural network; 
  (ii) the particular test point; and 
  (iii) the particular parameter of the network. 

  A wide range of verification methods have been proposed, including convex relaxations
  \citep{salman_2019}, mixed-integer programming \citep{tjeng_2019}, satisfiability
  modulo theories \citep{katz_2017}, and abstract interpretation \citep{gehr_2018}.
  While these approaches have achieved notable success in various settings \citep{weng_2018, zhang_2018, botoeva_2020, brix_2023}, they exhibit a fundamental trade-off between computational tractability and exactness: methods that provide exact guarantees may incur exponential complexity or depend strongly on architectural features \citep{meng_2022, liu_2024, froese_2025}, whereas scalable relaxations enable efficient computation but may produce loose certificates and fail on challenging instances \citep{singh_2018, wong_2018, salman_2019}. 
  There remains limited theoretical understanding of how the intrinsic structure of a neural network governs the complexity of robustness verification \citep{ferlez_2021, konig_2024}. 

  We propose an algebraic framework for quantifying the complexity of robustness verification focusing on polynomial models and the intrinsic properties of the decision boundary. 
  The decision boundary of a classifier consists of input points at which the predicted class changes, which are points satisfying equations of the form $f_{i}(\bx ) - f_{j}(\bx) = 0$.
  Robustness verification at a test point $\bxi$ then amounts to certifying that its distance to the decision boundary is bounded below by some $\epsilon$ (see \cref{fig:main}). 
  For polynomial networks, this decision boundary is an algebraic variety. 
  This allows us to connect robustness verification to metric algebraic geometry, where the complexity of the associated distance minimization problem is captured by algebraic invariants of the underlying variety. 
  
  \begin{figure}
    \begin{center}
      \centerline{\includegraphics[width=0.5\columnwidth]{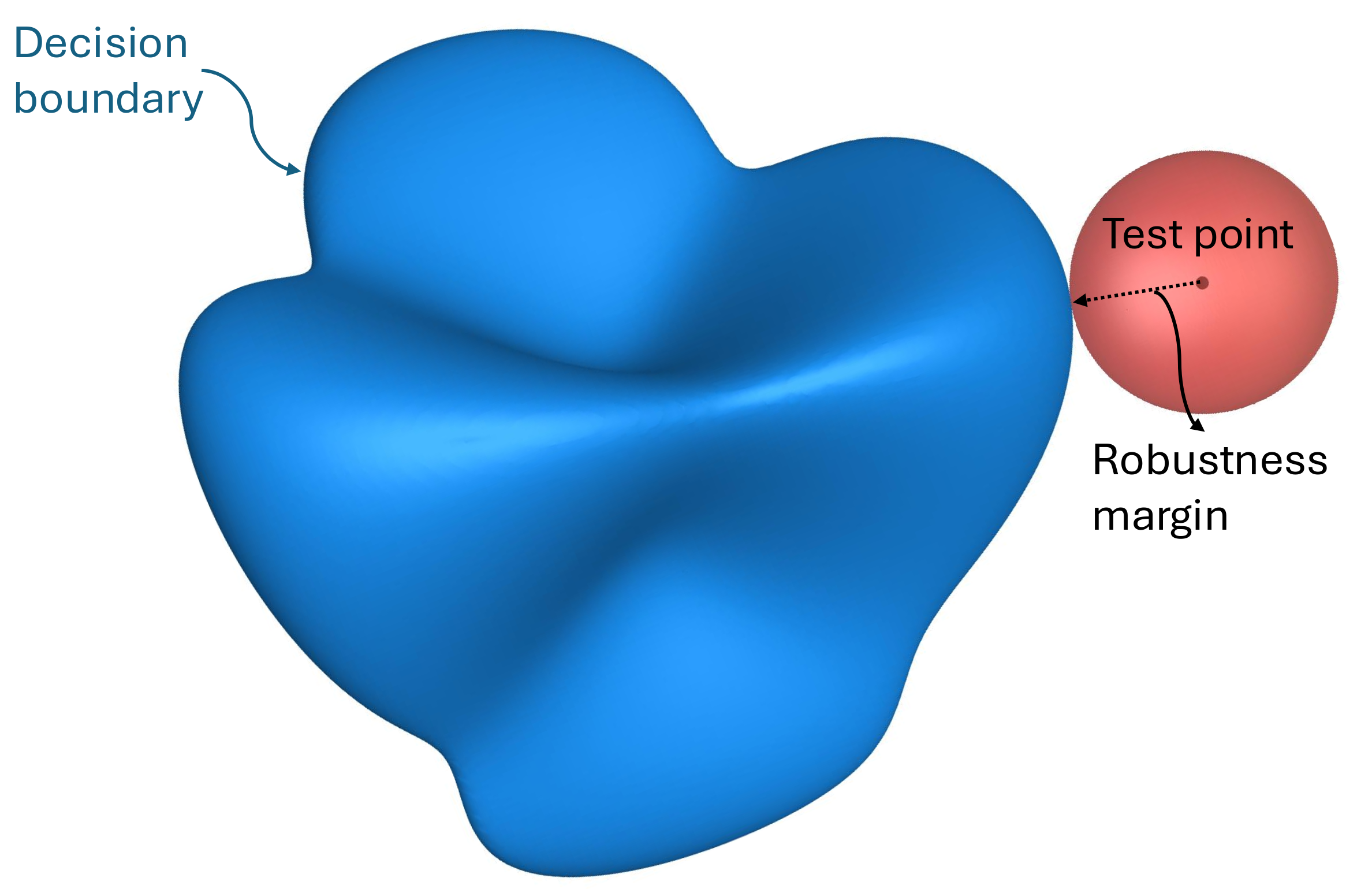}}
      \caption{The blue hypersurface shows the decision boundary of a polynomial
      neural network.
      The black dot is a test input; we compute its projection onto this decision boundary to obtain the robustness margin.} 
      \label{fig:main}
    \end{center}
    \vspace{-2em}
  \end{figure}
  
  Our central complexity measure is the Euclidean Distance (ED) degree, a fundamental invariant from algebraic geometry, defined as the generic number of complex critical points of the distance function to an algebraic variety \citep{draisma_2016,maxim_2020,breiding_2024a,lai_2025}. 
  In the context of robustness verification, the ED degree provides an architecture-dependent measure of verification complexity that is independent of optimization heuristics. 
  We further study the associated ED discriminant and introduce a parameter discriminant. 
  The \textit{ED discriminant} characterizes input points at which the number of \emph{real} critical points changes, while the \textit{parameter discriminant} identifies network parameters for which the ED degree drops, indicating a reduction in the algebraic complexity of the decision boundary. 

  We derive closed-form expressions for the ED degree and characterize the parameter discriminant for several classes of polynomial neural network architectures. 
  We further analyze the expected number of real critical points using the Kac-Rice framework. 
  For shallow networks in the infinite-width limit, we obtain explicit formulas and dimension-dependent bounds in terms of the degree of the activation function. 
  On the computational side, we develop symbolic and numerical algorithms for computing the ED degree and discriminants, and introduce an exact robustness certification algorithm based on numerical homotopy continuation. 
  This yields provably correct certificates whenever the underlying algebraic problem can be solved. 
  These results show that, by analyzing the algebraic invariants of the decision boundary, one can systematically quantify how verification complexity depends on the architecture, the network parameters, and the test point.
  
  \subsection{Contributions} 
  Our main contributions are summarized as follows:
  \begin{itemize}[leftmargin=*]
    \item In Section~\ref{sec:verification-as-algebraic-problem}, we formulate robustness verification as a metric projection problem onto algebraic varieties, providing an exact KKT-characterization and a relaxation. 
    We also give a numerical homotopy-continuation-based algorithm for exact robustness certification (Algorithm \ref{alg:hc}).
    
    \item In Section~\ref{sec:ed-degree}, we propose the ED degree as a measure of robustness verification complexity and describe its computation (Algorithm~\ref{alg:ed-degree}).

    \item In Section~\ref{sec:ed-degree-generic}, we derive closed-form expressions for the ED degree of decision boundaries of polynomial neural networks under generic parameters, for both wide and bottleneck architectures (Proposition~\ref{prop:generic-ed-degree-deep} and Proposition \ref{prop:bottleneck}), and complement these with numerical results for intermediate regimes.

    \item In Section~\ref{sec:exp-ed-degree}, we study the expected real ED degree of the decision boundary of polynomial networks via Kac-Rice methods. 
    We prove a general theorem for the expected number of real critical points (Theorem~\ref{thm:ed-kac-rice}). 
    For shallow networks in the infinite-width limit, we obtain an exact formula for one-dimensional input data (Proposition~\ref{prop:ed-kac-rice}) and an upper bound for arbitrary input dimensions (Theorem~\ref{thm:ed-kac-rice-high}). 

    \item In Section~\ref{sec:ed-disc}, we study the ED discriminant, which classifies test instances over the input space by the corresponding complexity of the verification task. 
    We provide a symbolic algorithm for its computation (Algorithm~\ref{alg:ed-discriminant}), together with examples. 
    
    \item In Section~\ref{sec:par-disc}, we introduce the parameter discriminant, which classifies the parameters of a neural network by the corresponding complexity of the verification task on generic test data.
    We establish a complete characterization for quadric boundaries (Theorem~\ref{thm:ed-degree-conics-n-variables}) and derive the parameter discriminant as a stratified hypersurface. 

    \item In Section~\ref{sec:attention}, we report numerical results on the ED degree and ED discriminant for lightning self-attention modules. 
  \end{itemize}
  
  An implementation of algorithms and experiments is provided at: \url{https://github.com/edwardduanhao/AlgebraicVerification}. 

  \subsection{Related Works}
    \paragraph{Neural network verification}
    Formal verification of neural networks has been extensively studied using satisfiability modulo theories, mixed-integer programming, and abstract interpretation \citep{katz_2017, tjeng_2019, anderson_2020, gehr_2018}.
    These approaches provide exact or conservative robustness guarantees for piecewise-linear networks, notably ReLU architectures, but typically face a tradeoff between scalability and precision \citep{meng_2022, liu_2024, froese_2025}.
    While highly effective in practice, existing methods largely treat robustness as an algorithmic problem and offer limited insight into how architectural or geometric properties of a network govern the intrinsic complexity of verification.

    \paragraph{Distance-based robustness}
    Robustness verification can be formulated as the problem of computing the minimum distance from an input to the classifier's decision boundary.
    This viewpoint underlies many adversarial robustness and certification methods based on local optimization \citep{szegedy_2014, carlini_2017} or convex relaxations \citep{wong_2018, salman_2019}.
    Prior work focuses on computing or bounding adversarial perturbations \citep{hein_2017,raghunathan_2018}, but does not characterize the global structure or number of critical points of the optimization~problem.

    \paragraph{Algebraic perspectives on neural networks}
    A growing body of work applies tools from algebraic geometry, tropical geometry, and related areas to the study of neural networks, addressing questions of expressivity \citep{kileel_2019, brandenburg_2024, kubjas_2024, alexandr_2025}, identifiability \citep{allman_2009, henry_2025, usevich_2025}, optimization
    landscapes \citep{venturi_2019, kohn_2022}. 
    These works model neural networks as polynomial or piecewise-polynomial maps and analyze the resulting algebraic varieties or stratified spaces.
    In contrast, our focus is on formal robustness verification, rather than representation or training dynamics. 

    \paragraph{Metric algebraic geometry}
    The Euclidean distance degree and the ED discriminant arise naturally in nearest-point problems on algebraic varieties and are central objects in metric algebraic geometry \citep{draisma_2016, maxim_2020, breiding_2024a, lai_2025}.
  This literature develops both the geometric theory and the computational methods
  for Euclidean distance optimization, including the study of critical points,
  discriminants, and related complexity measures \citep{breiding_kohn_sturmfels_2024}. 
  Our work applies this perspective to robustness verification for neural networks,
  where the relevant nearest-point problem is projection to the decision boundary.

\section{Verification as an Algebraic Problem} \label{sec:verification-as-algebraic-problem}

  \paragraph{Notation}
  We use $[n]$ to denote the index set $\{1,2,\ldots,n\}$.
  For a given $\bu = (u_{1},\ldots,u_{n})^{\top}\in \R^{n}$ and any $\bx = (x_{1}
  ,\ldots,x_{n})^{\top}\in \R^{n}$, we define the squared Euclidean distance
  \[
    d_{\bu}(\bx) := \sum_{i=1}^{n}(x_{i}- u_{i})^{2}.
  \]
  For a center $\bx_{0}\in \R^{n}$ and radius $r>0$, we define the open Euclidean
  ball as $$\mathbb{B}(\bx_{0},r) := \{ \bx \in \R^{n}\mid d_{\bx_0}(\bx) < r^{2}\}.$$

  \subsection{Robustness Verification via Metric Projection}
    Let $f_{\bt}: \R^{n}\to \R^{k}$ be a parametric model for $k$-class classification, with parameters $\bt \in \R^{p}$.
    For any input $\bx\in \R^{n}$, the model outputs a score vector $[f_{\bt,1}(\bx ), \ldots, f_{\bt,k}(\bx)]$, and predicts $\argmax_{i}f_{\bt,i}(\bx)$.

    The standard formulation of robustness verification \citep{wong_2018} is as follows: given a test point $\bxi\in \R^{n}$ assigned to class $c$, we seek to certify that the predicted class remains invariant under all $\ell_{2}$ perturbations of radius~$\epsilon$.
    Formally, the condition we want to verify is
    \[
      \argmax_{i\in[k]}f_{\bt, i}(\bx) = c, \quad \forall \bx \in \mathbb{B}(\bxi ,
      \epsilon).
    \]
    This condition can be reformulated as an optimization problem.
    For a fixed $c$, we define the worst-case margin~as
    \[
      \delta(\epsilon) := \min_{c' \neq c}\inf_{\bx \in \mathbb{B}(\bxi, \epsilon)} \{ f_{\bt, c}(\bx) - f_{\bt, c'}(\bx) \}.
    \]
    If $\delta(\epsilon) > 0$, then the score of class $c$ exceeds all others in $\mathbb{B}(\bxi, \epsilon)$, thereby certifying robustness at $\bxi$.

    \paragraph{Geometric interpretation}
    We interpret robustness verification as computing a metric projection of the test point onto the classifier's decision boundary.
    Define the \emph{active decision boundary} between classes $c$ and $c'$ as
    \[
      \B_{c,c'}^{\bt}:= \{ \bx \in \R^{n}: f_{\bt, c}(\bx) = f_{\bt, c'}(\bx) = \max
      _{i \in [k]}f_{\bt, i}(\bx) \},
    \]
    and denote the zero level set of logit difference as
    \[
      \V_{c,c'}^{\bt}:= \{ \bx \in \R^{n}: f_{\bt, c}(\bx) - f_{\bt, c'}(\bx) = 0\}.
    \]
    Note that $\B_{c,c'}^{\bt}\subseteq \mathcal{V}_{c,c'}^{\bt}$, since the decision boundary requires extra constraints
    $f_{\bt, i}(\bx) - f_{\bt, c}(\bx) \le 0$ for all $i \neq c, c'$.
    We quantify the robustness margin as the minimum distance from the test point $\bxi$ to the active decision boundary by defining
    \begin{align}
      \gamma := \min_{c' \neq c}\inf_{\bx \in \B_{c,c'}^{\bt}}\sqrt{d_{\bxi}(\bx)}. \label{eq:robustness-margin}
    \end{align}
    The condition $\gamma > \epsilon$ implies that the ball $\mathbb{B}(\bxi, \epsilon)$ does not intersect any active decision boundary, i.e.,
    \[
      \mathbb{B}(\bxi, \epsilon) \cap \bigcup_{c' \neq c}\B_{c,c'}^{\bt}= \varnothing.
    \]
    Equivalently, for every $\bx \in \mathbb{B}(\bxi, \epsilon)$ and every $c' \neq c$, we have $f_{\bt,c}(\bx) > f_{\bt,c'}(\bx)$, so the pairwise margins remain nonnegative everywhere in the ball, and $\delta(\epsilon) > 0$. 
    By continuity of the logits, the predicted class then remains invariant within $\mathbb{B}(\bxi, \epsilon)$, thereby certifying robustness at $\bxi$.

  \subsection{Algebraic Optimization Formulation}
    Suppose each logit $f_{\bt,i}$ is a polynomial in $\bx$, so $\V_{c,c'}^{\bt}$ is a \emph{real algebraic variety} defined by the polynomial equation $f_{\bt,c}(\bx)-f_{\bt,c'}(\bx)=0$.
    Imposing the additional inequality constraints $f_{\bt,i}(\bx)-f_{\bt,c}(\bx) \le 0$ for all $i\neq c,c'$ restricts this to a \emph{semialgebraic} set $\B_{c,c'}^{\bt}$.

    To compute the robustness margin $\gamma$, we formulate a constrained metric projection problem for each $c' \neq c$:
    \begin{align}
      \text{minimize}\quad    & d_{\bxi}(\bx) \nonumber                                                       \\
      \text{subject to}\quad  & f_{\bt, c'}(\bx) - f_{\bt, c}(\bx) = 0, \label{eq:main-opt-prob}              \\
                              & f_{\bt, i}(\bx) - f_{\bt, c}(\bx) \le 0, \quad \forall i \neq c,c'. \nonumber
    \end{align}
    This computes the projection of the test point $\bxi$ onto the active decision boundary between classes $c$ and $c'$.

    \paragraph{KKT conditions}
    To characterize candidate solutions of Problem \ref{eq:main-opt-prob}, we introduce the Lagrangian
    \begin{equation*}
      \begin{aligned}
        \mathcal{L}(\bx, \lambda, \boldsymbol{\mu}) :={} & \frac{1}{2}\, d_{\bxi}(\bx) + \lambda \bigl(f_{\bt, c'}(\bx) - f_{\bt, c}(\bx)\bigr) + \sum_{i \neq c, c'}\mu_{i}\bigl(f_{\bt, i}(\bx) - f_{\bt, c}(\bx)\bigr),
      \end{aligned}
    \end{equation*}
    where $\lambda$ and $\boldsymbol{\mu}= (\mu_{i})_{i \neq c, c'}$ are the KKT multipliers corresponding to the equality and inequality constraints.

    By the KKT conditions \citep{karush_1939, kuhn_2014}, any local optimum $(\bx^{\ast},\lambda^{\ast},\boldsymbol{\mu}^{\ast})$ satisfies
    \begin{equation}
      \left\{
      \begin{aligned}
        \nabla_{\bx}\mathcal{L}(\bx^{\ast},\lambda^{\ast},\boldsymbol{\mu}^{\ast}) & = 0,                              \\
        f_{\bt, c'}(\bx^{\ast}) - f_{\bt, c}(\bx^{\ast})                           & = 0,                              \\
        \mu_{i}^{\ast}\bigl(f_{\bt, i}(\bx^{\ast}) - f_{\bt, c}(\bx^{\ast})\bigr)  & = 0, \quad \forall i \neq c,c',   \\
        f_{\bt, i}(\bx^{\ast}) - f_{\bt, c}(\bx^{\ast})                            & \le 0, \quad \forall i \neq c,c', \\
        \mu_{i}^{\ast}                                                             & \ge 0, \quad \forall i \neq c,c'.
      \end{aligned}
      \right. \label{eq:kkt-conditions}
    \end{equation}
    The first three conditions in (\ref{eq:kkt-conditions}) form a square system of $n+k-1$ polynomial equations in $n+k-1$ unknowns, while the remaining conditions encode feasibility constraints. 
  
    Since all equations are polynomial, the KKT system can be solved using tools from numerical algebraic geometry. We consider specifically \textit{homotopy continuation} because it is a complete numerical method guaranteed to find all
    isolated complex solutions \citep{li_1997, sommese_2001, sommese_2005a, sommese_2005b, duff_2018}.
    This computes all complex solutions to the stationarity, complementary slackness, and equality constraints. 
    Subsequently, one retains the minimal real solution satisfying the remaining feasibility conditions.
    \cref{fig:hc} illustrates the concept, with details in \cref{app:homotopy}. 

    \begin{figure}[ht]
      \centering
      \begin{subfigure}
        [t]{0.38\textwidth}
        \centering
        \includegraphics[width=\linewidth]{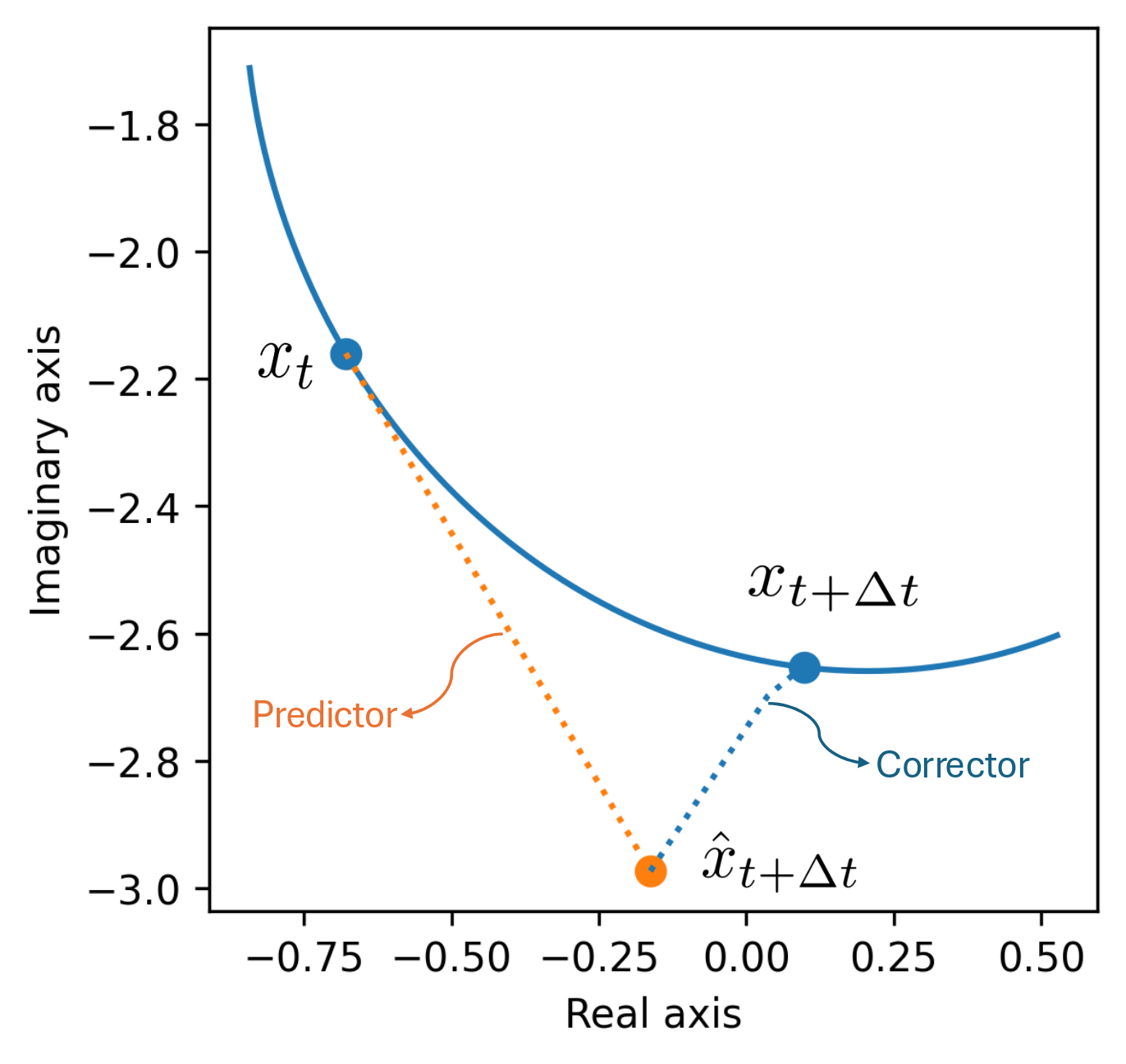}
        \label{fig:left}
      \end{subfigure}
      \hspace{0.1\textwidth}
      \begin{subfigure}
        [t]{0.374\textwidth}
        \centering
        \includegraphics[width=\linewidth]{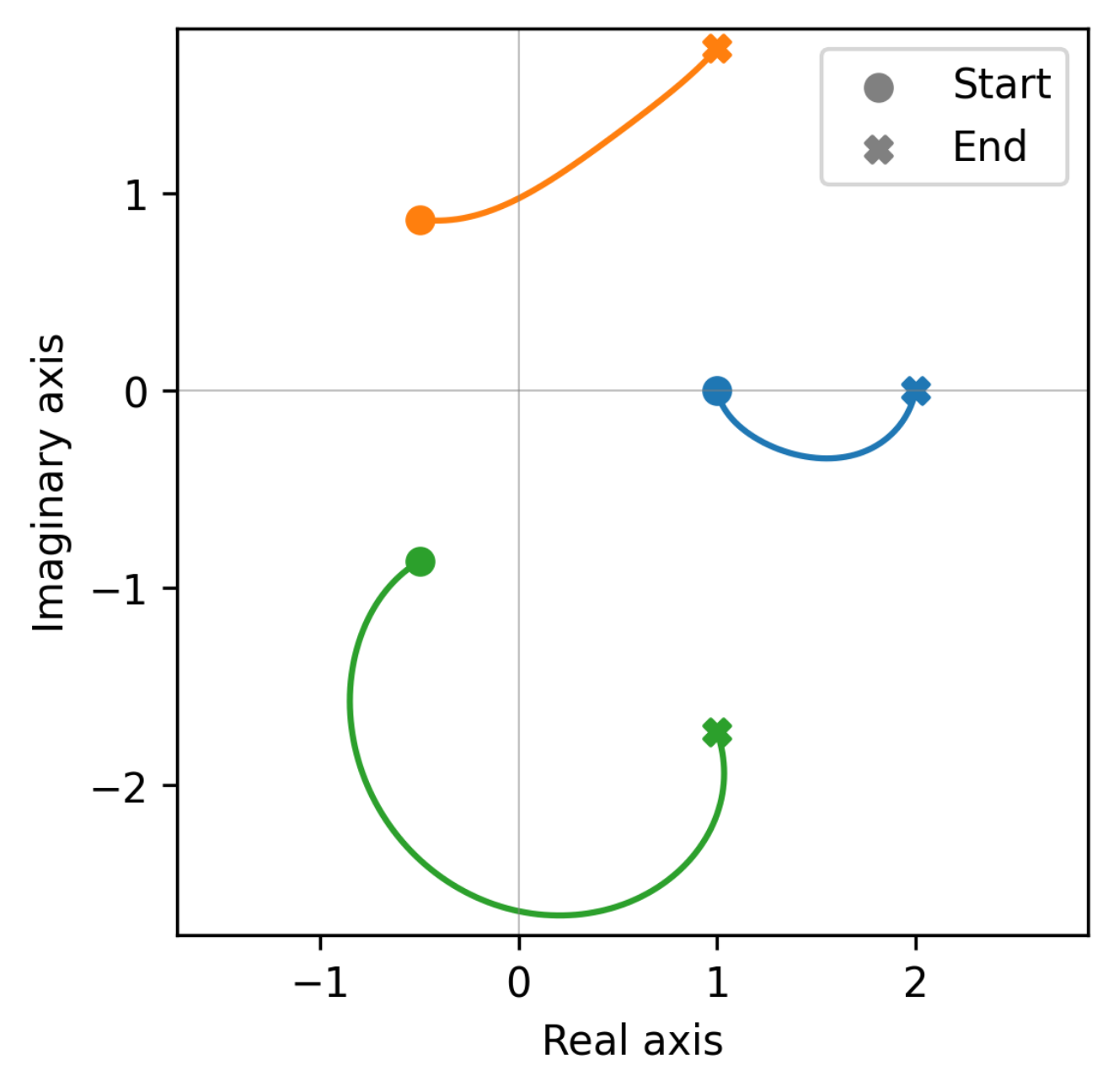}
        \label{fig:right}
      \end{subfigure}
      \caption{Left: An iteration of homotopy continuation, moving from $x_{t}$ to
      $x_{t+\Delta t}$ via a tangent predictor (orange) and a Newton corrector (blue).
      Right: Homotopy paths tracking the roots of $x^{3}- 4 x^{2}+ 8x - 8 = 0$ (with
      $x_{1,2}= 1 \pm \sqrt{3}\, i$ and $x_{3}= 2$) from the start system
      $x^{3}- 1 = 0$ (with $x_{1,2}= -\frac{1}{2}\pm \frac{\sqrt{3}}{2}\, i$ and $x
      _{3}= 1$).}
      \label{fig:hc}
    \end{figure}
    
    Solving the full KKT system can be of interest in some cases, but is computationally prohibitive in practice.
    The B\'ezout upper bound on the number of complex solutions \citep{bezout_1779} grows exponentially in the number of variables.

    \paragraph{Relaxation of inequality constraints}
    To mitigate this complexity, we consider a relaxation that removes the inequality constraints, resulting in a projection problem onto $\V_{c,c'}^{\bt}$ instead of $\B_{c,c'}^{\bt}$.
    Define the relaxed robustness margin
    \begin{align}
      \tilde{\gamma}:= \min_{c' \neq c}\inf_{\bx \in \V_{c,c'}^{\bt}}\sqrt{d_{\bxi}(\bx)}. 
      \label{eq:robustness-margin-relaxed}
    \end{align}
    This relaxation is exact under mild conditions:
    \begin{restatable}
      {proposition}{proprelax} \label{prop:relaxation-equivalence} Let
      $f_{\bt}: \mathbb{R}^{n}\to \mathbb{R}^{k}$ be a continuous classifier, and let
      $\hat{c}(\bx) := \arg\max_{i \in [k]}f_{\bt, i}(\bx)$ denote the predicted class.
      Fix a test point $\bxi \in \mathbb{R}^{n}$ and write $c := \hat{c}(\bxi)$.
      Assume that the prediction at $\bxi$ is unique, i.e.,
      \[
        f_{\bt,c}(\bxi) > f_{\bt,c'}(\bxi) \quad \text{for all }c' \neq c.
      \]
      Then, for $\gamma$ and $\tilde{\gamma}$ defined in (\ref{eq:robustness-margin})
      and (\ref{eq:robustness-margin-relaxed}), we have $\gamma = \tilde{\gamma}$.
    \end{restatable}
    The proof is provided in \cref{proof:relaxation-equivalence}.
    As a consequence, we may consider the relaxed optimization problem
    \begin{align}
      \begin{split}\text{minimize}\quad&d_{\bxi}(\bx) \\ \text{subject to}\quad&f_{\bt, c'}(\bx) - f_{\bt, c}(\bx) = 0. \label{eq:relaxed-opt-prob}\end{split}
    \end{align}
    The corresponding Lagrangian is
    \[
      \mathcal{L}(\bx, \lambda) = \frac{1}{2}d_{\bxi}(\bx) + \lambda \bigl(f_{\bt, c'}(\bx) - f_{\bt, c}(\bx)\bigr),
    \]
    and the KKT conditions reduce to
    \begin{align}
      \begin{split}\bx^{\ast}- \bxi + \lambda^{\ast}\bigl(\nabla f_{\bt, c'}(\bx^{\ast}) - \nabla f_{\bt, c}(\bx^{\ast})\bigr)&= 0, \\ f_{\bt, c'}(\bx^{\ast}) - f_{\bt, c}(\bx^{\ast})&= 0. \label{eq:relaxed-kkt-conditions}\end{split}
    \end{align}
    These yield a square system of $n+1$ equations in $n+1$ unknowns, instead of $n+k-1$, resulting in substantial computational savings when the number of classes $k$ is large.
    This relaxed polynomial system forms the basis of our exact certification procedure via numerical homotopy continuation, which is summarized in Algorithm~\ref{alg:hc} and described in more detail in ~\cref{app:experiments}.
    \begin{algorithm}[tb]
    \caption{Exact robustness certification via homotopy continuation}
    \label{alg:hc}
    \begin{algorithmic}[1]
    \Require Network parameters $\bt$, test point $\bxi \in \mathbb{R}^n$, radius $\epsilon > 0$
    \Ensure Relaxed robustness margin $\tilde{\gamma}$ and certification decision
    
    \State $c \gets \arg\max_{i\in[k]} f_{\bt,i}(\bxi)$
    \State $\tilde{\gamma} \gets +\infty$
    \For{each $c' \in [k]\setminus\{c\}$}
        \State Define the relaxed critical system
        \[
        F_{c,c'}^{\bt}(\bx,\lambda;\bxi):=
        \begin{bmatrix}
        \bx-\bxi+\lambda\bigl(\nabla f_{\bt,c'}(\bx)-\nabla f_{\bt,c}(\bx)\bigr)\\[2mm]
        f_{\bt,c'}(\bx)-f_{\bt,c}(\bx)
        \end{bmatrix}.
        \]
        \State Construct a homotopy
        \[
        H_{c,c'}(\bz,t;\gamma)=\gamma(1-t)G(\bz)+tF_{c,c'}^{\bt}(\bz;\bxi),
        \qquad \bz=(\bx,\lambda),
        \]
        \State where $G$ is a start system with known solutions and $\gamma\in\mathbb{C}^\ast$ is generic
        \State Track all solution paths from $t=0$ to $t=1$
        \State Let $\mathcal{R}_{c,c'}$ be the real endpoints of $F_{c,c'}^{\bt}(\bx,\lambda;\bxi)=0$
        \If{$\mathcal{R}_{c,c'} \neq \varnothing$}
            \State $\tilde{\gamma} \gets \min\!\Bigl\{\tilde{\gamma},\;
            \min_{(\bx,\lambda)\in\mathcal{R}_{c,c'}} \sqrt{d_{\bxi}(\bx)}\Bigr\}$
        \EndIf
    \EndFor
    \State \Return \textbf{certified} if $\tilde{\gamma}>\epsilon$, else \textbf{not certified}
    \end{algorithmic}
    \end{algorithm}

  \subsection{Polynomial Neural Networks} 
    In our theoretical analysis we will consider classifiers that are implemented as polynomial neural networks, defined as follows. 

    For given $L \ge 1$ and $\boldsymbol{h}=(h_0,h_1,\dots,h_L)\in\mathbb{N}_{\geq 0}^{L+1}$,  
    a \emph{fully connected feedforward neural network} with architecture $\boldsymbol{h}$ is a map
    $f_\vtheta:  \R^{h_0}\to\R^{h_L}$ of the form
    \[
      f_{\bt}
      =
      f_L \circ \sigma_{L-1} \circ f_{L-1} \circ \cdots \circ \sigma_1 \circ f_1,
    \]
    where, for each $\ell\in[L]$,
    \[
      f_\ell(\bx)=W_\ell \bx+\bb_\ell,
      \qquad
      W_\ell\in\R^{h_\ell\times h_{\ell-1}},\quad
      \bb_\ell\in\R^{h_\ell}, 
    \]
    and $\sigma_\ell:\R^{h_\ell}\to\R^{h_\ell}$ is an activation function that acts
    coordinatewise, i.e.,
    \[
    \sigma_\ell(x_1,\dots,x_{d_\ell})
      =
      \bigl(\phi_\ell(x_1),\dots,\phi_\ell(x_{d_\ell})\bigr) , 
    \]
    for some scalar activation function $\phi_\ell:\R\to\R$. 
    The parameter of the network is the tuple \(\bt = (\left(W_{\ell}, \bb_{\ell}\right))_{\ell=1}^{L}\) consisting of the weights $W_\ell$ and biases $\bb_\ell$ of each layer $\ell\in[L]$. 
    The network is \emph{shallow} if $L=2$ and \emph{deep} if $L>2$. 
    In this paper, we focus on \emph{polynomial neural networks}, which have activation functions $\phi_\ell$ that are polynomial.
    More specifically, we consider a monomial activation fixed across layers,  
    \[
      \phi_\ell(t)=t^d, \qquad \ell\in[L-1],
    \]
    for some fixed integer $d\ge 1$.  
    Equivalently, the coordinatewise activation maps are
    \[
      \sigma_\ell(x_1,\dots,x_{d_\ell})
      =
      (x_1^d,\dots,x_{d_\ell}^d),
      \qquad \ell\in[L-1].
    \]
    In this setting, each component of the network map $f_\vtheta$ is a polynomial in the input $\bx$ and also a polynomial in the parameter $\vtheta$. 
    If all biases $\bb_\ell$ are fixed to zero, then the resulting network map is homogeneous. 
    For $k$-class classification,  we take $h_0=n$ and $h_L=k$, so $f_{\bt}:\mathbb{R}^n\to\mathbb{R}^k$ has output coordinates $f_{\bt,1},\dots,f_{\bt,k}$  representing the class logits, and the predicted class is $\arg\max_{i\in[k]} f_{\bt,i}(\bx)$.

\section{ED Degree of a Decision Boundary} \label{sec:ed-degree}
  We analyze the complexity of the metric projection problem onto the decision
  boundary of a classifier using the \emph{Euclidean Distance (ED) degree} from algebraic
  geometry.
  \begin{definition}[Euclidean Distance Degree]
    Let $\mathcal{V}\subseteq \mathbb{C}^{n}$ be a complex algebraic variety\footnote{For
    a real variety $\mathcal{V}_{\mathbb{R}}\subseteq \mathbb{R}^{n}$, we consider
    its complexification $\mathcal{V}\subseteq \mathbb{C}^{n}$. 
    The number of real critical points depends on the data $\bu$, but the number of complex critical
    points is generically constant.}, and let
    \[
      \mathcal{V}_{\mathrm{reg}}:= \mathcal{V}\setminus \mathcal{V}_{\mathrm{sing}}
      = \left\{ \bx \in \mathcal{V}\mid \dim T_{\bx}\mathcal{V}= \dim \mathcal{V}
      \right\}
    \]
    denote its smooth locus.
    The \emph{Euclidean Distance (ED) degree} of $\mathcal{V}$ is the number of complex
    critical points of the squared Euclidean distance function $d_{\bu}(\bx)$ with
    respect to a generic\footnote{In this context, a point is \emph{generic} if
    it lies in a dense open Zariski subset of the parameter space. 
    Specifically, it does not lie on the ED discriminant locus defined in Section
    \ref{sec:ed-disc}.} data point $\bu \in \mathbb{R}^{n}$ restricted to
    $\mathcal{V}_{\mathrm{reg}}$:
    \[
      \mathrm{EDdegree}(\mathcal{V}) := \#\left\{ \bx \in \mathcal{V}_{\mathrm{reg}}
      \;\middle|\; \bx - \bu \perp T_{\bx}\mathcal{V}\right\},
    \]
    where $T_{\bx}\mathcal{V}$ denotes the tangent space of $\mathcal{V}$ at $\bx$
    and the orthogonality is with respect to the standard dot~product.
  \end{definition}

  In Problem~\ref{eq:relaxed-opt-prob}, we minimize the distance from $\bxi$ to~$\V_{c,c'}^{\bt}$.
  Let $I^{\bt}_{c,c'}= \langle f_{\bt,c}(\bx) - f_{\bt,c'}(\bx) \rangle$ denote the
  ideal generated by the logit difference for fixed $\bt$.
  Note the ED degree is defined with respect to $\mathbb{C}$, which ensures the count
  of critical points is a generic invariant.
  We treat $\mathcal{V}_{c,c'}^{\bt}\subseteq\mathbb{C}^{n}$ as a complex algebraic variety by taking the Zariski closure.

  The ED degree of $\mathcal{V}_{c,c'}^{\bt}$ is an upper bound on the optimization complexity: for our data point $\bxi$, the number of complex critical points is bounded above by $\mathrm{EDdegree}(\mathcal{V}_{c,c'}^{\bt})$, with equality holding when $\bxi$ is generic.

  For fixed $\bt$, the variety $\mathcal{V}^{\bt}_{c,c'}$ is a hypersurface, since its ideal is generated by a single polynomial. Its ED degree is computed via Algorithm~\ref{alg:ed-degree}, where the resulting ideal $C^{\bt}_{c,c'}$ is zero-dimensional \citep[][Lemma 2.1]{draisma_2016}, and hence has finitely many solutions. 
  Its degree equals the ED degree of the decision boundary. 
  \begin{algorithm}
    [tb]
    \caption{Symbolic computation of the ED degree}
    \label{alg:ed-degree}
    \begin{algorithmic}
      [1] \Require Network parameters $\bt$, classes $(c,c')$ \Ensure ED degree
      of $\mathcal{V}^{\bt}_{c,c'}$

      \State Sample a generic point $\bu = (u_{1},\ldots,u_{n}) \in \mathbb{R}^{n}$.

      \State Form the augmented Jacobian
      \[
        AJ^{\bt}_{c,c'}=
        \begin{bmatrix}
          \bu - \bx                                                   \\
          \nabla_{\bx}\!\left(f_{\bt,c}(\bx) - f_{\bt,c'}(\bx)\right)
        \end{bmatrix}.
      \]

      \State Define the ideal
      \[
        J^{\bt}_{c,c'}= I^{\bt}_{c,c'}+ \langle \text{all $2$-minors of }AJ^{\bt}
        _{c,c'}\rangle .
      \]

      \State Saturate to remove the singular locus:
      \[
        C^{\bt}_{c,c'}= \left( J^{\bt}_{c,c'}: \langle \nabla_{\bx}(f_{\bt,c}- f_{\bt,c'}
        ) \rangle^{\infty}\right).
      \]

      \State \Return $\deg(C^{\bt}_{c,c'})$.
    \end{algorithmic}
  \end{algorithm}

  In the remainder of this section, we address the questions:
  \begin{enumerate}[leftmargin=2em, labelsep=0.5em] 
    \item[(i)] For a given architecture and \textit{generic} parameters $\bt$, what
      is the ED degree of the decision boundary $\mathcal{V}_{c,c'}^{\bt}$?

    \item[(ii)] What is the number of \textit{real} critical points of the
      Euclidean distance minimization problem?
  \end{enumerate}

  \subsection{ED Degree for Generic Parameters} \label{sec:ed-degree-generic}
    In this subsection, we study the ED degree of the decision boundary for a parameter $\bt$ that is fixed to be generic.
    This reflects the typical complexity of the problem, noting that the non-generic parameters form a thin set of measure zero.
    Denote the polynomial of the boundary as $B_{\bt}(\bx) = f_{\bt,c}(\bx) - f_{\bt,c'}(\bx)$, so $\mathcal{V}_{c,c'}^{\bt} = \{\bx \in \mathbb{R}^{n}\mid B_{\bt}(\bx) = 0\}$.
    For shallow networks, we obtain the following result.
    \begin{restatable}
      {proposition}{propgenericeddegree} \label{prop:generic-ed-degree} 
      Consider a neural network with architecture $(n, h, k)$ and a degree-$d$ polynomial
      activation function. 
      Assume the parameters $\bt$ are generic. 
      The ED degree of the decision boundary $\mathcal{V}_{c,c'}^{\bt}$ is given by
      \[
        \text{EDdegree}(\mathcal{V}_{c,c'}^{\bt}) = d\sum_{i = 0}^{m-1}(d-1)^{i},\;
        \text{where } m = \min\{n, h\}.
      \]
    \end{restatable}

    This result reveals that the algebraic complexity is governed by the layer bottleneck $m = \min\{n, h\}$, implying that the difficulty of exact verification grows exponentially with the width of the network even when the other dimension is fixed.
    \begin{corollary}
      In the case of a quadratic activation $d=2$, the summation simplifies to $m$, yielding an ED degree of~$2m$.
    \end{corollary}

    \begin{example}
      For $d = 3$ and generic parameters, \cref{prop:generic-ed-degree} yields the
      following values:
      \begin{itemize}[leftmargin=2em, labelsep=0.5em] 
        \item Architecture $(3,2,2)$: $m=2$, ED degree $9$.

        \item Architecture $(3,3,2)$: $m=3$, ED degree $21$.

        \item Architecture $(4,4,2)$: $m=4$, ED degree $45$.
      \end{itemize}
    \end{example}

    \begin{remark}[Deep architectures]
      For the architecture $(n, h_{1}, \dots, h_{s}, k)$, the polynomial
      $B_{\bt}(\bx)$ has degree~$D=d^{s}$.
      Unlike the shallow case, deep networks define nested functional compositions.
      This rigid algebraic structure often leads to \textit{reducible} varieties
      rather than generic hypersurfaces.
      Consequently, substituting the total degree $D$ into the generic formula of~\cref{prop:generic-ed-degree}
      yields only a theoretical upper bound, while the true ED degree may be strictly
      lower.
      This is seen in Proposition \ref{prop:bottleneck} and Table
      \ref{tab:intermediate_experiments}.
    \end{remark}
    Computing the exact ED degree of deep networks is difficult and remains an open
    problem in general.
    However, for networks that constrict to width~$1$, the decision boundary
    factorizes into parallel copies of a shallow surface, allowing us to prove the
    following result.
    \begin{restatable}
      {proposition}{propbottleneck} \label{prop:bottleneck} Consider a deep neural
      network with architecture $(n, h_{1}, 1, \dots, 1, k)$ having $s$ hidden
      layers and degree-$d$ activation. Assume $h_{1}\ge n$ and generic $\bt$.
      Then
      \[
        \text{EDdegree}(\mathcal{V}_{c,c'}^{\bt}) = D \sum_{i=0}^{n-1}(d-1)^{i},
      \]
      where $D = d^{s}$ is the total composite degree of the network.
    \end{restatable}
    \begin{example}
      The polynomial neural network with architecture $(2, 2, 1, 1, 2)$ and $d= 2$
      has ED degree 16. This is four times smaller than the naive bound of 64 for a
      generic curve of degree 8, illustrating how the network's bottleneck
      structure significantly reduces algebraic complexity.
    \end{example}
    On the other hand, if the network maintains sufficient width, the decision
    boundary achieves maximal generic complexity.
    \begin{restatable}
      {proposition} {propgenericeddegreedeep}\label{prop:generic-ed-degree-deep}
      Consider a deep neural network with architecture
      $(n, h_{1}, \dots, h_{s}, k)$. Assume degree-$d$ activation, and
      $h_{i}\ge n$ for all $i \in [s]$. For generic parameters $\bt$,
      \[
        \text{ED degree}(\mathcal{V}_{c,c'}^{\bt}) = D \sum_{i=0}^{n-1}(D-1)^{i},
      \]
      where $D = d^{s}$ is the composite degree of the network.
    \end{restatable}
    For the intermediate layer width regime ($1 < h_{i}< n$), we can compute
    specific cases numerically.
    We present examples in Table~\ref{tab:intermediate_experiments}.
    \begin{table}[t]
      \caption{Numerical ED degrees for deep architectures.}
      \label{tab:intermediate_experiments}
      \begin{center}
          \begin{sc}
            \begin{tabular}{lcccr}
              \toprule Architecture           & {Total degree} & {Computed}  \\
              $(n, [h_{1},\ldots, h_{s}]), d$ & $D = d^{s}$    & {ED Degree} \\
              \midrule $(3, [3, 2]), 2$       & 4              & 36          \\
              $(3, [3, 2]), 3$                & 9              & 333         \\
              $(3, [2, 3]), 3$                & 9              & 81          \\
              $(4, [4, 3]), 2$                & 4              & 128         \\
              $(4, [4, 2]), 2$                & 4              & 64          \\
              $(5, [5, 3]), 2$                & 4              & 260         \\
              $(3, [3, 2, 2]), 2$             & 8              & 168         \\
              $(3, [3, 3, 2]), 2$             & 8              & 328         \\
              \bottomrule
            \end{tabular}
          \end{sc}
      \end{center}
    \end{table}
    \begin{remark}
      The results indicate that intermediate width architectures have ED degree between
      the theoretical bounds. \cref{prop:generic-ed-degree-deep} serves as a
      generic upper bound, and \cref{prop:bottleneck} serves as a constructive lower
      bound for fixed depth~$s$, where width 1 minimizes the algebraic mixing of
      variables.
      For example, architecture $(3, [3, 2], 2)$ yields ED degree $36$, between the
      bounds $12$ and $52$, which is greater than the lower bound $12$ and less
      than the generic upper bound $52$. This suggests that the decision boundary lives
      on a sparse subvariety whose complexity depends on the specific rank constraints
      of the hidden layers.
    \end{remark}

    Our framework allows us to distinguish between the geometric complexity of the boundary and the identifiability of the
    parameters. 
    Let $p$ be the number of parameters and $N$ be the dimension of the space of degree-$D$ polynomials. 
    Let $\Phi: \mathbb{R}^{p}\to \mathbb{R}^{N}$ be the parameterization map sending weights to coefficients.
    The observed drop in ED degree is governed by the dimension of the image
    $\mathrm{Im}(\Phi)$ relative to the ambient dimension $N$. 
    For intermediate width architectures, we typically have $\dim(\mathrm{Im}(\Phi)) \ll N$. 
    This strict inequality ensures that the decision boundary lies on a proper subvariety of the polynomial space,
    which explains the reduced algebraic complexity.
    In contrast, parameter identifiability is determined by the dimension of the
    image relative to the domain dimension $p$. 
    If $\dim(\mathrm{Im}(\Phi)) < p$, then by the Fiber Dimension Theorem \citep[][Chapter 1]{hartshorne_1977}, the fibers
    $\Phi^{-1}(f)$ are positive-dimensional varieties, and thus the network is
    defective.
    This comparison clarifies the trade-off: the ED degree depends on how small
    the image is inside $\mathbb{R}^{N}$, while identifiability depends on how compressed
    the map is from $\mathbb{R}^{p}$. 
    A network can thus achieve low verification complexity while remaining non-identifiable if the architecture is over-parameterized.

  \subsection{Expected Real ED Degree} \label{sec:exp-ed-degree}
    The results in the preceding section concern critical points of the distance minimization problem over the complex field.
    This is the natural quantity for homotopy continuation, which computes all isolated complex solutions of the critical system.
    For robustness verification, however, one is often ultimately interested only in the real critical points, since these correspond to candidate nearest points on the decision boundary in the real input space.
    In practice, the number of real ED critical points is typically much smaller than the ED degree and may vary substantially across different generic realizations of the network parameters.
    To capture this phenomenon, we study the \emph{expected real ED degree} under a distribution on the parameters.
  
    A convenient tool for this purpose is the Kac--Rice formalism \citep{kac_1943, rice_1944}, which gives an exact expression for the expected number of zeros of a random field.
    Applied to our ED critical system, it yields a formula for the expected number of real critical points encountered in robustness verification.
    \begin{restatable}
      {theorem}{thmedkacrice} \label{thm:ed-kac-rice} Let
      $f:\mathbb{R}^{n}\to\mathbb{R}$ be a random field that is almost surely
      twice continuously differentiable.
      Fix $\bxi \in \mathbb{R}^{n}$ and, for
      $(\bx,\lambda)\in\mathbb{R}^{n}\times(\mathbb{R}\setminus\{0\})$, define
      \[
        F(\bx, \lambda):=
        \begin{pmatrix}
          f(\bx)                         \\
          \bx-\bxi+\lambda \nabla f(\bx)
        \end{pmatrix}.
      \]
      Assume that, for almost every $(\bx,\lambda)$, the random vector $F(\bx,\lambda
      )$ is nondegenerate\footnote{By nondegenerate we mean that, for fixed $(\bx ,
      \lambda)$, the random vector $F(\bx,\lambda)$ is not supported on any proper
      affine subspace of $\mathbb{R}^{n+1}$, equivalently that it admits a density
      with respect to Lebesgue measure on $\mathbb{R}^{n+1}$.}, that
      \[
        \mathbb{P}\!\left(\det J_{F}(\bx,\lambda)=0 \,\middle|\, F(\bx,\lambda )=\mathbf{0}
        \right)=0,
      \]
      where $J_{F}(\bx,\lambda)$ denotes the Jacobian matrix of $F$ with respect to
      $(\bx,\lambda)$, and that the integrand below is integrable on a region
      $U := U_{\bx}\times U_{\lambda}$, with $U_{\lambda}\subset\mathbb{R}\setminus
      \{0\}$.
      Then
      \begin{equation*}
        \mathbb{E}\Big[\#\{(\bx,\lambda) \in U:\ F(\bx,\lambda)=\mathbf{0}\}\Big ]
        = \int_{U}|\lambda|^{-n}\, p_{(f(\bx), \nabla f(\mathbf{x}))}\big ( 0, \tfrac
        {\bxi - \bx}{\lambda}\big) \mathcal{I}(\bx, \lambda) \, \mathrm{d}(\bx,\lambda
        ),
      \end{equation*}
      where $p_{(f(\bx), \nabla f(\bx))}$ is the joint density of
      $(f(\bx) ,\nabla f(\bx))$,
      \begin{equation*}
        \mathcal{I}(\bx, \lambda) = \mathbb{E}\Big[\big|{\nabla f(\bx)}^{\top}\operatorname{adj}
        \!\left(I_{n}+\lambda \nabla^{2}f(\bx)\right ){\nabla f(\bx)}\big| \Big|\ f
        (\bx)=0,\ \nabla f(\bx)= \frac{\bxi - \bx}{\lambda}\Big] ,
      \end{equation*}
      with $\operatorname{adj}(\cdot)$ the adjugate of a square matrix.
    \end{restatable}
    The above theorem is quite general since it applies to any sufficiently regular random
    field satisfying the stated nondegeneracy and integrability assumptions.
    The difficulty is not the applicability of Kac--Rice, but the evaluation of the resulting integral.
    Indeed, the formula involves the joint density of $(f(\bx),\nabla f(\bx))$ and 
    a conditional expectation depending on the Hessian $\nabla^2 f(\bx)$.
    For a generic random field, these quantities are typically hard to characterize explicitly.
    
    To make the computation tractable, we therefore specialize to the neural network Gaussian process (NNGP) regime \citep{lee_2018}.
    This specialization is analytically convenient because, in the infinite-width limit, 
    the random network induces a Gaussian field whose law is completely determined by its kernel.
    Consequently, the random variables $f(\bx)$, $\nabla f(\bx)$, and $\nabla^2 f(\bx)$ are 
    jointly Gaussian, so the density and conditional moments appearing in \cref{thm:ed-kac-rice} can be 
    expressed in terms of kernel derivatives.
    Thus, while Kac--Rice itself is model-agnostic, the NNGP structure makes the formula sufficiently explicit to derive concrete bounds in high dimension and a closed-form expression in dimension one.
    \begin{restatable}{theorem}{thmedkacricehigh} \label{thm:ed-kac-rice-high}
      Let $f:\R^{n} \to \R$ be the neural network Gaussian process (NNGP) induced by an infinite-width two-layer
      network with polynomial activation of degree $d \geq 1$.
      Then there exists a constant $C_{n,\bxi}>0$, independent of $d$, such that
      \[
      \mathbb{E}\Big[\#\{(\bx,\lambda)\in \mathbb{R}^n\times(\mathbb{R}\setminus\{0\}):F(\bx,\lambda)=\mathbf{0}\}\Big]
      \le C_{n,\bxi}\, d^{n/2}.
      \]
      In particular, for fixed $n$, the expected real ED degree is $O(d^{n/2})$ as $d\to\infty$.
    \end{restatable}

    \begin{restatable}{proposition}{coredkacrice} \label{prop:ed-kac-rice} 
      Let $f:\R \to \R$ be the neural network Gaussian process (NNGP) induced by an infinite-width two-layer 
      network with polynomial activation of degree $d \geq 1$.
      Then the expected real ED degree~is
      \[
        \mathbb{E}\Big[\#\{\,(x,\lambda) \in \R \times (\R \setminus \{0\}) : F(x,\lambda)= \mathbf{0}\,\}\Big] = \frac{d}{\sqrt{2d-1}}.
      \]
    \end{restatable}
    
    By \cref{thm:ed-kac-rice-high}, for fixed $n$ the expected real ED degree is $O(d^{n/2})$, whereas the generic complex ED degree scales as $d^n$ by \cref{prop:generic-ed-degree}.
    Hence the real critical-point count grows like the square root of the complex algebraic count.
    This is reminiscent of the classical Kostlan--Shub--Smale model for random
    homogeneous polynomial systems: for the orthogonally invariant Gaussian
    ensemble of degrees $d_1,\dots,d_n$, the expected number of real projective
    roots is $\sqrt{d_1\cdots d_n}$, i.e., the square root of the generic complex
    B\'ezout count $d_1\cdots d_n$ \citep{kostlan_1993,edelmanz_1995,kostlan_2002}.
    Although the ED critical system arising from NNGPs is not a standard Kostlan system,
    it exhibits the same qualitative phenomenon: most complex critical points do
    not materialize over the reals, but the expected real count still grows
    polynomially with degree.
    In dimension one, \cref{prop:ed-kac-rice} gives the exact asymptotic
    $\frac{d}{\sqrt{2d-1}} \sim \sqrt{d/2}$.

\section{Discriminants}
  In this section, we analyze the singular loci that govern transitions in
  verification complexity.
  We characterize the partition of the input space into regions of constant real
  critical point counts via the ED discriminant, and the stratification of the
  parameter space according to the algebraic degree of the decision boundary via
  the parameter discriminant. 
  More precisely, we answer the following questions:
  \begin{enumerate}[leftmargin=2em, labelsep=0.5em] 
    \item[(i)] For fixed parameters $\bt$, which test data points $\bxi$ yield
      fewer than the generic number of complex solutions?

    \item[(ii)] How does the ED degree of $\mathcal{V}_{c,c'}^{\bt}$ change as the
      network parameters $\bt$ vary?
  \end{enumerate}

  \subsection{ED Discriminant}\label{sec:ed-disc} 
    In this subsection, we study the locus of data points that have fewer than the generic number of critical points when minimizing the distance to the decision boundary.
    For such data points, two distinct complex critical points collapse into a single
    critical point of higher multiplicity.
    This geometric degeneration marks the boundary where the number of real
    critical points changes; crossing this locus corresponds to the bifurcation event
    in the optimization landscape.
    We formalize this as the Euclidean distance discriminant.
    \begin{definition}[Euclidean Distance Discriminant]
      Let $\mathcal{V}\subset \mathbb{C}^{n}$ be an algebraic variety. The \textit{Euclidean
      Distance (ED) discriminant} of $\mathcal{V}$, denoted by
      $\Sigma(\mathcal{V})$, is the Zariski closure of all data points $\boldsymbol
      {u}\in \mathbb{C}^{n}$ for which the critical points of the squared distance
      function $d_{\boldsymbol{u}}(\bx)$ on $\mathcal{V}$ are not distinct. That
      is, $\Sigma(\mathcal{V})$ characterizes the data points for which the system 
      has roots of multiplicity greater than one.
    \end{definition}
    The notion of ED discriminant was introduced in the context of algebraic optimization by \citet{draisma_2016}, and is also known as the \textit{evolute}.
    It governs the stability of the optimization problem over the real numbers.
    In contrast to the number of complex critical points, which is generically constant, the number of \textit{real} critical points is constant only on the connected components of $\mathbb{R}^{n}\setminus \Sigma(\mathcal{V})$. 

    \begin{example}
      \label{ex:ed-degree-4} Consider a quadratic polynomial neural network with one
      hidden layer and architecture $(2,2,2)$.
      Fix parameter values
      $\bt = (W_{1}, \boldsymbol{b}_{1}, W_{2}, \boldsymbol{b}_{2})$ with:
      \[
        W_{1}= \left(
        \begin{smallmatrix}
          1 & 2 \\
          3 & 1
        \end{smallmatrix}\right), \boldsymbol{b}_{1}= \left(
        \begin{smallmatrix}
          0 \\
          1
        \end{smallmatrix}\right), W_{2}= \left(
        \begin{smallmatrix}
          2 & 1 \\
          1 & 2
        \end{smallmatrix}\right), \boldsymbol{b}_{2}= \left(
        \begin{smallmatrix}
          2 \\
          1
        \end{smallmatrix}\right),
      \]
      so the network output function $f_{\bt}$ becomes
      \begin{align*}
        (x_{1}, x_{2}) \mapsto \left[ \begin{matrix}f_{\bt,c}\\ f_{\bt,c'}\end{matrix} \right] = \left[ \begin{matrix}11\,x_{1}^{2}+14\,x_{1}\,x_{2}+9\,x_{2}^{2}+6\,x_{1}+2\,x_{2}+3 \\ 19\,x_{1}^{2}+16\,x_{1}\,x_{2}+6\,x_{2}^{2}+12\,x_{1}+4\,x_{2}+3\end{matrix}\right].
      \end{align*}
      The decision boundary defined by $B_{\bt}(\bx) = f_{\bt,c}(\bx )- f_{\bt,c'}( \bx)$ is a hyperbola. 
      A computation in \texttt{Macaulay2} \citep{M2} shows that $\mathcal{V}_{c,c'}^{\bt}$ has ED degree $4$. 
      The ED discriminant $\Sigma(\mathcal{V}_{c,c'}^{\bt})$ is the variety in~$\mathbb{C}^{2}$ defined as the set of zeros of a degree-$6$ polynomial $D \subseteq \mathbb{C}[u_{1}, u_{2}]$ in the data variables
      $u_{1}$ and $u_{2}$:
      \begin{align*}
        & D(u_{1},u_{2}) = 27\,u_{1}^{6}+54\,u_{1}^{5}u_{2}-180\,u_{1}^{4}u_{2}^{2}-280\,u_{1}^{3}u_{2}^{3}+ 480\,u_{1}^{2}u_{2}^{4}+384\,u_{1}\,u_{2}^{5 }-512\,u_{2}^{6}+54\,u_{1}^{5}         \\
        & +180\,u_{1}^{4}u_{2}-120\,u_{1}^{3}u_{2}^{2}-720\,u_{1}^{2}u_{2}^{3}+768\,u_{2}^{5}-126\,u_{1}^{4}+582 \,u_{1}^{3}u_{2}-2409\,u_{1}^{2}u_{2}^{2}-402\,u_{1}\,u_{2}^{3}-1371\,u_{2}^{4} \\
        & -334\,u_{1}^{3}+1566\,u_{1}^{2}u_{2}-1878\,u_{1}\,u_{2}^{2}+808u_{2}^{3}-132\,u_{1}^{2}+1152\,u_{1}\,u_{2}-1218\,u_{2}^{2}-12\,u_{1}+516\,u_{2}-152.
      \end{align*}

      For all points $(u_{1},u_{2})\in\mathbb{R}^{2}$ with $D(u_{1},u_{2})\neq 0$,
      the number of distinct critical points of $d_{\boldsymbol{u}}(\bx)$ on $\mathcal{V}
      _{c,c'}^{\bt}$ is $4$, which can be real or complex, matching the ED degree.

      On the other hand, for the test point $\bu = (-2, 0)$ on the ED discriminant
      $D(\bu) = 0$, the ideal $C^{\bt}_{c,c'}$ with respect to $\bu$ has a primary
      decomposition:
      \(
        \langle 3x_{1}+ x_{2}+ 2,\; 25x_{2}^{2}- 20x_{2}+ 4\rangle \;\cap \langle 3 x_{1}- 4x_{2}+ 1,\; 25x_{2}^{2}+ 4x_{2}- 2\rangle.
      \)
      The first ideal has one real solution $(-4/5,\, 2/5)$ with multiplicity $2$. 
      The second ideal has two real solutions. 
      Geometrically, this marks a bifurcation in the optimization landscape: crossing the discriminant
      at $\bu$ changes the number of real critical points, as shown in \cref{fig:discriminant}.
      \begin{figure}[ht]
        \centering
        \begin{tikzpicture}
          \node at (0,0)
                      {\includegraphics[width=0.3\textwidth]{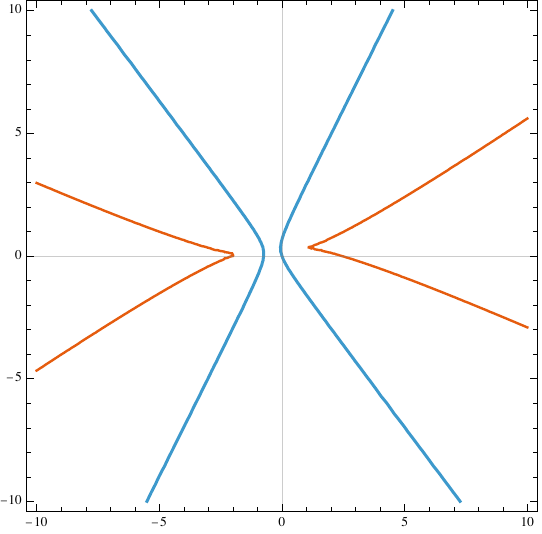}};
                  \node[anchor=west, inner sep=0, fill=white]
            at
            (-1.9,1.8)
            {\footnotesize{\textcolor{CornflowerBlue}{$\bullet$\;}decision boundary}};
          \node[anchor=west, inner sep=0, fill=white]
            at
            (-1.9,1.5)
            {\footnotesize{\textcolor{Apricot}{$\bullet$\;}discriminant}};
                  \node at
                      (5,0)
                      { \includegraphics[width=0.3\textwidth]{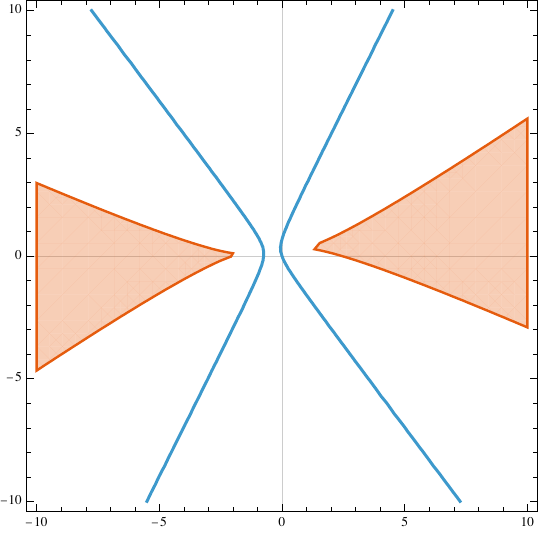}};
                      \node[anchor=west, inner sep=0, fill=white]
            at
            (3.1,1.8)
            {\footnotesize{$\#$ real critical points}};
                  \fill[white] (3,-1) rectangle (3.035,1);
                  \fill[white] (7.15,-.6) rectangle (7.2,1.4);
                  \node at (3.5,0) {\small $4$};
                  \node at (5,.75) {\small $2$};
                  \node at (6.5,0.25) {\small $4$};
                  \node[] (A) at (3.75,1) {\small $3$};
                  \node[] (B) at (4,.22) {};
                  \draw[-] (A) -- (B);
              \end{tikzpicture}
      \caption{{} Left: The decision boundary is shown in blue and the corresponding
        ED discriminant curve is shown in orange. Right: The number of distinct real
        critical points of the distance minimization problem on different regions of
        the data space separated by the discriminant. The shaded region is where
        $D(u_{1},u_{2})>0$.}
        \label{fig:discriminant}
      \end{figure}
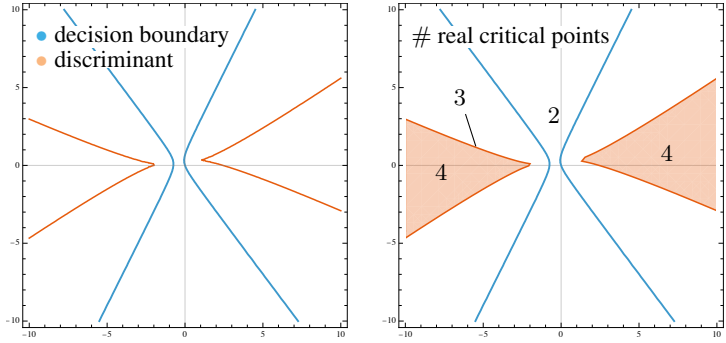

    \end{example}
    \begin{remark}
        The plot on the left shows the decision boundary in blue and the ED
        discriminant in orange. The plot on the right depicts the regions where $D >0$
        (shaded) and $D<0$ (unshaded). Points in the unshaded (white) region, such as $(
        u_{1},u_{2})=(0,0)$, have two real and two complex critical points, whereas points
        in the shaded region, such as $(10,0)$ and $(-10,0)$, have four real critical points.
        Along the separating curve $D=0$ two complex critical points merge into
        a single real one; data points on this curve have three real
        critical points.
    \end{remark}
    To generalize the idea in the above example,
    fix some parameters $\bt$ and let $B_{\bt}(\bx) = f_{\theta,c}(\bx) - f_{\theta,c'}
    (\bx)$ as before. The degree of the corresponding hypersurface can be
    arbitrarily large. In this case, we can compute the ED discriminant
    symbolically using Algorithm~\ref{alg:ed-discriminant}.

    \begin{algorithm}
      [tb]
      \caption{Symbolic computation of the ED discriminant}
      \label{alg:ed-discriminant}
      \begin{algorithmic}
        [1] \Require Network parameters $\bt$, classes $(c,c')$ \Ensure Ideal of
        the ED discriminant $\Sigma(\mathcal{V}^{\bt}_{c,c'})$

        \State Introduce a new variable $\lambda$ and work over $\mathbb{Q}[\bx,\bu
        ,\lambda]$.

        \State Define
        \[
          J^{\bt}_{c,c'}= I^{\bt}_{c,c'}+ \langle x_{i}- u_{i}- \lambda \tfrac{\partial B}
          {\partial x_i}: i \in [n] \rangle .
        \]

        \State Let $\operatorname{Jac}^{\bt}_{c,c'}$ be the $n \times n$ Jacobian
        matrix of $J^{\bt}_{c,c'}$.

        \State Define
        \[
          C^{\bt}_{c,c'}= \Big( (I + \langle \det \operatorname{Jac}^{\bt}_{c,c'}\rangle
          ) : \langle n\text{-minors of }\operatorname{Jac}^{\bt}_{c,c'}\rangle^{\infty}
          \Big).
        \]

        \State Compute the elimination ideal
        \[
          E = C^{\bt}_{c,c'}\cap \mathbb{Q}[u_{1},u_{2}].
        \]

        \State \Return $D = \sqrt{E}$.
      \end{algorithmic}
    \end{algorithm}

    \textbf{Intuition}: Step 1 introduces the Lagrange multiplier $\lambda$ to express the distance-minimization
    conditions algebraically. Step 2 encodes the critical-point equations $\bx - \bu
    = \lambda \nabla_{\bx}B$ and $B_{\bt}(\bx)=0$. Step 3 forms the Jacobian of
    the critical-point system with respect to the variables $(\bx,\lambda)$ to
    detect when the correspondence between data points $\bu$ and critical points
    $(\bx,\lambda)$ ceases to be locally identifiable. Each solution $(\bx ,\lambda
    )$ represents a critical point of the distance function for the given $\bu$.
    When the Jacobian is invertible, these solutions depend smoothly on $\bu$,
    which means that each critical point moves continuously as $\bu$ varies. When the
    Jacobian becomes singular, two or more critical points collide, or a critical point
    loses local uniqueness. These singular values of $\bu$ mark the onset of
    degeneracy in the optimization landscape and define the ED
    discriminant. Step 4 enforces this degeneracy by setting the determinant of the
    Jacobian to zero and removes trivial singularities via saturation. Step 5
    eliminates $\bx$ and $\lambda$, leaving only the condition on $\bu$. The
    algorithm returns the radical to remove extra embedded components that may
    have arisen from elimination. The returned ideal defines the locus where the critical
    points merge and the number of distinct critical points drops below the ED degree.

  \begin{example}
      \label{ex:ed-degree-2} Consider again a quadratic neural network with one hidden
      layer and the architecture $(2,2,2)$. Slightly tweak the parameter values
      $\theta = (W_{1}, \boldsymbol{b}_{1}, W_{2}, \boldsymbol{b}_{2})$ from the
      previous example:
      \[
        W_{1}= \left(
        \begin{smallmatrix}
          1 & 2 \\
          3 & 1
        \end{smallmatrix}\right), \boldsymbol{b}_{1}= \left(
        \begin{smallmatrix}
          1 \\
          2
        \end{smallmatrix}\right), W_{2}= \left(
        \begin{smallmatrix}
          2 & 1 \\
          1 & 2
        \end{smallmatrix}\right), \boldsymbol{b}_{2}= \left(
        \begin{smallmatrix}
          1 \\
          1
        \end{smallmatrix}\right),
      \]
      A computation shows that in this case $\mathcal{V}_{c,c'}^{\bt}$ has ED degree~2.
      This is not surprising, since the quadratic form $B_{\bt}(\bx)$ is
      degenerate and factors into two linear terms:
      \[
        B_{\bt}(x_{1},x_{2}) = (x_{2}- 2x_{1}- 1)(4x_{1}+ 3x_{2}+ 3).
      \]
      The ED discriminant is given by the following linear ideal
      \[
        D(u_{1}, u_{2}) = \langle 5u_{1}+ 3,\; 5u_{2}+ 1\rangle,
      \]
      whose variety is a single point $\bu^{*}= (-3/5, -1/5)$. At any other point
      $\bu$, the ideal $C_{c,c'}^{\bt}$ has two distinct critical real points.
      However, at $\bu^{*}$ they collide, producing a single critical point. This critical
      point is $\bu^{*}$ itself, since it is on the model, i.e.
      $B_{\bt}(\bu^{*}) = 0$.
    \end{example}
    The phenomenon observed in Example \ref{ex:ed-degree-2} generalizes to all reducible quadratic
    hypersurfaces of ED degree~2.

    \begin{restatable}
      {proposition} {propquadraticeddegree}\label{prop:quadratic-ed-degree} Let
      $\mathcal{V}= \{\bx : \ell_{1}(\bx)\ell_{2}(\bx) = 0\}$ be a reducible
      quadratic hypersurface consisting of two non-parallel hyperplanes. The ED degree
      of $\mathcal{V}$ is 2. The two critical points collide if and only if the data
      $\bu$ lies on the intersection of the two hyperplanes $\ell_{1}(\bu)=\ell_{2}
      (\bu)=0$. Consequently, the ED discriminant is contained in the model~$\mathcal{V}$.
    \end{restatable}

  \subsection{Parameter Discriminant}
  \label{sec:par-disc}
  In this section, we investigate how the complexity of the verification problem
  varies with the model parameters.
  Each parameter choice yields a decision boundary with an associated ED degree.
  The \textit{parameter discriminant} is a locus in the parameter space where
  the ED degree drops below the generic value.
  As a result, this discriminant stratifies the parameter space into regions of
  constant ED degree.

  We fix a shallow polynomial network with activation degree $d=2$ and
  architecture $(n,h,k)$, and denote the polynomial defining the decision boundary
  as $B_{\bt}( \bx)$.
  Then $B_{\bt}(\bx)$ defines a quadric hypersurface in~$\mathbb{C}^{n}$:
  \[
    B_{\bt}(\bx) = \bx^{\top}A_{\bt}\bx + b_{\bt}^{\top}\bx + c_{\bt}
  \]
  for some symmetric matrix $A_{\bt}\in\mathbb{R}^{n\times n}$, vector $b_{\bt}\in
  \mathbb{R}^{n}$, and scalar $c_{\bt}\in \mathbb{R}$ whose entries depend on
  $\bt$.
  Let $r$ be the number of distinct non-zero eigenvalues of $A_{\bt}$ and define the augmented matrix
  \[
    M_{\bt}=
    \begin{pmatrix}
      A_{\bt}                   & \frac{1}{2}b_{\bt} \\
      \frac{1}{2}b_{\bt}^{\top} & c_{\bt}
    \end{pmatrix}.
  \]
  \begin{restatable}
    {theorem}{eddegreeconicsnvariables} \label{thm:ed-degree-conics-n-variables}
    The ED degree of the decision boundary
    $\text{EDdegree}(\mathcal{V}_{c,c'}^{\bt})$ is given by
    \[
      \begin{cases}
        2r + 1 & \text{if }\text{rank}(M_{\bt}) = \text{rank}(A_{\bt}) + 2 \\
        2r     & \text{if }\text{rank}(M_{\bt}) = \text{rank}(A_{\bt}) + 1 \\
        2r - 2 & \text{if }\text{rank}(M_{\bt}) = \text{rank}(A_{\bt}).
      \end{cases}
    \]
  \end{restatable}
  As a corollary, we characterize when the ED degree attains the generic value
  described in Proposition \ref{prop:generic-ed-degree}.
  \begin{restatable}
    {corollary}{conicgeneric} \label{cor:conic-generic} The ED degree
    $\text{EDdegree}(\mathcal{V}_{c,c'}^{\bt})$ is exactly $2n$ if and only if
    all of the following hold:
    \begin{itemize}[leftmargin=2em, labelsep=0.5em] 
      \item $A_{\bt}$ is non-singular: $\text{rank}(A_{\bt}) = n$;

      \item $\mathcal{V}_{c,c'}^{\bt}$ is non-singular:
        $\text{rank}(M_{\bt}) = n+1$;

      \item $A_{\bt}$ has distinct eigenvalues: the discriminant of the
        characteristic polynomial of $A_{\bt}$ is non-zero.
    \end{itemize}
  \end{restatable}
  This implies that the ED degree drops strictly below $2n$ precisely when one of
  the generic conditions fails.
  This allows us to characterize the \textit{parameter discriminant} algebraically:
  \begin{proposition}
    \label{prop:parameter-discriminant} The parameter discriminant $\mathcal{D}\subset
    \mathbb{R}^{\text{p}}$ is the hypersurface defined by the vanishing of the polynomial
    \[
      \Delta(\bt) = \det(A_{\bt}) \cdot \det(M_{\bt}) \cdot \mathrm{Disc}_{\lambda}
      (\det(\lambda I - A_{\bt})).
    \]
    The set of parameters where $\text{EDdegree}(\mathcal{V}_{c,c'}^{\bt}) < 2n$
    corresponds to the zero locus~$\mathcal{D}= \{\bt \mid \Delta (\bt) =
    0\}$.
  \end{proposition}
  The term $\mathrm{Disc}_{\lambda}(\det(\lambda I - A_{\bt}))$ is the
  discriminant of the characteristic polynomial
  $P(\lambda) = \det(\lambda I - A_{\bt})$ with respect to
  $\lambda$. This vanishes if and only if $P(\lambda)$ has multiple roots,
  meaning $A_{\bt}$ has repeated~eigenvalues.
  \begin{remark}
    The discriminant $\mathcal{D}$ is the union of three distinct components,
    each with a geometric meaning:
    \begin{itemize}[leftmargin=2em, labelsep=0.5em] 
      \item $\det(A_{\bt}) = 0$: the quadratic form is rank-deficient, reducing the
        number of non-zero eigenvalues.

      \item $\det(M_{\bt}) = 0$: the decision boundary
        $\mathcal{V}_{c,c'}^{\bt}$ becomes singular as a variety.

      \item $\mathrm{Disc}_{\lambda}= 0$: the matrix $A_{\bt}$ has repeated
        eigenvalues, so $r < n$ even if $A_{\bt}$ is full rank.
    \end{itemize}
  \end{remark}

  \paragraph{Conic sections}
  We now specialize the preceding result to the case of $n=2$ input variables,
  $\bx = (x_{1}, x_{2})^{\top}$.
  In this setting, the decision boundary is a \emph{plane conic}, defined by
    \[
        B_{\bt}(x_{1}, x_{2}) = a x_{1}^{2}+ b x_{1}x_{2}+ c x_{2}^{2}+ d x_{1}+
        e x_{2}+ f = 0,
    \]
    where the coefficients $a, b, c, d, e, f$ depend on the parameters $\bt$. The
    augmented matrix $M_{\bt}$ takes the form
    \[
        M_{\bt}=
        \begin{pmatrix}
            a   & b/2 & d/2 \\
            b/2 & c   & e/2 \\
            d/2 & e/2 & f
        \end{pmatrix}, \quad \text{with}\quad A_{\bt}=
        \begin{pmatrix}
            a   & b/2 \\
            b/2 & c
        \end{pmatrix}.
    \]
  
  Every plane conic belongs to one of the following types, determined by the
  algebraic properties of $A_{\bt}$ and~$M_{\bt}$:
  \begin{itemize}[leftmargin=2em, labelsep=0.5em] 
    \item \textit{Ellipse:} $\det(A_{\bt}) = ac - b^{2}/4 > 0$. ED degree is $4$.

    \item \textit{Hyperbola:} $\det A_{\bt}= ac - b^{2}/4 < 0$. ED degree is $4$.

    \item \textit{Parabola:} $\det A_{\bt}= ac - b^{2}/4 = 0$. ED degree is $3$.

    \item \textit{Circle:} $A_{\bt}$ has repeated eigenvalues ($a=c$ and $b=0$).
      ED degree is $2$.

    \item \textit{Degenerate:} $\det M_{\bt}= 0$. ED degree is $\leq 2$.
  \end{itemize}

  Thus, the parameter discriminant $\mathcal{D}$ for $n=2$ is the union of the zero
  loci of three polynomials:
  \[
    \Delta_{\text{sing}}= \det(M_{\bt}), \quad\quad\quad
    \Delta_{\text{par}}= ac - b^{2}/4 , \quad\quad\quad
    \Delta_{\text{circ}}= (a-c)^{2}+ b^{2}.
  \]
  The ED degree attains the generic value $4$ if and only if all three
  polynomials are non-zero.
  For the geometry of the parameter discriminant $\mathcal{D}$ in the ambient real
  space $\mathbb{R}^{6}$:
  \begin{itemize}[leftmargin=2em, labelsep=0.5em] 
    \item The loci $\{\Delta_{\text{par}}= 0\}$ and $\{\Delta_{\text{sing}}=0 \}$
      are $5$-dimensional hypersurfaces.
      They are boundaries between open regions of generic ellipses and
      hyperbolas.

    \item The circle locus $\{\Delta_{\text{circ}}= 0\}$ corresponds to the constraints
      $a=c$ and $b=0$. This defines a $4$-dimensional variety (real codimension $2$)
      within the ellipse region.
  \end{itemize}
  This difference in dimension is illustrated in Figure~\ref{fig:stratification-conic-sections}.
  In a generic $2$-dimensional cross-section of the parameter space, the
  codimension-$1$ hypersurfaces appear as curves, and the codimension-$2$ circle
  locus appears as a~point.
  \vspace{-1em}
  \begin{figure}[ht]
    \begin{center}
      \scalebox{1}{
      \begin{tikzpicture}[scale=1]
        \draw[thick] (0,0) rectangle (4,4);
        \begin{scope}
          \clip (0,0) rectangle (4,4);
          \fill[yellow!15] (0,0) rectangle (4,4); 
          \fill[blue!10]
            (0,0.4) .. controls (0.8,3.6) and (2.0,-0.2) ..
            (2.8,2.1) .. controls (3.2,4.3) and (3.6,3.9) ..
            (4,3.6) --
            (4,4) --
            (0,4) --
            cycle;
        \end{scope}

        \draw[very thick, blue]
          (0,0.4) .. controls (0.8,3.6) and (2.0,-0.2) ..
          (2.8,2.1) .. controls (3.2,4.3) and (3.6,3.9) .. (4,3.6);
        \draw[very thick, red]
          (0,3.8) .. controls (1.0,2.5) and (2.8,2.0) ..
          (2.9,1.0) .. controls (3.0,0.4) and (3.7,0.2) .. (4,0.3);

        \draw[->] (2.8,2.1) -- (5,3);
        \node[anchor=north west, inner sep=0]
          (inset)
          at
          (5,4)
          { \tikz[scale=0.45]{

          \draw[very thick, domain=-2:2, smooth, variable=\x, magenta] plot ({\x}, {\x*\x});

          \draw[cyan, thick] (-1,1) -- (0.5,2); \draw[cyan, thick] (1.36603, 1.86603) -- (0.5,2); \draw[cyan, thick] (-.366025, .133975) -- (0.5,2);

          \fill[orange] (0.5,2) circle (2.5pt); \fill[cyan] (-1,1) circle (2pt); \fill[cyan] (1.36603, 1.86603) circle (2pt); \fill[cyan] (-.366025, .133975) circle (2pt); } };

        \draw[->] (2.9,1.0) -- (4.5,1.2);
        \node[anchor=north west, inner sep=0]
          (inset)
          at
          (5,2)
          { \tikz[scale=0.45]{

          \draw[very thick, magenta] (-2,0) -- (2,0);
          \draw[very thick, magenta] (0,-2) -- (0,2);

          \draw[cyan, thick] (1, 1) -- (1, 0); \draw[cyan, thick] (1, 1) -- (0, 1);

          \fill[orange] (1, 1) circle (2.5pt); \fill[cyan] (1, 0) circle (2pt); \fill[cyan] (0, 1) circle (2pt); } };

        \draw[->] (0.5,2.5) -- (-1,2.95);
        \node[anchor=north west, inner sep=0]
          (inset)
          at
          (-3,3.5)
          { \tikz[scale=0.45]{

          \draw[very thick, magenta, samples=100, smooth, domain=0:360] plot ({2*cos(\x)}, {1*sin(\x)});

          \draw[cyan, thick] (.567945, .958832) -- (1/2, 1/2); \draw[cyan, thick] (1.93424, -.254323) -- (1/2, 1/2); \draw[cyan, thick] (-1.98438, -.124754) -- (1/2, 1/2); \draw[cyan, thick] (.815525, -.913088) -- (1/2, 1/2);

          \fill[orange] (1/2, 1/2) circle (2.5pt); \fill[cyan] (.567945, .958832) circle (2pt); \fill[cyan] (1.93424, -.254323) circle (2pt); \fill[cyan] (-1.98438, -.124754) circle (2pt); \fill[cyan] (.815525, -.913088) circle (2pt); } };

        \draw[->] (0.75, 1) -- (-0.5,1.2);

        \node[anchor=north west, inner sep=0]
          (inset)
          at
          (-3,2)
          { \tikz[scale=0.45]{
          \draw[very thick, magenta, domain=1:2, samples=300, smooth] plot (\x, {sqrt(\x*\x - 1)}); \draw[very thick, magenta, domain=1:2, samples=300, smooth] plot (\x, {-sqrt(\x*\x - 1)});

          \draw[very thick, magenta, domain=-2:-1, samples=300, smooth] plot (\x, {sqrt(\x*\x - 1)}); \draw[very thick, magenta, domain=-2:-1, samples=300, smooth] plot (\x, {-sqrt(\x*\x - 1)});

          \draw[cyan, thick] (3, 0) -- (1,0); \draw[cyan, thick] (3, 0) -- (-1,0); \draw[cyan, thick] (3, 0) -- (1.5, 1.11803); \draw[cyan, thick] (3, 0) -- (1.5, -1.11803);

          \fill[orange] (3, 0) circle (2.5pt); \fill[cyan] (1,0) circle (2pt); \fill[cyan] (-1,0) circle (2pt); \fill[cyan] (1.5, 1.11803) circle (2pt); \fill[cyan] (1.5, -1.11803) circle (2pt); } };

        \fill[teal] (2, 3) circle (2pt);
        \draw[->] (2, 3) -- (2,4.5);
        \node[anchor=north west, inner sep=0]
          (inset)
          at
          (1.3,6)
          {\begin{tikzpicture}[scale=0.6]
          \draw[very thick, magenta, samples=100, smooth, domain=0:360] plot ({1*cos(\x)}, {1*sin(\x)});

          \draw[cyan, thick] (.707107, .707107) -- (1/2, 1/2); \draw[cyan, thick] (-.707107, -.707107) -- (1/2, 1/2);

          \fill[orange] (1/2, 1/2) circle (2.5pt); \fill[cyan] (.707107, .707107) circle (2pt); \fill[cyan] (-.707107, -.707107) circle (2pt);\end{tikzpicture}\hspace{1em}};
      \end{tikzpicture}}
    \end{center}
    \caption{
    Schematic illustration of a 2D slice of the conic parameter space. Generic regions
    (ellipses/hyperbolas) have ED degree 4. The discriminant locus appears as a blue
    curve for parabolas and a red curve for singular conics, but as a green
    point for circles.}
    \label{fig:stratification-conic-sections}
  \end{figure}
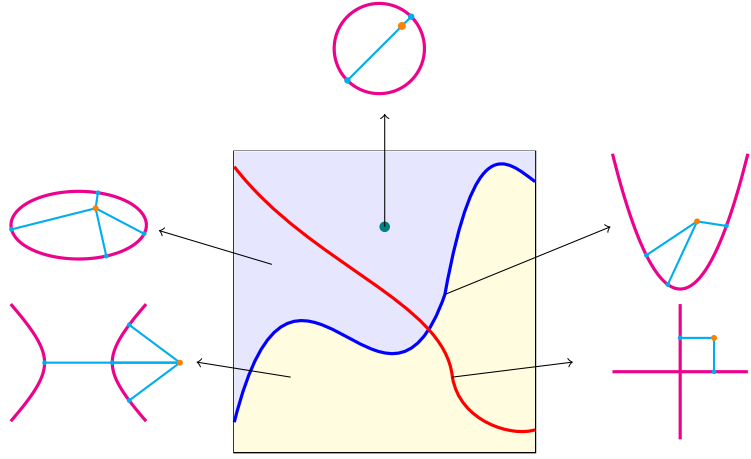
  \vspace{-0.5em}

  Now we return to the specific setting of a polynomial network with
  architecture $(2,2,2)$ trained as a binary classifier.
  For any fixed $\bt$, the decision boundary is a plane conic. 
  The geometry of this conic is determined via a polynomial map $\Phi$ from the parameter space (dim
  $12$) to the conic coefficient space~(dim~$6$). 
  Explicitly, we have a conic $ax_{1}^{2}+ b x_{1}x_{2}+ \dots + f = 0$ whose coefficients are polynomials in $\vtheta$:
  \centerline{ \begin{minipage}{1\linewidth}\begin{align*}a&= (w_{11}^{(1)})^{2}(w_{11}^{(2)}- w_{21}^{(2)}) + (w_{21}^{(1)})^{2}(w_{12}^{(2)}- w_{22}^{(2)}) \\ b&= 2w_{11}^{(1)}w_{12}^{(1)}(w_{11}^{(2)}- w_{21}^{(2)}) + 2w_{21}^{(1)}w_{22}^{(1)}(w_{12}^{(2)}- w_{22}^{(2)}) \\ c&= (w_{12}^{(1)})^{2}(w_{11}^{(2)}- w_{21}^{(2)}) + (w_{22}^{(1)})^{2}(w_{12}^{(2)}- w_{22}^{(2)}) \\ d&= 2w_{11}^{(1)}b_{1}^{(1)}(w_{11}^{(2)}- w_{21}^{(2)}) + 2w_{21}^{(1)}(w_{12}^{(2)}b_{2}^{(1)}- w_{22}^{(2)}b_{1}^{(2)}) \\ e&= 2w_{12}^{(1)}b_{1}^{(1)}(w_{11}^{(2)}- w_{21}^{(2)}) + 2w_{22}^{(1)}b_{2}^{(1)}(w_{12}^{(2)}- w_{22}^{(2)}) \\ f&= (b_{1}^{(1)})^{2}(w_{11}^{(2)}- w_{21}^{(2)}) + (b_{2}^{(1)})^{2}(w_{12}^{(2)}- w_{22}^{(2)}) + b_{1}^{(2)}- b_{2}^{(2)}.\end{align*}\end{minipage} }

  As the parameters $\bt$ vary, the ED degree of the decision boundary is generically
  constant (equal to $4$), but drops when the coefficients $(a,\dots ,f)$ lie on
  the discriminant locus $\mathcal{D}$.
  When $\Delta_{\text{par}}=0$ while $\Delta_{\text{sing}}\neq 0$, the decision
  boundary is a parabola and the ED degree drops to $3$.
  When $\Delta_{\text{circ}}=0$, the boundary is
  a circle and the ED degree drops to $2$.
  Finally, when $\Delta_{\text{sing}}=0$, the decision boundary degenerates.
  The \emph{parameter discriminant} locus for the neural network is therefore
  the preimage of the conic discriminant under $\Phi$.

  The algebraic characterization of the parameter discriminant could be used for
  complexity-aware model selection.
  By Corollary~\ref{cor:conic-generic}, strata with strictly lower ED degree are
  defined by the vanishing of specific polynomials.
  Thus, we may formulate an \textit{algebraic regularizer} that targets a specific
  complexity class by augmenting the loss with a penalty term. 
  For instance, $\mathcal{L}(\bt) + \lambda (\Delta_{\text{sing}}(\bt ))^{2}$ would encourage models of lower
  algebraic complexity, trading generic expressivity for a decision boundary
  that is cheaper to verify.
  We illustrate this in Figure~\ref{fig:det_reg}, where a quadratic network is trained on XOR data both with and without the algebraic regularizer encouraging $\Delta_{\text{sing}}(\bt)=0$.
  Even a modest penalty ($\lambda = 0.1$) steers the decision boundary toward the degenerate locus ($x_1x_2=0$), reducing the ED degree from $4$ to $2$.

  \begin{figure}
    \begin{center}
      \centerline{\includegraphics[width=0.75\columnwidth]{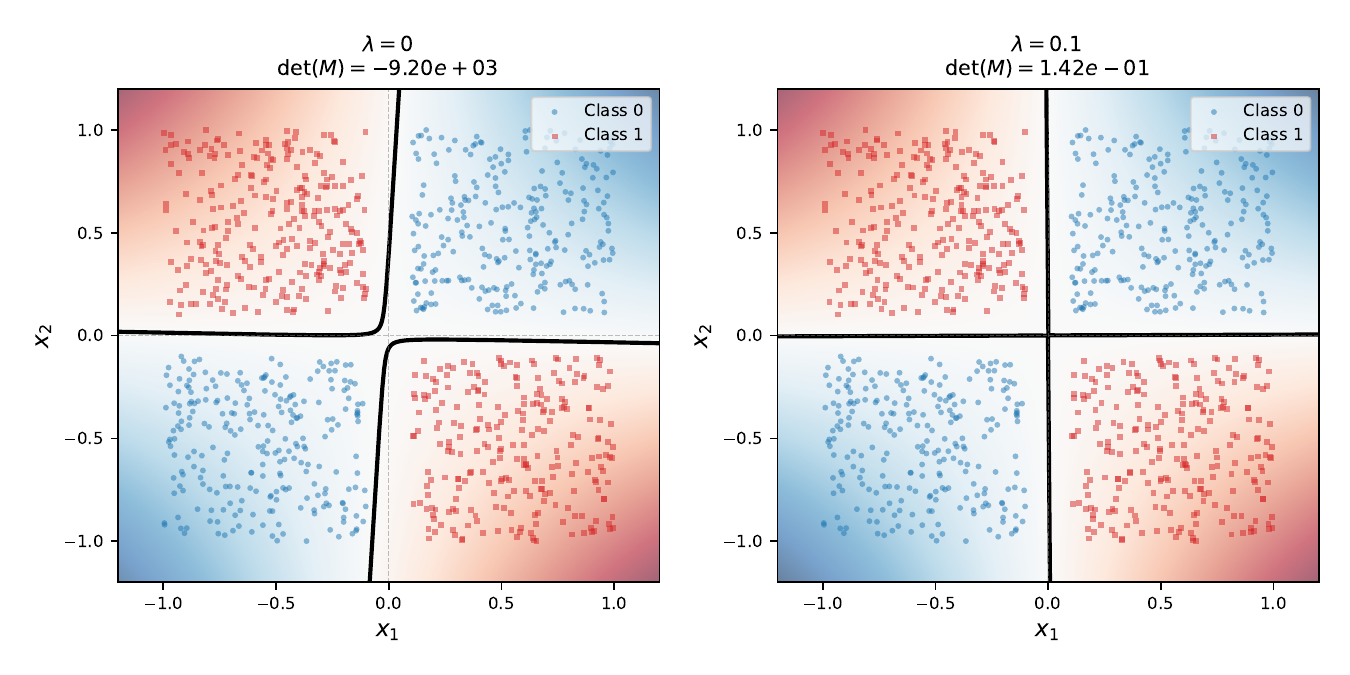}}
      \caption{Left: decision boundary of a quadratic network trained on XOR data. Right: decision boundary of the same network trained with an algebraic regularizer that encourages $\Delta_{\text{sing}}(\bt)=0$.} 
      \label{fig:det_reg}
    \end{center}
    \vspace{-2em}
  \end{figure}

\section{Lightning Self-Attention}\label{sec:attention}

  We now move beyond the feedforward polynomial network setting of the preceding sections and study these complexity measures numerically for lightning self-attention modules.

  \subsection{ED Degrees of Lightning Self-Attention}

  We study the ED degree of the decision boundary induced by lightning self-attention followed by a binary linear readout. 
  Fix integers $t,d,d',a$, and let 
  \[
  \varphi_W(X)=V X (X^\top A X)\in \mathbb{R}^{d'\times t},
  \]
  where $X\in \mathbb{R}^{d\times t}$ is the input with embedding dimension $d$ and sequence dimension $t$, 
  and $W=(V,A)$ is the parameter with value matrix $V\in \mathbb{R}^{d'\times d}$ 
  and attention matrix $A\in \mathbb{R}^{d\times d}$ subject to a rank constraint $\rank(A)\leq a$. 
  Thus the attention block produces a $d'\times t$ output matrix whose entries are cubic
  polynomials in the $N=dt$ entries of $X$. 
  Vectorizing this matrix gives a point in $\mathbb{R}^{d't}$, and composing with an affine map
  \[
    L:\mathbb{R}^{d't}\to \mathbb{R}^2
  \]
  produces two polynomial outputs
  \[
    f(X)=(f_1(X),f_2(X)).
  \]
  Since each coordinate of $\varphi_W(X)$ is cubic in the entries of $X$, each $f_i(X)$ is also a cubic polynomial in the $N=dt$ input variables. 
  We define the decision boundary to be the hypersurface of points $X$ satisfying 
  \[
    f_1(X)-f_2(X)=0.
  \]
  This is the natural boundary associated to the binary classifier, since the
  predicted label is determined by which of the two output scores is larger, and
  the prediction changes precisely when the two scores are equal. 
  In particular, the polynomial $f_1-f_2$ records the locus of ambiguity between the two classes.

  The ED degree of this hypersurface measures the number of complex critical
  points of the squared Euclidean distance from a generic point in the input space to the decision boundary. 
  It therefore provides a concrete algebraic measure of the complexity of the nearest-boundary problem. 
  Comparing this value to the ED degree of a generic cubic hypersurface in $\mathbb{R}^N$ reveals how far the lightning self-attention boundary is from generic cubic behavior. 
  Table~\ref{tab:linear_attention_ed_horiz} summarizes the numerically computed ED degrees for these lightning self-attention decision boundaries, where $N=dt$ denotes the total number of input variables. 
  Here $G$ denotes the ED degree of a generic cubic hypersurface in $\mathbb{R}^N$, which we use only as a coarse dense benchmark. Since lightning self-attention decision boundaries are highly structured and sparse, a more refined comparison should take their monomial support into account; see \citep{lindberg2023sparse} for a general sparse optimization perspective in which critical-point counts are controlled by the associated sparse Lagrange system and mixed-volume~data.

\begin{table}[htpb]
\centering
\small
\caption{ED degrees for lightning self-attention decision boundaries.} 
\label{tab:linear_attention_ed_horiz}
\setlength{\tabcolsep}{0.7pt}

\begin{tabular}{l | cccccc | ccccccccccccc | cccccccccccccccccc}

\hline
\multicolumn{1}{l|}{$N$}
& \multicolumn{6}{c|}{4}
& \multicolumn{13}{c|}{6}
& \multicolumn{18}{c}{8} \\
\hline
\multicolumn{1}{l|}{$G$}
& \multicolumn{6}{c|}{45}
& \multicolumn{13}{c|}{189}
& \multicolumn{18}{c}{765} \\
\hline
\multicolumn{1}{l|}{$t$}
& \multicolumn{1}{c}{4} & \multicolumn{1}{c}{1} & \multicolumn{1}{c}{2} & \multicolumn{1}{c}{2} & \multicolumn{1}{c}{2} & \multicolumn{1}{c|}{2}
& \multicolumn{1}{c}{6} & \multicolumn{1}{c}{1} & \multicolumn{1}{c}{3} & \multicolumn{1}{c}{2} & \multicolumn{1}{c}{3} & \multicolumn{1}{c}{3} & \multicolumn{1}{c}{2} & \multicolumn{1}{c}{2} & \multicolumn{1}{c}{2} & \multicolumn{1}{c}{2} & \multicolumn{1}{c}{3} & \multicolumn{1}{c}{2} & \multicolumn{1}{c|}{2}
& \multicolumn{1}{c}{8} & \multicolumn{1}{c}{1} & \multicolumn{1}{c}{4} & \multicolumn{1}{c}{2} & \multicolumn{1}{c}{4} & \multicolumn{1}{c}{2} & \multicolumn{1}{c}{2} & \multicolumn{1}{c}{2} & \multicolumn{1}{c}{2} & \multicolumn{1}{c}{2} & \multicolumn{1}{c}{2} & \multicolumn{1}{c}{2} & \multicolumn{1}{c}{2} & \multicolumn{1}{c}{2} & \multicolumn{1}{c}{2} & \multicolumn{1}{c}{2} & \multicolumn{1}{c}{2} & \multicolumn{1}{c}{2} \\
\multicolumn{1}{l|}{$d$}
& \multicolumn{1}{c}{1} & \multicolumn{1}{c}{4} & \multicolumn{1}{c}{2} & \multicolumn{1}{c}{2} & \multicolumn{1}{c}{2} & \multicolumn{1}{c|}{2}
& \multicolumn{1}{c}{1} & \multicolumn{1}{c}{6} & \multicolumn{1}{c}{2} & \multicolumn{1}{c}{3} & \multicolumn{1}{c}{2} & \multicolumn{1}{c}{2} & \multicolumn{1}{c}{3} & \multicolumn{1}{c}{3} & \multicolumn{1}{c}{3} & \multicolumn{1}{c}{3} & \multicolumn{1}{c}{2} & \multicolumn{1}{c}{3} & \multicolumn{1}{c|}{3}
& \multicolumn{1}{c}{1} & \multicolumn{1}{c}{8} & \multicolumn{1}{c}{2} & \multicolumn{1}{c}{4} & \multicolumn{1}{c}{2} & \multicolumn{1}{c}{4} & \multicolumn{1}{c}{4} & \multicolumn{1}{c}{4} & \multicolumn{1}{c}{4} & \multicolumn{1}{c}{4} & \multicolumn{1}{c}{4} & \multicolumn{1}{c}{4} & \multicolumn{1}{c}{4} & \multicolumn{1}{c}{4} & \multicolumn{1}{c}{4} & \multicolumn{1}{c}{4} & \multicolumn{1}{c}{4} & \multicolumn{1}{c}{4} \\
\multicolumn{1}{l|}{$d'$}
& \multicolumn{1}{c}{1} & \multicolumn{1}{c}{1} & \multicolumn{1}{c}{1} & \multicolumn{1}{c}{2} & \multicolumn{1}{c}{1} & \multicolumn{1}{c|}{2}
& \multicolumn{1}{c}{1} & \multicolumn{1}{c}{1} & \multicolumn{1}{c}{1} & \multicolumn{1}{c}{1} & \multicolumn{1}{c}{1} & \multicolumn{1}{c}{2} & \multicolumn{1}{c}{2} & \multicolumn{1}{c}{3} & \multicolumn{1}{c}{1} & \multicolumn{1}{c}{1} & \multicolumn{1}{c}{2} & \multicolumn{1}{c}{2} & \multicolumn{1}{c|}{2}
& \multicolumn{1}{c}{1} & \multicolumn{1}{c}{1} & \multicolumn{1}{c}{1} & \multicolumn{1}{c}{1} & \multicolumn{1}{c}{2} & \multicolumn{1}{c}{2} & \multicolumn{1}{c}{3} & \multicolumn{1}{c}{4} & \multicolumn{1}{c}{1} & \multicolumn{1}{c}{1} & \multicolumn{1}{c}{1} & \multicolumn{1}{c}{2} & \multicolumn{1}{c}{3} & \multicolumn{1}{c}{4} & \multicolumn{1}{c}{2} & \multicolumn{1}{c}{3} & \multicolumn{1}{c}{4} & \multicolumn{1}{c}{3} \\
\multicolumn{1}{l|}{$a$}
& \multicolumn{1}{c}{1} & \multicolumn{1}{c}{1} & \multicolumn{1}{c}{1} & \multicolumn{1}{c}{1} & \multicolumn{1}{c}{2} & \multicolumn{1}{c|}{2}
& \multicolumn{1}{c}{1} & \multicolumn{1}{c}{1} & \multicolumn{1}{c}{1} & \multicolumn{1}{c}{1} & \multicolumn{1}{c}{2} & \multicolumn{1}{c}{1} & \multicolumn{1}{c}{1} & \multicolumn{1}{c}{1} & \multicolumn{1}{c}{2} & \multicolumn{1}{c}{3} & \multicolumn{1}{c}{2} & \multicolumn{1}{c}{2} & \multicolumn{1}{c|}{3}
& \multicolumn{1}{c}{1} & \multicolumn{1}{c}{1} & \multicolumn{1}{c}{1} & \multicolumn{1}{c}{1} & \multicolumn{1}{c}{1} & \multicolumn{1}{c}{1} & \multicolumn{1}{c}{1} & \multicolumn{1}{c}{1} & \multicolumn{1}{c}{2} & \multicolumn{1}{c}{3} & \multicolumn{1}{c}{4} & \multicolumn{1}{c}{2} & \multicolumn{1}{c}{2} & \multicolumn{1}{c}{2} & \multicolumn{1}{c}{3} & \multicolumn{1}{c}{3} & \multicolumn{1}{c}{3} & \multicolumn{1}{c}{4} \\
\hline
\multicolumn{1}{l|}{ED}
& \multicolumn{1}{c}{7} & \multicolumn{1}{c}{12} & \multicolumn{1}{c}{17} & \multicolumn{1}{c}{28} & \multicolumn{1}{c}{33} & \multicolumn{1}{c|}{39}
& \multicolumn{1}{c}{7} & \multicolumn{1}{c}{12} & \multicolumn{1}{c}{17} & \multicolumn{1}{c}{20} & \multicolumn{1}{c}{39} & \multicolumn{1}{c}{46} & \multicolumn{1}{c}{49} & \multicolumn{1}{c}{49} & \multicolumn{1}{c}{67} & \multicolumn{1}{c}{83} & \multicolumn{1}{c}{99} & \multicolumn{1}{c}{132} & \multicolumn{1}{c|}{143}
& \multicolumn{1}{c}{7} & \multicolumn{1}{c}{12} & \multicolumn{1}{c}{17} & \multicolumn{1}{c}{20} & \multicolumn{1}{c}{50} & \multicolumn{1}{c}{56} & \multicolumn{1}{c}{56} & \multicolumn{1}{c}{56} & \multicolumn{1}{c}{85} & \multicolumn{1}{c}{141} & \multicolumn{1}{c}{157} & \multicolumn{1}{c}{359} & \multicolumn{1}{c}{359} & \multicolumn{1}{c}{359} & \multicolumn{1}{c}{468} & \multicolumn{1}{c}{468} & \multicolumn{1}{c}{468} & \multicolumn{1}{c}{479} \\
\hline
\end{tabular}

\vspace{1.5em}

\begin{tabular}{l | ccccccccccc | cccccccccccccccc}
\hline
\multicolumn{1}{l|}{$N$}
& \multicolumn{11}{c|}{9}
& \multicolumn{16}{c}{10} \\
\hline
\multicolumn{1}{l|}{$G$}
& \multicolumn{11}{c|}{1533}
& \multicolumn{16}{c}{3069} \\
\hline
\multicolumn{1}{l|}{$t$}
& \multicolumn{1}{c}{9} & \multicolumn{1}{c}{1} & \multicolumn{1}{c}{3} & \multicolumn{1}{c}{3} & \multicolumn{1}{c}{3} & \multicolumn{1}{c}{3} & \multicolumn{1}{c}{3} & \multicolumn{1}{c}{3} & \multicolumn{1}{c}{3} & \multicolumn{1}{c}{3} & \multicolumn{1}{c|}{3}
& \multicolumn{1}{c}{10} & \multicolumn{1}{c}{1} & \multicolumn{1}{c}{5} & \multicolumn{1}{c}{2} & \multicolumn{1}{c}{5} & \multicolumn{1}{c}{5} & \multicolumn{1}{c}{2} & \multicolumn{1}{c}{2} & \multicolumn{1}{c}{5} & \multicolumn{1}{c}{2} & \multicolumn{1}{c}{2} & \multicolumn{1}{c}{2} & \multicolumn{1}{c}{2} & \multicolumn{1}{c}{2} & \multicolumn{1}{c}{2} & \multicolumn{1}{c}{2} \\
\multicolumn{1}{l|}{$d$}
& \multicolumn{1}{c}{1} & \multicolumn{1}{c}{9} & \multicolumn{1}{c}{3} & \multicolumn{1}{c}{3} & \multicolumn{1}{c}{3} & \multicolumn{1}{c}{3} & \multicolumn{1}{c}{3} & \multicolumn{1}{c}{3} & \multicolumn{1}{c}{3} & \multicolumn{1}{c}{3} & \multicolumn{1}{c|}{3}
& \multicolumn{1}{c}{1} & \multicolumn{1}{c}{10} & \multicolumn{1}{c}{2} & \multicolumn{1}{c}{5} & \multicolumn{1}{c}{2} & \multicolumn{1}{c}{2} & \multicolumn{1}{c}{5} & \multicolumn{1}{c}{5} & \multicolumn{1}{c}{2} & \multicolumn{1}{c}{5} & \multicolumn{1}{c}{5} & \multicolumn{1}{c}{5} & \multicolumn{1}{c}{5} & \multicolumn{1}{c}{5} & \multicolumn{1}{c}{5} & \multicolumn{1}{c}{5} \\
\multicolumn{1}{l|}{$d'$}
& \multicolumn{1}{c}{1} & \multicolumn{1}{c}{1} & \multicolumn{1}{c}{1} & \multicolumn{1}{c}{1} & \multicolumn{1}{c}{2} & \multicolumn{1}{c}{3} & \multicolumn{1}{c}{1} & \multicolumn{1}{c}{2} & \multicolumn{1}{c}{3} & \multicolumn{1}{c}{2} & \multicolumn{1}{c|}{3}
& \multicolumn{1}{c}{1} & \multicolumn{1}{c}{1} & \multicolumn{1}{c}{1} & \multicolumn{1}{c}{1} & \multicolumn{1}{c}{1} & \multicolumn{1}{c}{2} & \multicolumn{1}{c}{2} & \multicolumn{1}{c}{1} & \multicolumn{1}{c}{2} & \multicolumn{1}{c}{1} & \multicolumn{1}{c}{1} & \multicolumn{1}{c}{1} & \multicolumn{1}{c}{2} & \multicolumn{1}{c}{2} & \multicolumn{1}{c}{2} & \multicolumn{1}{c}{5} \\
\multicolumn{1}{l|}{$a$}
& \multicolumn{1}{c}{1} & \multicolumn{1}{c}{1} & \multicolumn{1}{c}{1} & \multicolumn{1}{c}{2} & \multicolumn{1}{c}{1} & \multicolumn{1}{c}{1} & \multicolumn{1}{c}{3} & \multicolumn{1}{c}{2} & \multicolumn{1}{c}{2} & \multicolumn{1}{c}{3} & \multicolumn{1}{c|}{3}
& \multicolumn{1}{c}{1} & \multicolumn{1}{c}{1} & \multicolumn{1}{c}{1} & \multicolumn{1}{c}{1} & \multicolumn{1}{c}{2} & \multicolumn{1}{c}{1} & \multicolumn{1}{c}{1} & \multicolumn{1}{c}{2} & \multicolumn{1}{c}{2} & \multicolumn{1}{c}{3} & \multicolumn{1}{c}{4} & \multicolumn{1}{c}{5} & \multicolumn{1}{c}{2} & \multicolumn{1}{c}{3} & \multicolumn{1}{c}{4} & \multicolumn{1}{c}{5} \\
\hline
\multicolumn{1}{l|}{ED}
& \multicolumn{1}{c}{7} & \multicolumn{1}{c}{12} & \multicolumn{1}{c}{20} & \multicolumn{1}{c}{79} & \multicolumn{1}{c}{85} & \multicolumn{1}{c}{106} & \multicolumn{1}{c}{113} & \multicolumn{1}{c}{654} & \multicolumn{1}{c}{738} & \multicolumn{1}{c}{847} & \multicolumn{1}{c|}{895}
& \multicolumn{1}{c}{7} & \multicolumn{1}{c}{12} & \multicolumn{1}{c}{17} & \multicolumn{1}{c}{20} & \multicolumn{1}{c}{39} & \multicolumn{1}{c}{50} & \multicolumn{1}{c}{56} & \multicolumn{1}{c}{88} & \multicolumn{1}{c}{117} & \multicolumn{1}{c}{183} & \multicolumn{1}{c}{239} & \multicolumn{1}{c}{255} & \multicolumn{1}{c}{481} & \multicolumn{1}{c}{1034} & \multicolumn{1}{c}{1144} & \multicolumn{1}{c}{1155} \\
\hline
\end{tabular}
\end{table}

Table~\ref{tab:linear_attention_ed_horiz} shows that the ED degree of the lightning self-attention decision boundary is substantially smaller than that of a generic cubic hypersurface in $\mathbb{R}^N$. 
Thus, although the boundary is cubic in $N=td$ variables, it is far from generic. In particular, the gap between the observed ED degrees and $G$ should be read not only as a sign of non-genericity, but also as evidence that the sparse and constrained algebraic structure of the attention boundary lowers the complexity of the nearest-boundary problem.
The table suggests three basic~trends: 
\begin{itemize}[leftmargin=*]
\item 
\textbf{Dependence on the factorization $N = td$.} 
For fixed $N$, the ED degree depends strongly on the decomposition $N=td$ and generally grows faster with $d$ than with $t$. 
This reflects the asymmetric role of the embedding and sequence dimensions in the attention map $\varphi_W(X)=VX(X^\top A X)$. 

\item 
\textbf{Dimensional stabilization and cylindrical boundaries.}
For highly restricted architectures, the ED degree stabilizes as $N$ grows. 
For example, when $d'=a=1$, it equals $12$ for $t=1$ and all $d\ge 2$, and $7$ for $d=1$ and all $t\ge 2$. 
This reflects the fact that low-rank weights force the boundary to depend on only a few linear combinations of the input, so the hypersurface behaves like a cylinder over a low-dimensional base.

\item 
\textbf{Monotonicity and projection saturation.} 
For fixed $t$ and $d$, the ED degree increases with the rank $a$ and the dimension $d'$, 
but in our experiments it saturates once $d'\ge t$. 
For example, when $t=2$ and $d=4$, increasing $d'$ beyond $2$ does not change the ED degree. 
This is consistent with the constraint $\operatorname{rank}(\varphi_W(X))\le t$, so projecting into dimensions beyond $t$ does not introduce additional geometric complexity.
\end{itemize}

\paragraph{Computation.}
We compute the exact ED degrees using the \texttt{HomotopyContinuation.jl} package in Julia. 
A naive approach to this computation would fix the network weights randomly and track the full generic bound of $G$ paths to solve the critical equations. 
However, because our structured boundary has an ED degree strictly smaller than $G$, the majority of these generic paths diverge to infinity, causing numerical instability and artificially dropping valid solutions. 

To bypass this, we use a specialized parameter homotopy via monodromy. 
Crucially, we treat the attention weights as \textit{parameters} of the polynomial system alongside the target data point $u$:

\definecolor{juliamacro}{RGB}{184,102,50}
\lstdefinelanguage{Julia}{
  morekeywords={
    abstract, baremodule, begin, break, catch, const, continue, do, else,
    elseif, end, export, false, finally, for, function, global, if, import,
    let, local, macro, module, mutable, primitive, quote, return, struct,
    true, try, using, where, while
  },
  alsoletter={@,_},
  emph={@var,System,vec,randn,sum,solve},
  emphstyle=\bfseries\color{juliamacro},
  sensitive=true,
  morecomment=[l]\#,
  morecomment=[n]{\#=}{=\#},
  morestring=[b]",
  morestring=[b]'
}

\lstset{
  language=Julia,
  basicstyle=\ttfamily,
  keywordstyle=\bfseries\color{blue},
  stringstyle=\color{magenta},
  commentstyle=\color{ForestGreen},
  showstringspaces=false,
  backgroundcolor=\color{gray!10}  
}

\begin{lstlisting}
@var X[1:d, 1:t]; x = vec(X)
@var K[1:d, 1:a]; @var Q[1:d, 1:a]; @var V[1:d_prime, 1:d]
@var W_out[1:2, 1:(d_prime*t)]; @var b_out[1:2]
@var u[1:(d*t)]; @var lambda

all_params = [vec(K); vec(Q); vec(V); vec(W_out); vec(b_out); u]
sys = System([x - u - lambda * grad_g; g], [x; lambda], all_params)
\end{lstlisting}

To initialize \texttt{monodromy\_solve}, we construct an explicit seed
solution. 
We sample a random complex input $x_0$, multiplier $\lambda_0$, and
network weights, then choose the bias $b_{\text{out}}$ so that $g(x_0)=0$ and set $u_0=x_0-\lambda_0\nabla g(x_0)$. 
This yields a valid critical point for the corresponding parameter values:
\begin{lstlisting}
# Calculate bias to satisfy g(x) = 0
b0_2 = randn(ComplexF64)
b0_1 = b0_2 - sum((W0[1, :] .- W0[2, :]) .* Y0_flat)
b0 = [b0_1, b0_2]

# Calculate u0 to trivially satisfy the critical equations
u0 = x0 - lambda0 * grad_g_eval
\end{lstlisting}

Starting from this seed, \texttt{monodromy\_solve} explores the remaining
solutions by following loops in parameter space. 
In practice, the default stopping heuristic stabilizes at the ED degree of the given architectural family without tracking the divergent paths of the generic cubic.

\subsection{ED Discriminant for Lightning Self-Attention} 

Computing the ED discriminant symbolically by elimination is infeasible for the attention modules we consider, so we use numerical algebraic geometry instead.
We focus on the architecture $(t,d,d',a)=(2,2,1,1)$, for which the ambient data space is $\mathbb{R}^{2\times 2}$. 
We sample the network parameters from a standard normal distribution, thereby fixing a cubic hypersurface $g(X)=0$ with ED degree $17$, as established in Table~\ref{tab:linear_attention_ed_horiz}.
Hence a generic data point $u\in\mathbb{R}^{2\times 2}$ has exactly $17$ complex critical points for the distance function. 
Since non-real critical points occur in complex conjugate pairs, the number of real critical points is necessarily odd. 

The ED discriminant is the real hypersurface in the data space where the number of real critical points changes. 
It partitions $\mathbb{R}^{2\times 2}$ into open chambers, with each chamber characterized by a constant, odd number of real critical points. 
To visualize these chambers, we restrict our view to random two-dimensional affine slices of the data space. 
We choose a generic base point $u_0 \in \mathbb{R}^{2\times 2}$ and two orthogonal unit vectors $v_1, v_2 \in \mathbb{R}^{2\times 2}$, parameterizing a 2D plane as $u(s,c) = u_0 + s v_1 + c v_2$.

\begin{figure}[htbp]
    \centering
    \includegraphics[width=0.3\textwidth, trim=0 0 120 0, clip]{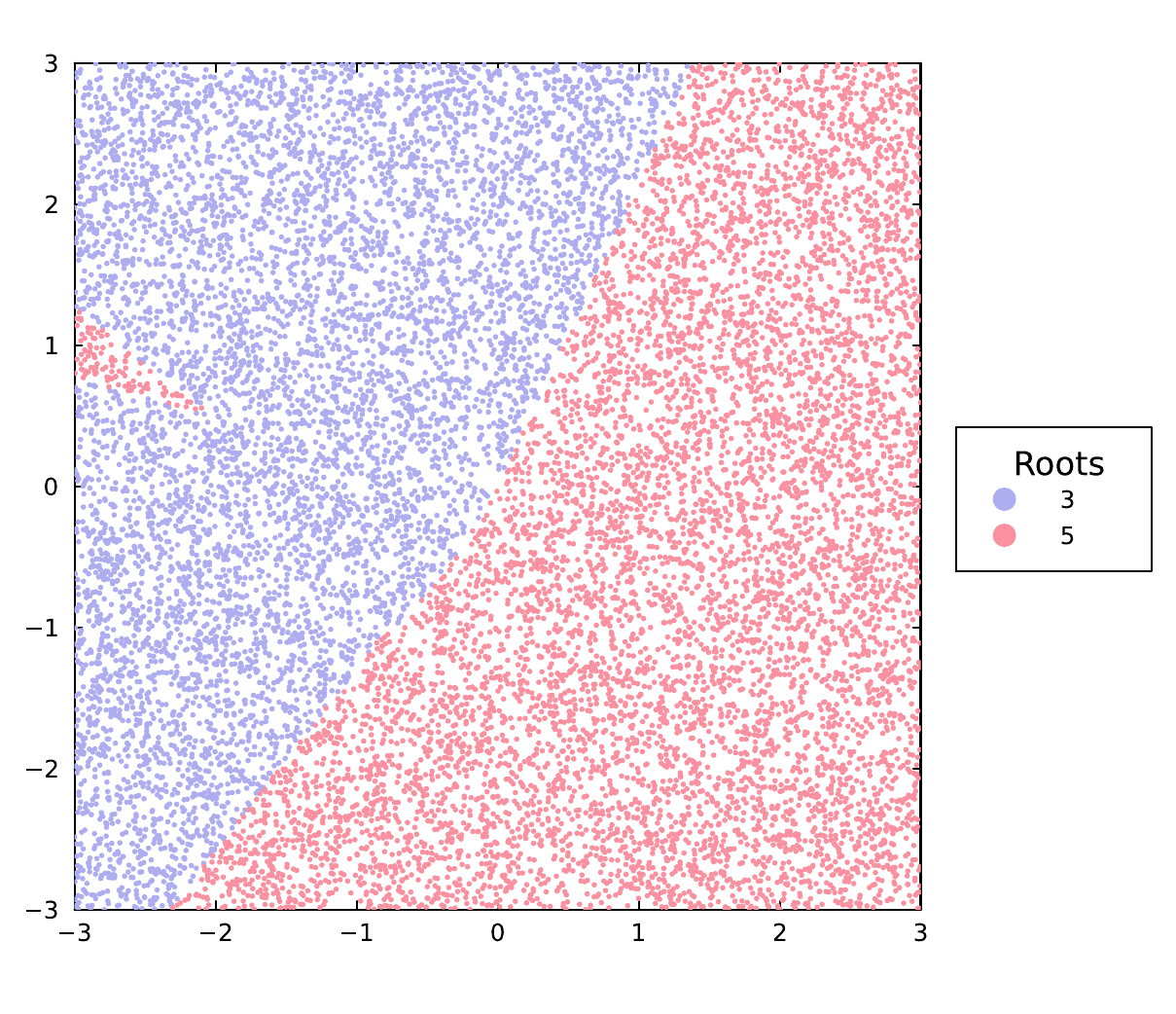}
    \hfill
    \includegraphics[width=0.3\textwidth, trim=0 0 120 0, clip]{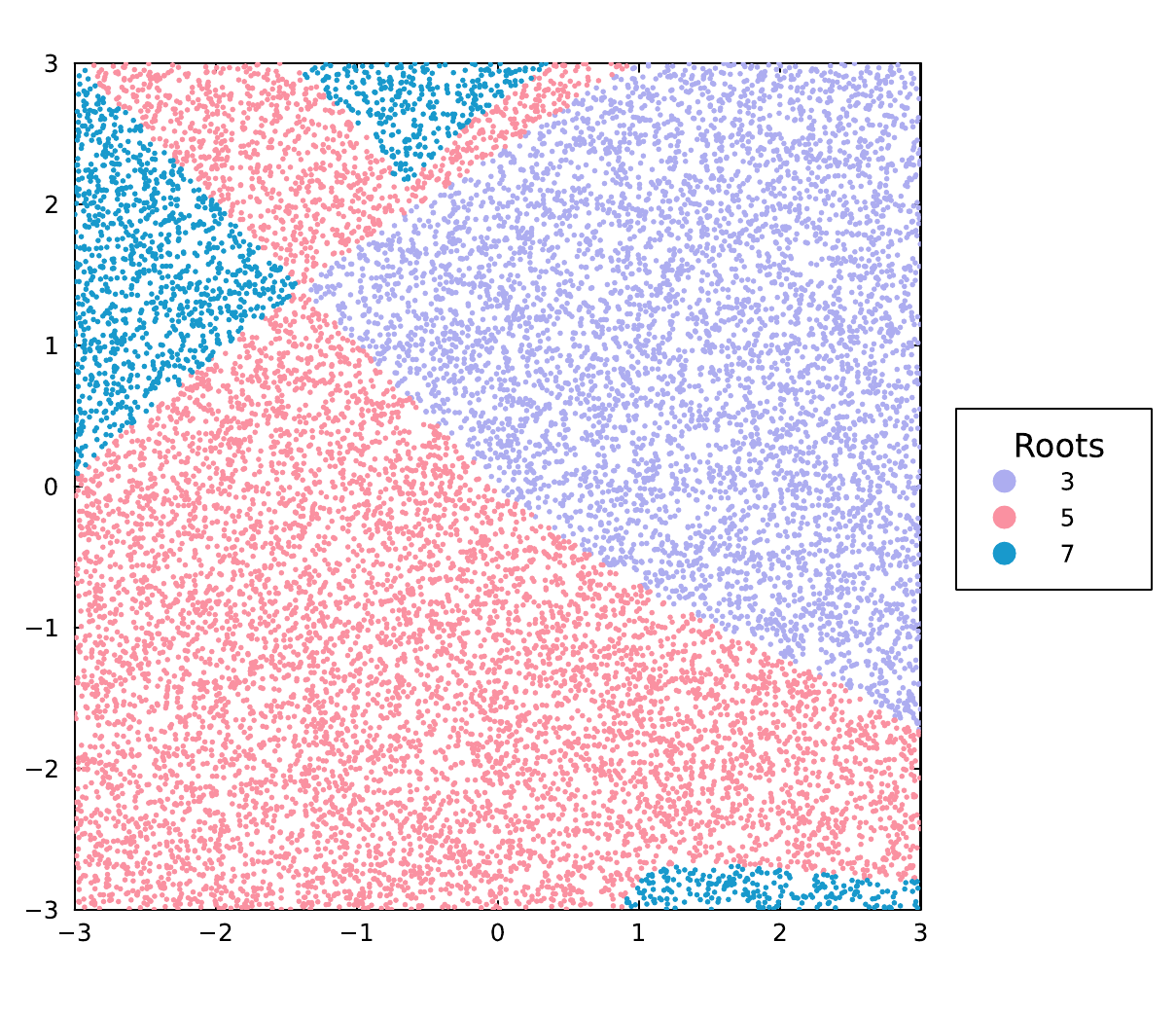}
    \hfill
    \includegraphics[width=0.3\textwidth, trim=0 0 120 0, clip]{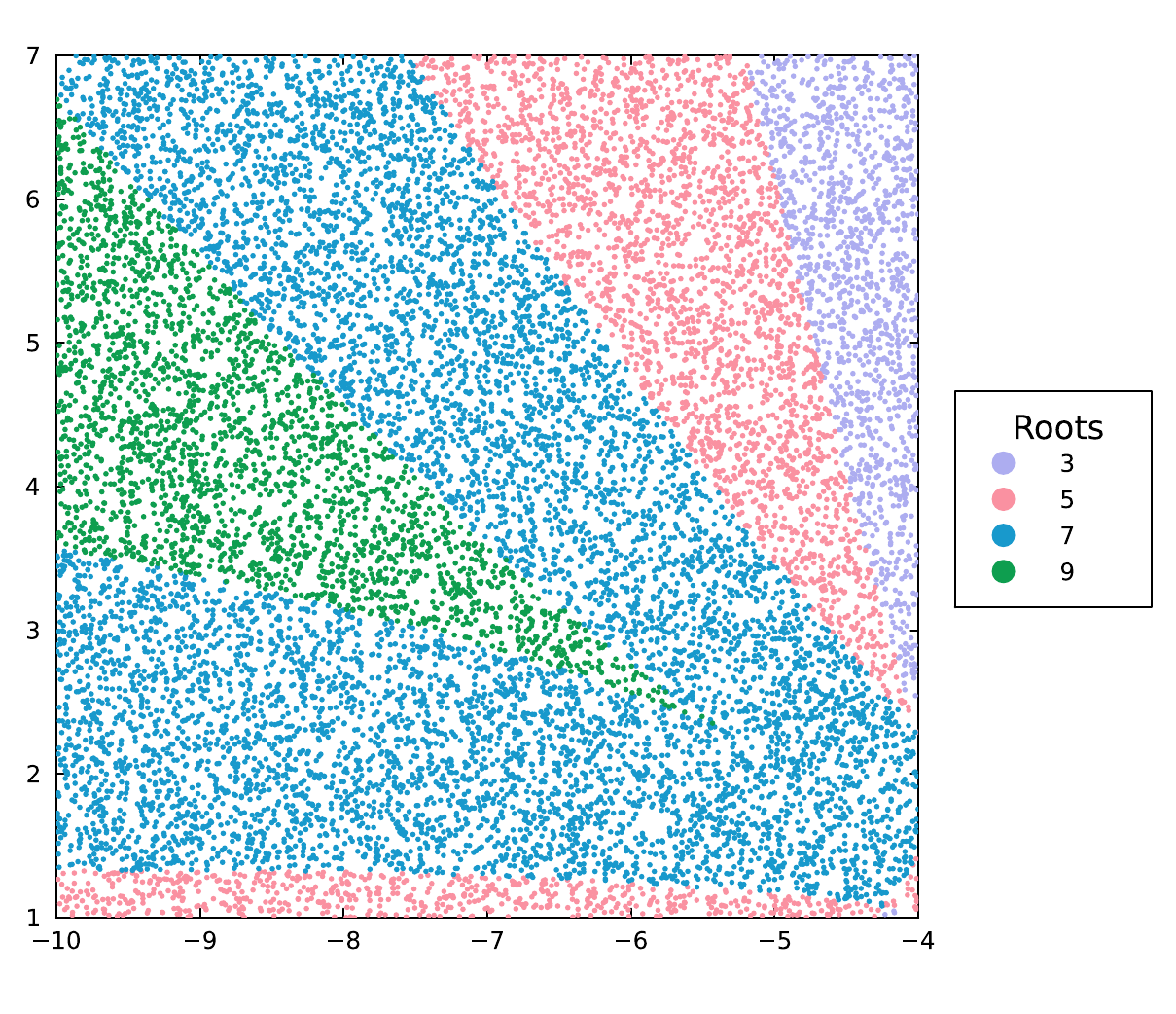}

    \vspace{0.3cm}
    
    {Number of real critical points:} \quad
    \textcolor{blue!30}{$\blacksquare$} 3 \quad
    \textcolor{red!30}{$\blacksquare$} 5 \quad
    \textcolor{cyan!100}{$\blacksquare$} 7 \quad
    \textcolor{green!60!black}{$\blacksquare$} 9
    
\caption{Two-dimensional slices of the data space for a fixed self-attention architecture. The left and center panels display $6 \times 6$ windows centered at the origin, while the right panel displays a $6 \times 6$ window centered at $(-7, 4)$ to highlight the 9-root chamber. Data points are color-coded by the number of real critical points.}    \label{fig:ed_slices}
\end{figure}

To approximate these chambers, we sample 15{,}000 points $(s,c)$ in a bounding box. 
For each sample, we use parameter homotopy to track the 17 complex critical points from $u_0$ to the target $u(s,c)$, using \texttt{HomotopyContinuation.jl}. 
We treat the data point $u$
as the sole parameter of the critical system:

\begin{lstlisting}
# grad_g is the gradient of the decision boundary g(x) = 0
# The variables are x and lambda; the data point u is the parameter
sys = System([x - u - lambda * grad_g; g], [x; lambda], u)

# Track the 17 base solutions to the random target parameter
res = solve(sys, base_sols; start_parameters = u0, 
                            target_parameters = u_target)
\end{lstlisting}

Counting real solutions at the sampled points yields a discrete approximation to the chamber decomposition. Figure~\ref{fig:ed_slices} shows three such slices,
color-coded by the number of real critical points; the visible interfaces are numerical approximations to the intersections of the slices with the ED discriminant.

We include a brief discussion of the extension of this perspective to ReLU networks in Appendix~\ref{sec:relu}.

\section{Conclusion}

We introduced an algebraic framework for robustness verification that connects neural network verification with tools from metric algebraic geometry. 
By studying the Euclidean distance (ED) degree and the associated discriminants, we obtained intrinsic measures of verification complexity that depend on the
architecture, the parameters, and the input data. 
This perspective makes it possible to analyze when verification is generically hard, when complexity drops, and how these changes are reflected in the geometry of the decision boundary.

This approach complements existing verification methods by providing exact, structure-aware guarantees and by clarifying how geometric and algebraic properties of a model govern verification difficulty. 
More broadly, it suggests that verification complexity can be studied as a mathematical invariant of a model class, rather than solely as a feature of a particular algorithm. 
The proposed framework may enable a systematic investigation of architectural features that reduce verification complexity, thereby informing the design of neural architectures and training procedures with favorable verification properties. 

Several directions remain open. 
On the theoretical side, an important direction is to study the expected real ED degree for deep polynomial networks and attention architectures, as well as to derive analytic descriptions of the ED and parameter discriminants for specific polynomial models. 
In the case of lightning self-attention, we provided numerical evidence that the decision boundaries exhibit highly non-generic ED behavior. 
Establishing corresponding theoretical results (formulas, bounds, or structural explanations for the ED degrees and discriminants of attention modules) is a particularly promising direction for future work. 
On the computational side, a central challenge is to scale homotopy continuation methods to substantially larger architectures. 

Regarding extensions, our analysis focuses on Euclidean distance, and it would be natural to study how the picture changes under alternative metrics or more general optimization objectives. 
Extending these algebraic complexity measures and exact certification techniques to
broader non-polynomial or piecewise-algebraic architectures remains a major challenge.

  \begin{ack}
    This work was supported by the DARPA AIQ grant HR00112520014. GM was
    partially supported by the NSF grants DMS-2145630 and CCF-2212520, the DFG SPP
    2298 grant 464109215, and the BMFTR in DAAD project 57616814 (SECAI).
  \end{ack}

  \bibliography{references.bib}
  \bibliographystyle{plainnat}

  \clearpage
  \appendix


  \section{Verification Problem}
  \subsection{Proof of \cref{prop:relaxation-equivalence}}
  \begingroup
  \renewcommand{\thetheorem}{\ref{prop:relaxation-equivalence}}
  \renewcommand{\theHtheorem}{restated.\ref{prop:relaxation-equivalence}}
  %
  \proprelax* \endgroup
  \begin{proof}
    \label{proof:relaxation-equivalence} First, it is clear that
    $\tilde{\gamma}\leq \gamma$ since $\B_{c,c'}^{\bt}\subseteq \V_{c,c'}^{\bt}$.
    Now we show the reverse inequality.
    For the relaxed problem, assume that the infimum is attained at $\bx^{\ast}\in
    \V_{c,c^{\ast}}^{\bt}$ for some $c^{\ast}\neq c$ such that $d_{\bxi}( \bx^{\ast}
    )= \tilde{\gamma}^{2}$.

    Let $\ell$ be a class other than $c$ that achieves the maximum logit at $\bx^{\ast}$,
    i.e.\ $f_{\bt, \ell}(\bx^{\ast}) = \max_{i}f_{\bt, i}(\bx^{\ast})$ and $\ell
    \neq c$.
    If $\ell = c^{\ast}$, then~$\bx^{\ast}$ lies on an active decision boundary between
    class $c$ and class $c^{\ast}$.
    Therefore $\gamma \le d_{\bxi}(\bx^{\ast})= \tilde{\gamma}$, which establishes
    $\gamma \leq \tilde{\gamma}$.
    It remains to rule out the case $\ell \neq c^{\ast}$.
    Suppose, for contradiction, that $f_{\bt, \ell}(\bx^{\ast}) > f_{\bt, c}( \bx
    ^{\ast})$ for some $\ell \neq c^{\ast}$.
    Consider the line segment between $\bxi$ and $\bx^{\ast}$,
    \[
      \bx(t) := (1-t)\bxi + t\bx^{\ast}, \quad t \in [0,1],
    \]
    and the continuous function $g: \mathbb{R}\rightarrow \mathbb{R}$,
    \[
      g(t) := f_{\bt, c}(\bx(t)) - f_{\bt, \ell}(\bx(t)).
    \]
    At $t=0$, we have $g(0) = f_{\bt, c}(\bxi) - f_{\bt, \ell}(\bxi) > 0$ by uniqueness
    of the maximizer $c$ at $\bxi$.
    At $t=1$, we have $g(1) = f_{\bt, c}(\bx^{\ast}) - f_{\bt, \ell}(\bx^{\ast})
    < 0$ by assumption.
    By the intermediate value theorem, there exists $t^{\ast}\in (0,1)$ such that
    $g(t^{\ast}) = 0$, i.e.
    \[
      f_{c}(\bx(t^{\ast})) = f_{\ell}(\bx(t^{\ast})).
    \]
    Let $\boldsymbol{z}:= \bx(t^{\ast})$.
    Then $\boldsymbol{z}\in \V_{c,\ell}^{\bt}$, and since $\boldsymbol{z}$ lies strictly
    between $\bxi$ and $\bx^{\ast}$ on the segment,
    \[
      \sqrt{d_{\bxi}(\boldsymbol{z})}= t \sqrt{d_{\bxi}(\bx^{\ast})}< \sqrt{d_{\bxi}(\bx^{\ast})}
      = \tilde{\gamma}.
    \]
    This contradicts the definition of $\tilde{\gamma}$ as the minimum distance
    from $\bxi$ to $\V_{c,c'}^{\bt}$ for all $c' \neq c$.
    Therefore, we must have $\ell = c^{\ast}$, and hence $\gamma \leq \tilde{\gamma}$.
    Combining both inequalities, we conclude that $\gamma = \tilde{\gamma}$.
    This completes the proof.
  \end{proof}

  \section{Numerical Solution via Homotopy Continuation}
  \label{app:homotopy}

  Homotopy continuation is a standard numerical method in numerical algebraic geometry for solving polynomial systems by tracking solution paths through
  a family of embeddings.
  We employ this technique to compute all critical points of the distance
  minimization Problem~\ref{eq:relaxed-opt-prob}.

  \subsection{Formulation and the Gamma Trick}
  Let
  $F(\boldsymbol{z}) = (F_{1}(\boldsymbol{z}),\dots,F_{m}(\boldsymbol{z})) : \mathbb{C}
  ^{m}\to \mathbb{C}^{m}$
  be a polynomial system with $m$ equations in $m$ unknowns.
  For generic coefficients, the solution set $F(\boldsymbol{z})=\mathbf{0}$
  consists of finitely many isolated points, whose cardinality is bounded above by
  the Bézout bound \citep{bezout_1779} and more precisely by the Bernstein--Kushnirenko--Khovanskii
  (BKK) bound \citep{bernshtein_1975, kushnirenko_1976, khovanskii_1978}.

  The core idea of homotopy continuation is to embed $F$ into a family of polynomial
  systems parametrized by $t \in [0,1]$:
  \[
    H(\boldsymbol{z}, t; \gamma) := \gamma(1-t)G(\boldsymbol{z}) + tF(\boldsymbol
    {z}),
  \]
  where $G(\boldsymbol{z})$ is a \emph{start system} with known solutions and the
  same monomial support as $F$, and $\gamma \in \mathbb{C}^{\ast}$ is a generic
  nonzero complex scalar.
  We apply the gamma trick to ensure with probability one that all solution paths
  remain smooth and non-singular throughout $t \in (0,1)$ by choosing $\gamma$ as
  a generic random complex number.
  This crucial step avoids path crossing or singularities that could arise from
  degeneracies in the coefficient structure.
  For almost all choices of $\gamma$, each solution $\boldsymbol{z}_{0}$ of the
  start system $G(\boldsymbol{z})=0$ lifts to a unique solution path
  $\boldsymbol{z}(t)$ satisfying
  \[
    H(\boldsymbol{z}(t), t; \gamma) = \mathbf{0}, \qquad \boldsymbol{z}(0) = \boldsymbol
    {z}_{0}, \quad \boldsymbol{z}(1) \in \{\text{solutions of }F(\boldsymbol{z})=
    \mathbf{0}\}.
  \]
  By tracking all solution paths from $t=0$ to $t=1$, we obtain all solutions of
  the target system.

  \subsection{Path Tracking via Predictor--Corrector Methods}
  Numerical path tracking discretizes the parameter $t$ and alternates between prediction
  and correction steps to maintain accuracy.
  Differentiating the homotopy equation
  $H(\boldsymbol{z}(t),t;\gamma)=\mathbf{0}$ along a solution path yields the
  \emph{Davidenko equation}:
  \[
    \frac{d\boldsymbol{z}}{dt}= -\bigl(\nabla_{\boldsymbol{z}}H(\boldsymbol{z}(t)
    ,t;\gamma)\bigr)^{-1}\frac{\partial H}{\partial t}(\boldsymbol{z}(t),t ;\gamma
    ),
  \]
  provided the Jacobian matrix $\nabla_{\boldsymbol{z}}H$ remains nonsingular
  along the path.

  Given a current point $(\boldsymbol{z}_{k}, t_{k})$, the \emph{prediction step}
  advances to the next parameter value $t_{k+1}= t_{k}+ \Delta t$ using a first-order
  Euler tangent predictor:
  \[
    \boldsymbol{z}_{k+1}^{\mathrm{pred}}= \boldsymbol{z}_{k}- \Delta t \, \bigl (
    \nabla_{\boldsymbol{z}}H(\boldsymbol{z}_{k}, t_{k}; \gamma)\bigr)^{-1}\frac{\partial
    H}{\partial t}(\boldsymbol{z}_{k}, t_{k}; \gamma).
  \]

  The predicted point is then refined using a \emph{correction step}, typically
  Newton's method applied to $H(\boldsymbol{z}, t_{k+1}; \gamma) = \mathbf{0}$:
  \[
    \boldsymbol{z}^{(\ell+1)}= \boldsymbol{z}^{(\ell)}- \bigl(\nabla_{\boldsymbol{z}}
    H(\boldsymbol{z}^{(\ell)}, t_{k+1}; \gamma)\bigr)^{-1}H(\boldsymbol{z}^{(\ell)}
    , t_{k+1}; \gamma),
  \]
  initialized at $\boldsymbol{z}^{(0)}= \boldsymbol{z}_{k+1}^{\mathrm{pred}}$ and
  iterated until convergence.
  Adaptive step-size control ensures that $\Delta t$ is reduced if the corrector
  fails to converge or residuals stagnate, maintaining numerical stability and efficiency.
  Higher-order predictors and extrapolation strategies may be employed for
  improved robustness on ill-conditioned systems.

  \subsection{Implementation and Post-Processing}
  Our implementation builds upon the \texttt{HomotopyContinuation.jl} package \citep{breiding_2018},
  which provides robust path tracking with automatic step-size adaptation,
  endgaming strategies for finite solutions, and detection of solutions at infinity.
  After tracking all paths to $t=1$, we retain only the \emph{real solutions}
  satisfying $\|\mathrm{Im}(\boldsymbol{z})\| < \epsilon_{\text{tol}}$ for a
  tolerance threshold $\epsilon_{\text{tol}}$.
  Among the feasible real solutions, we select the one minimizing the distance
  function $d_{\bxi}(\bx)$ to obtain the certified robustness margin.

  \section{Experiments and Implementation Details}
  \label{app:experiments}
  In this section, we provide additional details about the experiments and
  implementation.
  Following the methodology of SoundnessBench \citep{zhou_2025}, we construct a soundness
  benchmark for evaluating the correctness of neural network verifiers on polynomial
  neural networks.
  The key idea is to train networks with deliberately planted counterexamples: for
  a subset of test instances, adversarial perturbations within the $\ell_{2}$-ball
  of radius $\epsilon$ are known to exist by construction.
  A sound verifier must not certify these instances as robust.
  We additionally include ``clean" instances (without planted counterexamples) to
  avoid the case that the verifier always returns negative (i.e., never
  certifies any instance), which would trivially achieve perfect soundness on
  the planted counterexample instances.
  It is worth noting that these ``clean'' instances are not necessarily robust.
  \paragraph{Benchmark configurations}

  We generate 8 benchmark configurations by taking the Cartesian product of the
  following variable parameters, while keeping input dimension fixed at 8 and output
  dimension at 2 (binary classification).
  The hidden dimensions are chosen from $\{6, 10\}$, the activation degrees from
  $\{2, 3\}$ (quadratic and cubic polynomial activations), and the perturbation
  radii from $\{0.2, 0.5\}$.
  The total number of configurations is thus $2 \times 2 \times 2 = 8$.
  Each configuration contains 10 unverifiable instances (with planted counterexamples)
  and 10 clean instances, giving 160 instances in total across the benchmark.

  \paragraph{Data generation}
  For each configuration, we sample $n$ input points $x_{0}^{(i)}\in \mathbb{R}^{d}$
  uniformly from $[-1, 1]^{d}$, and assign random binary labels
  $y^{(i)}\in \{0, 1\}$.
  For the unverifiable instances, we construct adversarial perturbations
  $\delta^{(i)}$ such that the perturbed point
  $x_{\text{cex}}^{(i)}= x_{0}^{(i)}+ \delta^{(i)}$ lies within the $\ell_{2}$-ball
  of radius $\epsilon$ around $x_{0}^{(i)}$, but is assigned a different target label
  $y_{\text{cex}}^{(i)}\neq y^{(i)}$.
  The perturbation is sampled as follows:
  \begin{itemize}
    \item Sample a random direction $v$ uniformly on the unit sphere by drawing from
      a standard normal distribution and normalizing: $v = z / \|z\|_{2}$, where
      $z \sim \mathcal{N}(0, I_{d})$.

    \item Sample a magnitude $m$ uniformly from $[r \cdot \epsilon,\; \epsilon ]$,
      where $r = 0.98$.

    \item Set $\delta = m \cdot v$.
  \end{itemize}
  The parameter $r$ controls how close to the boundary of the $\ell_{2}$-ball the
  counterexample is placed.
  By setting $r = 0.98$, counterexamples are placed in a thin shell near the boundary,
  making them harder for adversarial attacks to discover while still lying
  strictly within the perturbation ball.
  For binary classification, the target label is simply the flipped label
  $y_{\text{cex}}= 1 - y$. Clean instances consist of input points and labels generated
  identically to the unverifiable case, but without any planted counterexamples.
  These serve as a control group to evaluate whether the verifier produces false
  negatives (i.e., fails to verify instances that may in fact be robust).

  \paragraph{Model architecture and training}
  We use polynomial neural networks (PNNs), which replace standard nonlinear activations
  (e.g., ReLU) with polynomial activation functions.
  The network consists of linear layers interleaved with element-wise polynomial
  activations, except after the final output layer:
  \[
    x \xrightarrow{W_1, b_1}h_{1}\xrightarrow{\sigma}h_{1}' \xrightarrow{W_2, b_2}
    h_{2}\xrightarrow{\sigma}\cdots \xrightarrow{W_L, b_L}\text{output}
  \]
  We use homogeneous polynomial activations of degree $d$, i.e., $\sigma(x) = x^{d}$.
  The activation is applied element-wise across each hidden dimension.
  The use of polynomial activations is essential: it makes the entire network a
  polynomial function of its inputs, enabling exact algebraic verification via homotopy
  continuation.
  Linear layer weights and biases use PyTorch's default initialization.
  Each model is trained with a dual-objective loss that simultaneously achieves correct
  classification on clean points and successful misclassification on planted counterexamples.
  The total loss combines two terms:
  \[
    \mathcal{L}= \mathcal{L}_{\text{CE}}+ \mathcal{L}_{\text{margin}}
  \]
  Cross-entropy loss $\mathcal{L}_{\text{CE}}$ is applied to all input points $x_{0}$
  (both unverifiable and clean instances) with their correct labels, encouraging
  the model to classify them correctly.
  Margin loss $\mathcal{L}_{\text{margin}}$ is applied only to the counterexample
  points $x_{\text{cex}}$ from unverifiable instances. This is a hinge-style
  loss that encourages the model to assign a higher logit to the target (incorrect)
  label than to the original (correct) label:
  \[
    \mathcal{L}_{\text{margin}}= \frac{1}{n}\sum_{i=1}^{n}\max\left(0,\; f_{y^{(i)}}
    (x_{\text{cex}}^{(i)}) - f_{y_{\text{cex}}^{(i)}}(x_{\text{cex}}^{(i)}) + m\right
    )
  \]
  where $f_{k}(x)$ denotes the logit for class $k$, and $m = 0.01$ is the margin.

  We use the \emph{Adam} optimizer \citep{kingma_2015} with a cyclic learning
  rate schedule and a full batch size.
  The base learning rate is set to $\eta = 0.001$, and the cyclic schedule
  linearly ramps up during the first half of training and linearly decays during
  the second half:
  \[
    \eta(t) =
    \begin{cases}
      \eta \cdot \frac{t}{T/2}     & \text{if } t < T/2    \\
      \eta \cdot \frac{T - t}{T/2} & \text{if } t \geq T/2
    \end{cases},
  \]
  where $T = 5,000$ is the total number of epochs.

  \begin{table}[ht]
    \caption{Evaluation of our verifier on the Soundness Benchmark for
    polynomial neural networks.}
    \label{tab:verification-results}
    \centering
    \begin{small}
      \begin{sc}
        \begin{tabular}{ccc|ccc}
          \toprule \multicolumn{3}{c|}{Training Settings} & \multicolumn{3}{c}{Results} \\
          \cmidrule(lr){1-3} \cmidrule(lr){4-6} $h$       & $d$                        & $\epsilon$ & Falsified (Unverifiable) & Falsified (Clean) & Running Time (s)   \\
          \midrule 6                                      & 2                          & 0.2        & 100\%                    & 60\%              & $0.0448\,(0.0063)$ \\
          6                                               & 2                          & 0.5        & 100\%                    & 60\%              & $0.0546\,(0.0242)$ \\
          6                                               & 3                          & 0.2        & 100\%                    & 20\%              & $5.8013\,(0.5377)$ \\
          6                                               & 3                          & 0.5        & 100\%                    & 40\%              & $5.4401\,(0.2486)$ \\
          10                                              & 2                          & 0.2        & 100\%                    & 30\%              & $0.0331\,(0.0063)$ \\
          10                                              & 2                          & 0.5        & 100\%                    & 50\%              & $0.0335\,(0.0102)$ \\
          10                                              & 3                          & 0.2        & 100\%                    & 30\%              & $4.6127\,(0.1721)$ \\
          10                                              & 3                          & 0.5        & 100\%                    & 50\%              & $4.4646\,(0.2993)$ \\
          \bottomrule
        \end{tabular}
      \end{sc}
    \end{small}
  \end{table}

  \paragraph{Verification}
  Verification is built upon \emph{HomotopyContinuation.jl}, which computes the
  exact robust radius for each test point.
  Given an input instance $(\bxi, \epsilon )$, verification is determined by
  whether the model admits an adversarial example within an $\ell_{2}$-ball of
  radius $\epsilon$ around $\bxi$.
  If no such adversarial example exists, the instance is classified as \emph{verified}.
  Otherwise, the instance is \emph{falsified}, indicating the presence of a valid
  counterexample within the allowed perturbation region.

  For instances designated as \emph{unverifiable} by construction, a sound verification
  method must always report falsification, since counterexamples are guaranteed
  to exist.
  Any verification claim on an unverifiable instance therefore constitutes a false
  positive and signals a soundness error in the verifier.
  In contrast, for clean instances, both verification and falsification are
  acceptable outcomes, depending on the true robustness of the model at the given
  input.
  This distinction allows us to separately evaluate verifier soundness on
  unverifiable instances and verifier precision on clean instances.
  \Cref{tab:verification-results} summarizes the verification results across all
  8 benchmark configurations.
  The verifier achieves perfect soundness on all unverifiable instances,
  correctly identifying all planted counterexamples.
  On clean instances, the verifier successfully falsifies a significant fraction,
  demonstrating its ability to detect non-robustness even without planted counterexamples.
  The average verification time per instance is also reported, showing efficient
  performance across configurations.

\section{ED Degree of the Decision Boundary of Polynomial Networks}  

  \subsection{Proof of \cref{prop:generic-ed-degree}}
  \begingroup
  \renewcommand{\thetheorem}{\ref{prop:generic-ed-degree}}
  \renewcommand{\theHtheorem}{restated.\ref{prop:generic-ed-degree}}
  %
  \propgenericeddegree* \endgroup

  \begin{proof}
    The decision boundary is defined by the polynomial $B_{\bt}(\bx) = f_{\bt,c}(
    \bx ) - f_{\bt,c'}(\bx)$. Since the network has a single hidden layer with
    width $h$ and a degree-$d$ polynomial activation $\sigma(z) = z^{d}$, the function
    $B_{\bt}(\bx)$ takes the form
    \[
      B_{\bt}(\bx) = \sum_{j=1}^{h}\alpha_{j}(\mathbf{w}_{j}^{T}\bx + b_{j})^{d}+
      \beta,
    \]
    where $\mathbf{w}_{j}\in \mathbb{R}^{n}$ are the rows of the first layer weight
    matrix $W^{(1)}\in \mathbb{R}^{h \times n}$, and $b_{j}\in \mathbb{R}$ are
    the entries of the first layer bias vector
    $\boldsymbol{b}^{(1)}\in \mathbb{R}^{h}$. The coefficients are given by $\alpha
    _{j}= W^{(2)}_{j,c}- W^{(2)}_{j,c'}$ and $\beta = b^{(2)}_{c}- b^{(2)}_{c'}$,
    where $W^{(2)}\in \mathbb{R}^{h \times k}$ is the output layer weight matrix
    and $\boldsymbol{b}^{(2)}\in \mathbb{R}^{k}$ is the output bias vector. We analyze
    the ED degree of $\mathcal{V}_{c,c'}^{\bt}$ in two cases based on the relationship
    between $n$ and $h$.

    First, assume $n \leq h$. Since the parameters $\bt$ are generic,
    $B_{\bt}(\bx)$ is a weighted sum of $n$ linearly independent powers, which identifies
    $\mathcal{V}_{c,c'}^{\bt}$ as a generic scaled Fermat hypersurface. Following \citet[][Section 3]{lee_2017}, the Euclidean distance degree of a generic scaled
    Fermat hypersurface is equal to that of a generic hypersurface in $\mathbb{C}
    ^{n}$. 
    \citet[][Proposition 2.6]{draisma_2016} show the ED degree for a generic hypersurface of degree $d$ in $n$ variables is 
    \[
      \text{ED degree}(\mathcal{V}_{c,c'}^{\bt}) = d\sum_{i = 0}^{n-1}(d-1)^{i} , 
    \]
    as desired.

    Now, suppose $n > h$. Let $W = \operatorname{span}\{\mathbf{w}_{1}, \dots , \mathbf{w}
    _{h}\} \subset \mathbb{R}^{n}$ be the subspace spanned by the weight vectors.
    Since the parameters $\bt$ are generic and $h < n$, the vectors $\mathbf{w}_{1}
    , \dots, \mathbf{w}_{h}$ are linearly independent, so $\dim(W) = h$. We perform
    an orthogonal change of coordinates $\bx = Q\boldsymbol{y}$, where
    $Q \in \mathbb{R}^{n \times n}$ is an orthogonal matrix constructed such
    that its first $h$ columns form a basis for $W$.

    In this new coordinate system, for every $j \in \{1, \dots, h\}$, the vector
    $\mathbf{w}_{j}$ lies entirely within the span of the first $h$ basis
    vectors. Consequently, the inner product $\mathbf{w}_{j}^{T}\bx$ depends
    exclusively on the first $h$ coordinates $y_{1}, \dots, y_{h}$. The polynomial
    $B_{\bt}(\bx)$ can thus be written as
    \[
      B_{\bt}(Q\boldsymbol{y}) = \tilde{B}_{\bt}(y_{1}, \dots, y_{h}),
    \]
    where $\tilde{B}_{\bt}$ is a generic scaled Fermat polynomial of degree $d$ in
    $h$ variables. This implies that the variety $\mathcal{V}_{c,c'}^{\bt}$
    decomposes~as:
    \[
      \mathcal{V}_{c,c'}^{\bt}\cong \mathcal{Z}\times \mathbb{C}^{n-h},
    \]where
    $\mathcal{Z}= \{\boldsymbol{z}\in \mathbb{C}^{h}\mid \tilde{B}_{\bt}(\boldsymbol
    {z}) = 0\}$
    is a generic hypersurface in $\mathbb{C}^{h}$.

    The squared Euclidean distance from a generic data point $\boldsymbol{u}$, transformed
    to $\tilde{\boldsymbol{u}}= Q^{T}\boldsymbol{u}$, splits into two
    independent terms:
    \[
      d_{\bu}(\bx) = d_{\tilde{\boldsymbol{u}}}(\boldsymbol{y}) = \sum_{i=1}^{h}(
      y_{i}- \tilde{u}_{i})^{2}+ \sum_{i=h+1}^{n}(y_{i}- \tilde{u}_{i})^{2}.
    \]

    To find the ED degree, we determine the number of critical points of this distance
    function on the variety. The second term depends only on
    $y_{h+1}, \dots, y_{n}$, which are unconstrained by $\tilde{B}_{\bt}$. This
    term has exactly one critical point (the global minimum) at
    $y_{i}= \tilde{u}_{i}$ for all $i > h$. 
    The first term seeks critical points
    of the distance to the hypersurface $\mathcal{Z}$ in $\mathbb{C}^{h}$. 
    Since $\mathcal{Z}$ is defined by a generic polynomial of degree $d$ in $h$
    variables, 
    by \citet[][Proposition 2.6]{draisma_2016}, 
    the number of such
    points is $d\sum_{i=0}^{h-1}(d-1)^{i}$. 
    The total ED degree of $\mathcal{V}_{c,c'}
    ^{\bt}$ is the product of these counts, and the conclusion~follows.
  \end{proof}

  \subsection{Proof of \cref{prop:bottleneck}}
  \begingroup
  \renewcommand{\thetheorem}{\ref{prop:bottleneck}}
  \renewcommand{\theHtheorem}{restated.\ref{prop:bottleneck}}
  %
  \propbottleneck* \endgroup

  \begin{proof}
    The first hidden layer $h_{1}\ge n$ produces a generic multivariate polynomial
    $P(\bx)$. The subsequent $s-1$ hidden layers of width 1 compute a univariate
    polynomial $G(t)$ of degree $d^{s-1}$, evaluated at $t=P(\bx)$. The final
    output layer is linear, so the decision boundary $B_{\bt}(\bx) = f_{c}(\bx) -
    f_{c'}(\bx) = 0$ is given by:
    \[
      (\alpha_{c}- \alpha_{c'}) G(P(\bx)) + (\beta_{c}- \beta_{c'}) = 0 .
    \]Let $\Delta \alpha = \alpha_{c}- \alpha_{c'}$ and $\Delta \beta = \beta_{c}
    - \beta_{c'}$. For generic weights, $\Delta \alpha \neq 0$. We must solve:
    \[
      G(P(\bx)) = -\frac{\Delta \beta}{\Delta \alpha}.
    \]The univariate equation $G(t) = -\Delta \beta / \Delta \alpha$ has
    $M = d^{s-1}$ distinct roots $\{\lambda_{1}, \dots, \lambda_{M}\}$. This implies
    that the decision boundary decomposes into $M$ disjoint hypersurfaces defined
    by $P(\bx) = \lambda_{j}$. Since these are parallel copies of the generic
    base surface $P(\bx) = 0$, the total ED degree is the sum of their individual
    ED degrees. The ED degree of a generic degree-$d$ hypersurface in $n$
    variables is $d \sum_{i=0}^{n-1}(d-1)^{i}$. Multiplying this by $M$ yields:
    \[
      \text{EDdegree}(\mathcal{V}_{c,c'}^{\bt}) = d^{s-1}\cdot \left[ d \sum_{i=0}
      ^{n-1}(d-1)^{i}\right] = d^{s}\sum_{i=0}^{n-1}(d-1)^{i}.
    \]
  \end{proof}

  \subsection{Proof of \cref{prop:generic-ed-degree-deep}}
  \begingroup
  \renewcommand{\thetheorem}{\ref{prop:generic-ed-degree-deep}}
  \renewcommand{\theHtheorem}{restated.\ref{prop:generic-ed-degree-deep}}
  %
  \propgenericeddegreedeep* \endgroup

  \begin{proof}
    The Euclidean Distance degree is a lower semi-continuous function on the irreducible
    parameter space, so the generic value is the maximum attainable value. 
    As pointed out by \citet{draisma_2016},
    the ED degree of a degree-$D$ hypersurface is bounded by
    $D \sum_{i=0}^{n-1}(D-1)^{i}$, with equality holding if the hypersurface is smooth
    and intersects the isotropic quadric at infinity transversally.

    To prove the proposition, it suffices to construct a single parameter
    configuration $\bt^{*}$ such that the decision boundary defines a
    hypersurface achieving this bound. Since $h_{i}\ge n$ for all $i\in [s]$, we
    can enforce $n$ disjoint computational paths by setting the weights of each
    hidden layer to be a block matrix with the $n \times n$ identity in the top-left
    corner and zeros elsewhere:
    \[
      W^{(i)}=
      \begin{bmatrix}
        I_{n}      & \mathbf{0} \\
        \mathbf{0} & \mathbf{0}
      \end{bmatrix}.
    \]
    Setting all hidden biases to zero, this configuration maps the input $\bx$
    to the vector $(x_{1}^{D}, \dots, x_{n}^{D})$ at the final hidden layer. For
    the output layer weights $W^{(s+1)}\in \mathbb{C}^{k \times h_s}$, let
    $\mathbf{w}_{c}, \mathbf{w}_{c'}$ denote the two rows corresponding to $c$
    and $c'$. We choose weights such that their difference $\mathbf{w}_{\text{diff}}
    = \mathbf{w}_{c}- \mathbf{w}_{c'}$ has generic entries $\lambda_{1}, \dots, \lambda
    _{n}$ in the first $n$ positions and zero otherwise. The decision boundary
    becomes the scaled Fermat polynomial:
    \[
      f_{1}(\bx) - f_{2}(\bx) = \sum_{j=1}^{n}\lambda_{j}x_{j}^{D}- C = 0.
    \]For generic $\lambda_{j}$, this hypersurface is smooth and transversal to the
    isotropic quadric $\sum x_{j}^{2}= 0$ at infinity \citep{lee_2017}. Thus, the
    witness achieves the maximal ED degree. Since the network can realize a
    maximal instance, the generic network achieves the maximal count.
  \end{proof}

\section{Expected Real ED Degree of the Decision Boundary of Polynomial Networks} 

 \subsection{Proof of \cref{thm:ed-kac-rice}}
  \begingroup
  \renewcommand{\thetheorem}{\ref{thm:ed-kac-rice}}
  \renewcommand{\theHtheorem}{restated.\ref{thm:ed-kac-rice}}
  %
  \thmedkacrice* \endgroup
  \begin{proof}
    \citet{kac_1943} and \citet{rice_1944} developed a formula to compute the expected
    number of zeros of a random field using the joint density of the field and its
    derivatives.
    A more contemporary treatment can be found in the work of \citet{berzin_2022}.

    Kac--Rice formula states that the expected number of zeros of a random field
    $F:\mathbb{R}^{m}\to\mathbb{R}^{m}$ in a Borel set
    $U\subseteq\mathbb{R}^{m}$ is given by
    \[
      \mathbb{E}\Big[\#\{\boldsymbol{z}\in U:\ F(\boldsymbol{z})=0\}\Big] = \int_{U}
      \mathbb{E}\left[|\det J_{F}(\boldsymbol{z})| \ \middle|\ F(\boldsymbol{z})=
      0 \right] p_{F(\boldsymbol{z})}(0)\, d\boldsymbol{z},
    \]
    where $J_{F}(\boldsymbol{z})$ is the Jacobian of $F$ at $\boldsymbol{z}$, and
    $p_{F(\boldsymbol{z})}(0)$ is the density of $F(\boldsymbol{z})$ evaluated at
    zero.
    Apply this formula to the random field $F(\bx,\lambda)$ defined in the
    proposition, we have
    \begin{align}
      \mathbb{E}\Big[\#\{(\bx,\lambda) \in U:\ F(\bx,\lambda)=0\}\Big] & = \int_{U}\mathbb{E}\left[|\det J_{F}(\bx,\lambda)| \middle|\ F(\bx,\lambda)=0 \right] p_{F(\bx,\lambda)}(0)\, d\bx\, d\lambda.
    \end{align}

    To find the $\det J_{F}(\bx,\lambda)$, we compute the Jacobian matrix
    \[
      J_{F}(\bx,\lambda) =
      \begin{pmatrix}
        \nabla f(\bx)^{\top}            & 0             \\
        I_{n}+ \lambda \nabla^{2}f(\bx) & \nabla f(\bx)
      \end{pmatrix},
    \]
    where $\nabla^{2}f(\bx)$ is the Hessian matrix of $f$ at $\bx$.
    By the column swap rule of determinants, we have
    \[
      \det J_{F}(\bx,\lambda) = (-1)^{n}\det
      \begin{pmatrix}
        0             & \nabla f(\bx)^{\top}            \\
        \nabla f(\bx) & I_{n}+ \lambda \nabla^{2}f(\bx)
      \end{pmatrix}.
    \]
    Define the Schur complement of the block $I_{n}+ \lambda \nabla^{2}f(\bx )$ as
    \[
      S(\bx) := -\nabla f(\bx)^{\top}(I_{n}+ \lambda \nabla^{2}f(\bx))^{-1}\nabla
      f(\bx).
    \]
    According to Schur's formula, we have
    \[
      \det
      \begin{pmatrix}
        0             & \nabla f(\bx)^{\top}            \\
        \nabla f(\bx) & I_{n}+ \lambda \nabla^{2}f(\bx)
      \end{pmatrix}
      = \det(I_{n}+ \lambda \nabla^{2}f(\bx)) \cdot \det(S(\bx)) = \det(I_{n}+ \lambda
      \nabla^{2}f(\bx)) \cdot S(\bx).
    \]
    The last equality holds since $S$ is a scalar ($1\times 1$ matrix).
    Therefore,
    \begin{align*}
      |\det J_{F}(\bx,\lambda)| & = \left| \det \begin{pmatrix}0&\nabla f(\bx)^{\top}\\ \nabla f(\bx)&I_{n}+ \lambda \nabla^{2}f(\bx)\end{pmatrix} \right|             \\
                                & = \left| \det(I_{n}+ \lambda \nabla^{2}f(\bx)) \cdot S(\bx) \right|                                                                  \\
                                & = \left| \det(I_{n}+ \lambda \nabla^{2}f(\bx)) \cdot \nabla f(\bx)^{\top}(I_{n}+ \lambda \nabla^{2}f(\bx))^{-1}\nabla f(\bx) \right| \\
                                & = \left| \nabla f(\bx)^{\top}\operatorname{adj}\!\left(I_{n}+ \lambda \nabla^{2}f(\bx)\right) \nabla f(\bx) \right| . 
    \end{align*}

    On the other hand, since $F(\bx,\lambda) = 0$ implies $f(\bx) = 0$ and
    $\nabla f(\bx) = (\bxi - \bx)/\lambda$, by changing of variable, we have
    \[
      p_{F(\bx,\lambda)}(0) = |\lambda|^{-n}\, p_{\left(f(\bx),
      \nabla f(\bx)\right)}(0, \frac{\bxi - \bx}{\lambda}).
    \]

    Combining the above results, we can obtain the desired formula for the
    expected number of real ED critical points.
  \end{proof}

  \subsection{Proof of \cref{thm:ed-kac-rice-high}}
  \begingroup
  \renewcommand{\thetheorem}{\ref{thm:ed-kac-rice-high}}
  \renewcommand{\theHtheorem}{restated.\ref{thm:ed-kac-rice-high}}
  %
  \thmedkacricehigh* \endgroup

  \begin{proof}
    Handling the system $F(\bx,\lambda) = 0$ directly is difficult due to the complicated joint density of $(f(\bx), \nabla f(\bx))$.
    However, we observe that we can normalize the field $f$ by its standard deviation $\sqrt{K(\bx,\bx)}$ without changing the zero set and the KKT system, which simplifies the joint density significantly.
    Let
    \[
        q(\bx):=1+\frac{\|\bx\|^2}{n},
        \qquad
        r(\bx,\bx'):=1+\frac{\bx^\top \bx'}{n}.
    \]
    By \cref{prop:nngp-shallow-poly}, we have $K(\bx,\bx')=\sum_{s=0}^{\lfloor d/2\rfloor} c(s,d)\, q(\bx)^s q(\bx')^s r(\bx,\bx')^{d-2s}$, and $K(\bx,\bx) = (2d-1)!!\, q(\bx)^d$.
    We can then define the normalized field
    \[
        g(\bx):=\frac{f(\bx)}{\sqrt{K(\bx,\bx)}} =\frac{f(\bx)}{\sqrt{(2d-1)!!}\, q(\bx)^{d/2}}.
    \]
    Since $K(\bx,\bx)>0$ for all $\bx \in \R^{n}$, the zero sets of $f$ and $g$ coincide.  
    Moreover, on the zero set $\{\bx \in \R^{n}:f(\bx)=0\}=\{\bx \in \R^{n}:g(\bx)=0\}$ we have
    \[
        \nabla g(\bx) = \frac{1}{\sqrt{K(\bx,\bx)}}\,\nabla f(\bx),
    \]
    because, by chain rule, we have
    \[
        \nabla \left(\frac{f(\bx)}{\sqrt{K(\bx,\bx)}}\right) = \frac{\nabla f(\bx)}{\sqrt{K(\bx,\bx)}}+f(\bx) \nabla\left(\frac{1}{\sqrt{K(\bx,\bx)}}\right),
    \]
    and the second term vanishes when $f(\bx)=0$.
    Hence, the KKT systems for $f$ and $g$ are in bijection:
    \[
        f(\bx)=0,\quad \bx-\bxi+\lambda_f \nabla f(\bx)=0
    \]
    if and only if
    \[
        g(\bx)=0,\quad \bx-\bxi+\lambda_g \nabla g(\bx)=0
    \]
    with
    $\lambda_g=\lambda_f \sqrt{K(\bx,\bx)}$.
    Therefore, the number of real critical pairs $(\bx,\lambda)$ is the same for $f$ and for $g$.
    So it suffices to work with $g$.
    We denote the normalized KKT system as
    \[ 
      G(\bx,\lambda) :=
      \begin{bmatrix}
        g(\bx)        \\
        \bx-\bxi+\lambda \nabla g(\bx)
      \end{bmatrix}.
    \]
    According to the above argument, the number of real ED critical points of $F$ is the same as the number of real zeros of $G$, denoted by $\mathcal{N}_d$.
    Then, applying \cref{thm:ed-kac-rice} to the system $G(\bx,\lambda)=0$, we have
    \begin{equation}
      \mathbb{E}[\mathcal{N}_d]
      = \int_{\R^{n} \times (\mathbb{R}\setminus\{0\})}|\lambda|^{-n}\, p_{(g(\bx), \nabla g(\mathbf{x}))}\big ( 0, \tfrac
      {\bxi - \bx}{\lambda}\big) \mathcal{I}(\bx, \lambda) \, \mathrm{d}(\bx,\lambda),
    \end{equation}
    where $p_{(g(\bx), \nabla g(\bx))}$ is the joint density of $(g(\bx) ,\nabla g(\bx))$,
    \begin{equation*}
      \mathcal{I}(\bx, \lambda) = \mathbb{E}\Big[\big|{\nabla g(\bx)}^{\top}\operatorname{adj}
      \!\left(I_{n}+\lambda \nabla^{2}g(\bx)\right){\nabla g(\bx)}\big| \Big|\ g
      (\bx)=0,\ \nabla g(\bx)= \frac{\bxi - \bx}{\lambda}\Big].
    \end{equation*}
    To compute the above integral, we need to find the joint distribution of $(g(\bx), \nabla g(\bx))$ and the conditional distribution of $\nabla^{2}g(\bx)$ given $g(\bx)=0$ and $\nabla g(\bx)=\frac{\bxi - \bx}{\lambda}$.
    We show some useful identities for the kernel of $g$ that will be used in the subsequent calculations.
    We define the normalized $r(\bx, \bx')$ as
    \[
        \rho(\bx,\bx'):=\frac{r(\bx,\bx')}{\sqrt{q(\bx)q(\bx')}}
        =\frac{1+\frac{\bx^\top \bx'}{n}}{\sqrt{\left(1+\frac{\|\bx\|^2}{n}\right)\left(1+\frac{\|\bx'\|^2}{n}\right)}}.
    \]
    We also define
    \[
        \Phi(t):=\frac{1}{(2d-1)!!}\sum_{s=0}^{\lfloor d/2\rfloor} c(s,d)\, t^{\,d-2s}.
    \]
    Then the kernel of $g$ can be expressed as
    \begin{equation}
        \widetilde K(\bx,\bx'):=\mathbb{E}[g(\bx)g(\bx')] =\frac{K(\bx,\bx')}{\sqrt{K(\bx,\bx)K(\bx',\bx')}} =\Phi(\rho(\bx,\bx')).
    \end{equation}
    Since the coefficients are nonnegative and $\sum_{s=0}^{\lfloor d/2\rfloor} c(s,d)=(2d-1)!!$, by \cref{rmk:identities}, we have $\Phi(1)=1$.
    Next, the first derivative of $\Phi$ at $1$ is given by
    \[
        \Phi'(1)
        =
        \frac{1}{(2d-1)!!}
        \sum_{s=0}^{\lfloor d/2\rfloor} c(s,d)(d-2s).
    \]
    Using again \cref{rmk:identities},
    \[
        \sum_{s=0}^{\lfloor d/2\rfloor} c(s,d)=(2d-1)!!,
        \qquad
        \sum_{s=0}^{\lfloor d/2\rfloor} s\,c(s,d)=\frac{d(d-1)(2d-3)!!}{2},
    \]
    hence
    \[
        \sum_{s=0}^{\lfloor d/2\rfloor} c(s,d)(d-2s)
        =
        d(2d-1)!!-d(d-1)(2d-3)!!
        =
        d^2(2d-3)!!.
    \]
    Since $(2d-1)!!=(2d-1)(2d-3)!!$, it follows that
    \[
        \alpha_d:=\Phi'(1)=\frac{d^2}{2d-1} \leq d.
    \]
    Since \(c(s,d)\ge 0\) and \(\sum_s c(s,d)=(2d-1)!!\), the polynomial
    \[
        \Phi(t)=\frac{1}{(2d-1)!!}\sum_s c(s,d)t^{d-2s}
    \]
    is a convex combination of monomials \(t^k\) with \(0\le k\le d\). 
    Hence, for every \(r\ge 1\), the \(r\)-th derivative of \(\Phi_d\) at \(1\) is upper bounded by the \(r\)-th derivative of \(t^d\) at \(1\), which is \(d(d-1)\cdots (d-r+1)=(d)_r\).
    In particular,
    \begin{equation} \label{eq:Phi-derivative}
        \Phi^{(r)}(1) = \frac{1}{(2d-1)!!}\sum_s c(s,d)(d-2s)_r \le (d)_r \le d^r.
    \end{equation}
    Besides the above identities, we also have the following identities for $\rho$ on the diagonal $\{\bx'=\bx\}$:
    \begin{equation}
        \rho(\bx,\bx)=1,
        \qquad
        \nabla_{\bx} \rho(\bx,\bx')\big|_{\bx'=\bx}=\mathbf{0},
        \qquad
        \nabla_{\bx'} \rho(\bx,\bx')\big|_{\bx'=\bx}=\mathbf{0}.
    \end{equation}
    It is clear to see that $\rho(\bx,\bx)=1$ for every $\bx\in \R^{n}$.
    To compute the first derivatives, we have for every $i\in[n]$,
    \[
        \partial_{x_i}\rho(\bx,\bx') = \frac{x'_i}{n\sqrt{q(\bx)q(\bx')}} - \frac{r(\bx,\bx')x_i}{n\, q(\bx)^{\frac{3}{2}}q(\bx')^{\frac{1}{2}}},
    \]
    and at $\bx'=\bx$ this becomes
    \[
        \partial_{x_i}\rho(\bx,\bx')\big|_{\bx'=\bx} = \frac{x_i}{n q(\bx)}-\frac{q(\bx)x_i}{n q(\bx)^2}=0.
    \]
    We also have the following second derivative identity for every $i,k\in[n]$ by diffentiating once more:
    \begin{equation*}
        \partial_{x_i}\partial_{x'_k}\rho(\bx,\bx')\big|_{\bx'=\bx} = \frac{\delta_{ik}}{nq(\bx)}-\frac{x_i x_k}{n^2 q(\bx)^2}.
    \end{equation*}
    So we can denote the second-derivative matrix by
    \begin{equation}
        B(\bx):=\nabla_{\bx} \nabla_{\bx'}^\top \rho(\bx,\bx')\big|_{\bx'=\bx}=\frac{1}{nq(\bx)}I_n-\frac{1}{n^2 q(\bx)^2} \bx \bx^{\top},
    \end{equation}
    According to the Sherman-Morrison formula, the inverse of $B(\bx)$ is given by
    \begin{equation} \label{eq:B-inverse}
        B(\bx)^{-1} = n q(\bx) \left(I_n+\frac{1}{n}\bx\bx^\top\right).
    \end{equation}
    We also have the following third derivative identity for every $i,j,k\in[n]$ by diffentiating once more. 
    We denote this third-derivative tensor by
    \begin{equation} \label{eq:A-ijk}
        A_{ij,k}(\bx):= \partial_{x_i}\partial_{x_j}\partial_{x'_k}\rho(\bx,\bx')\big|_{\bx'=\bx} = -\frac{\delta_{ik}x_j+\delta_{jk}x_i}{n^2 q(\bx)^2}+\frac{2x_i x_j x_k}{n^3 q(\bx)^3}.
    \end{equation}
    Finally, a direct differentiation of $\rho(\bx,\bx')=r(\bx,\bx')(q(\bx)q(\bx'))^{-1/2}$ shows that for every multi-index $\beta$ with $|\beta|\le 4$,
    \[
        \bigl|\partial^\beta \rho(\bx,\bx')\big|_{\bx'=\bx}\bigr|\le C_{n,\beta}\, q(\bx)^{-|\beta|/2}.
    \]
    In particular, for every $i,j,k\in[n]$, we have the following bounds for the second and third derivatives of $\rho$ on the diagonal, which will be used in the subsequent calculations:
    \begin{equation} \label{eq:diagonal-derivative}
      |B_{ik}(\bx)|\lesssim_{n} q(\bx)^{-1},
      \qquad
      |A_{ij,k}(\bx)|\lesssim_{n} q(\bx)^{-3/2},
      \qquad
      \bigl|\partial^4 \rho(\bx,\bx')\big|_{\bx'=\bx}\bigr|\lesssim_{n} q(\bx)^{-2}.
    \end{equation}
    Since $\widetilde K(\bx,\bx')=\Phi(\rho(\bx,\bx'))$ and the first derivatives of $\rho$ vanish on the diagonal,
    \[
        \operatorname{Cov}(g(\bx),\nabla g(\bx))
        =
        \nabla_{\bx'} \widetilde K(\bx,\bx')\big|_{\bx'=\bx}
        =
        \Phi'(1)\nabla_{\bx'} \rho(\bx,\bx')\big|_{\bx'=\bx}
        =0.
    \]
    Thus $g(\bx)$ and $\nabla g(\bx)$ are independent.
    Also,
    \[
        \Sigma(\bx):=\operatorname{Cov}(\nabla g(\bx))
        =
        \nabla_{\bx} \nabla_{\bx'}^\top \widetilde K(\bx,\bx')\big|_{\bx'=\bx}
        =
        \Phi'(1) B(\bx)
        =
        \alpha_d\left(\frac{1}{nq(\bx)}I_n-\frac{1}{n^2 q(\bx)^2}\bx\bx^\top\right).
    \]
    The eigenvalues of $\Sigma(\bx)$ are $\dfrac{\alpha_d}{nq(\bx)}$ with multiplicity of $n-1$, and $\dfrac{\alpha_d}{nq(\bx)^2}$ in the radial direction.
    Therefore, the determinant and the inverse of $\Sigma(\bx)$ are given by
    \begin{align}
      \det \Sigma(\bx) &=\dfrac{\alpha_d^n}{n^n q(\bx)^{n+1}}, \label{eq:detSigma} \\  
      \Sigma(\bx)^{-1} &=\frac{nq(\bx)}{\alpha_d}\left(I_n+\frac{1}{n}\bx\bx^\top\right), \label{eq:Sigma-inverse}
    \end{align}
    where the inverse can be found based on \cref{eq:B-inverse}.

    Let $\bu:=\bxi-\bx$. 
    Then
    \[
        \bu^\top \Sigma(\bx)^{-1}\bu
        =
        \frac{nq(\bx)}{\alpha_d}
        \left(
        \|\bu\|^2+\frac{(\bx^\top \bu)^2}{n}
        \right).
    \]
    Define
    \[
        D_{\bx}:=\|\bu\|^2+\frac{(\bx^\top \bu)^2}{n},
        \qquad
        \beta_d(\bx):=\frac{1}{2} \bu^\top \Sigma(\bx)^{-1} \bu
        =\frac{nq(\bx)}{2\alpha_d} D_{\bx}.
    \]
    Since $g(\bx)$ is standard normal and independent of $\nabla g(\bx)$, according to \cref{eq:detSigma}, we have 
    \begin{align}
        p_{(g(\bx),\nabla g(\bx))} \left(0,\frac{\bu}{\lambda}\right)
        & =
        (2\pi)^{-(n+1)/2}
        (\det \Sigma(\bx))^{-1/2}
        \exp\!\left(-\frac{\beta_d(\bx)}{\lambda^2}\right) \notag \\
        &= C_n\, \alpha_d^{-n/2} q(\bx)^{(n+1)/2} \exp\!\left(-\frac{\beta_d(\bx)}{\lambda^2}\right), \label{eq:joint-density}
    \end{align}
    where $C_n$ is a constant depending only on $n$.
    Set
    \[
      H(\bx):=\nabla^2 g(\bx)\in \mathrm{Sym}_n, \qquad
      \Sigma(\bx):=\operatorname{Cov}(\nabla g(\bx)).
    \]
    Since \(g\) is a Gaussian field, the random vector
    \[
      \Big(g(\bx),\,\nabla g(\bx),\,\operatorname{vec}(H(\bx))\Big)
    \]
    is jointly Gaussian. 
    Hence, by the Gaussian regression formula,
    \[
      \mathbb E\!\left[\operatorname{vec}(H(\bx)) \,\middle|\, g(\bx)=0,\ \nabla g(\bx)=\boldsymbol{v}\right]
      =
      \operatorname{Cov}\!\big(\operatorname{vec}(H(\bx)),(g(\bx),\nabla g(\bx))\big)
      \operatorname{Cov}\!\big((g(\bx),\nabla g(\bx))\big)^{-1} \begin{bmatrix}
        0 \\
        \boldsymbol{v}
      \end{bmatrix}.
    \]
    Because \(\operatorname{Cov}(g(\bx),\nabla g(\bx))=0\), this reduces to
    \[
      \mathbb E\!\left[\operatorname{vec}(H(\bx)) \,\middle|\, g(\bx)=0,\ \nabla g(\bx)=\boldsymbol{v}\right]
      =
      \operatorname{Cov}\!\big(\operatorname{vec}(H(\bx)),\nabla g(\bx)\big)\Sigma(\bx)^{-1}\boldsymbol{v}.
    \]
    Equivalently, entrywise,
    \[
      \mathbb E\!\left[\partial_{ij}g(\bx)\,\middle|\, g(\bx)=0, \nabla g(\bx)=\boldsymbol{v}\right]
      =
      \sum_{k=1}^n
      \operatorname{Cov}\!\big(\partial_{ij} g(\bx),\partial_k g(\bx)\big)
      (\Sigma(\bx)^{-1} \boldsymbol{v})_k.
    \]
    Now we have
    \[
        \operatorname{Cov}(\partial_{ij}g(\bx),\partial_k g(\bx)) = \partial_{x_i}\partial_{x_j}\partial_{x'_k}\widetilde K(\bx,\bx')\big|_{\bx'=\bx}.
    \]
    Since \(\widetilde K(\bx,\bx')=\Phi_d(\rho(\bx,\bx'))\), the Fa\`{a} di Bruno's formula \citep{bruno_1855,bruno_1857} gives
    \[
      \partial_{x_i}\partial_{x_j}\partial_{x'_k}\widetilde K(\bx, \bx')
      =
      \Phi_d'''(\rho)\,\partial_{x_i}\rho\,\partial_{x_j}\rho\,\partial_{x'_k}\rho
      +\Phi_d''(\rho)\Big(
      \partial_{x_i x_j}\rho\,\partial_{x'_k}\rho
      +\partial_{x_i x'_k}\rho\,\partial_{x_j}\rho
      +\partial_{x_j x'_k}\rho\,\partial_{x_i}\rho \Big)
      +\Phi_d'(\rho)\,\partial_{x_i x_j x'_k}\rho.
    \]
    On the diagonal \(\bx'=\bx\), we have \(\rho(\bx,\bx)=1\), \(\nabla_{\bx} \rho(\bx,\bx')\big|_{\bx'=\bx}=0\), and \(\nabla_{\bx'} \rho(\bx,\bx')\big|_{\bx'=x}=0\).
    Hence all terms containing first derivatives of \(\rho\) vanish, and therefore
    \[
      \operatorname{Cov}(\partial_{ij}g(\bx),\partial_k g(\bx))
      =
      \partial_{x_i}\partial_{x_j}\partial_{x'_k}\widetilde K(\bx,\bx')\big|_{\bx'=\bx}
      =
      \Phi_d'(1)\,
      \partial_{x_i}\partial_{x_j}\partial_{x'_k}\rho(\bx,\bx')\big|_{\bx'=\bx}
      = \alpha_d\, A_{ij,k}(\bx).
    \]
    Therefore, the $ij$-th entry of the conditional Hessian mean is given by
    \[
      \mathbb{E}[H_{ij}(\bx)\mid g(\bx)=0, \nabla g(\bx)=\boldsymbol{v}]
      =
      \alpha_d \sum_{k=1}^n A_{ij,k}(\bx)\,(\Sigma(\bx)^{-1} \boldsymbol{v})_k.
    \]
    Using the explicit formulas for $A_{ij,k}(\bx)$ (\cref{eq:A-ijk}) and $\Sigma(\bx)^{-1}$ (\cref{eq:Sigma-inverse}), we have
    \begin{align*}
      \sum_{k=1}^n A_{ij,k}(\bx)\,(\Sigma(\bx)^{-1} \boldsymbol{v})_k &=
      \sum_{k=1}^n \left(-\frac{\delta_{ik}x_j+\delta_{jk}x_i}{n^2 q(\bx)^2}+\frac{2x_i x_j x_k}{n^3 q(\bx)^3}\right) \left(\frac{nq(\bx)}{\alpha_d}\left(v_k+\frac{1}{n}(\bx^\top \boldsymbol{v}) x_k\right)\right) \\
      &=
      -\frac{x_j v_i+x_i v_j}{n q(\bx)\,\alpha_d}.
    \end{align*}
    Hence, the conditional Hessian mean can be expressed in matrix form as
    \begin{equation}
      \mathbb{E}[H(\bx)\mid g(\bx)=0, \nabla g(\bx)=\boldsymbol{v}]
      =
      \Big[-\frac{x_j v_i+x_i v_j}{n q(\bx)}\Big]_{i,j=1}^n
      = -\dfrac{\bx \boldsymbol{v}^\top + \boldsymbol{v} \bx^\top}{n q(\bx)}.
    \end{equation}
    In particular, the operator norm of the conditional Hessian mean is upper bounded by
    \begin{equation} \label{eq:conditional-Hessian-mean}
      \left\|
      \mathbb{E}[H(\bx)\mid g(\bx)=0,\nabla g(\bx)=\boldsymbol{v}]
      \right\|_{\mathrm{op}}
      \le
      \frac{2\|\bx\|\|\boldsymbol{v}\|}{n q(\bx)}
      \lesssim_{n} q(\bx)^{-1/2}\|\boldsymbol{v}\|.
    \end{equation}

    We next bound the covariance of $H(\bx)$.
    For every $i,j,k,\ell\in[n]$, we have
    \[
      \operatorname{Cov}(\partial_{ij} g(\bx),\partial_{k\ell} g(\bx))
      =
      \partial_{x_i x_j x'_k x'_\ell}\widetilde K(\bx,\bx')\Big|_{\bx'=\bx}.
    \]
    According to Fa\`{a} di Bruno's formula, the fourth derivative of \(\Phi_d(\rho)\) contains the following types of terms:
    \[
      \partial_{abcd}\Phi_d(\rho)
      =
      \Phi_d^{(4)}(\rho)\,\rho_a \rho_b \rho_c \rho_d
      +\Phi_d^{(3)}(\rho)\sum \rho_{ab}\rho_c \rho_d
      +\Phi_d^{(2)}(\rho)\sum \rho_{ab}\rho_{cd}
      +\Phi_d^{(2)}(\rho)\sum \rho_{abc}\rho_d
      +\Phi_d'(\rho)\,\rho_{abcd},
    \]
    where the sums range over the usual partitions of \(\{a,b,c,d\}\).
    Taking \(a=x_i\), \(b=x_j\), \(c=x'_k\), \(d=x'_\ell\), and evaluating at \(\bx'=\bx\), we use
    \[
      \rho(\bx,\bx)=1,\qquad
      \nabla_{\bx} \rho(\bx,\bx')\big|_{\bx'=\bx}=0,\qquad
      \nabla_{\bx'} \rho(\bx,\bx')\big|_{\bx'=\bx}=0.
    \]
    Hence every term containing a first derivative of \(\rho\) vanishes, and only the pairings of second derivatives and the fourth derivative remain:
    \[
      \partial_{x_i x_j x'_k x'_\ell}\widetilde K(\bx,\bx')\Big|_{\bx'=\bx}
      =
      \Phi_d''(1)\Big(
      \rho_{x_i x_j}\rho_{x'_k x'_\ell}
      +\rho_{x_i x'_k}\rho_{x_j x'_\ell}
      +\rho_{x_i x'_\ell}\rho_{x_j x'_k}
      \Big)\Big|_{\bx'=\bx}
      +\Phi_d'(1)\rho_{x_i x_j x'_k x'_\ell}\Big|_{\bx'=\bx}.
    \]
    According to \cref{eq:diagonal-derivative}, each second derivative of \(\rho\) at the diagonal is \(O(q(\bx)^{-1})\) and each fourth derivative is \(O(q(\bx)^{-2})\).
    It follows that
    \[
      \bigl|
      \operatorname{Cov}(\partial_{ij}g(\bx),\partial_{k\ell}g(\bx))
      \bigr|
      \lesssim_{n}
      \Phi_d''(1)q(\bx)^{-2}+\Phi_d'(1)q(\bx)^{-2}
      \lesssim_{n}
      d^2 q(\bx)^{-2},
    \]
    where the last inequality follows from \cref{eq:Phi-derivative}.

    Now let
    \[
      Y:=\operatorname{vec}(H(\bx)),\qquad
      Z:=\binom{g(\bx)}{\nabla g(\bx)}.
    \]
    Since \((Y,Z)\) is jointly Gaussian, the conditional covariance is
    \[
      \operatorname{Cov}(Y\mid Z)
      =
      \Sigma_{YY}-\Sigma_{YZ}\Sigma_{ZZ}^{-1}\Sigma_{ZY}
      \preceq
      \Sigma_{YY}.
    \]
    Therefore each conditional Hessian entry has variance at most \(C_n d^2 q(\bx)^{-2}\) for some constant \(C_n\) depending only on \(n\). 
    Writing
    \[
      H(\bx)\mid \{g(\bx)=0,\nabla g(\bx)=\boldsymbol{v}\}=M(\bx,\boldsymbol{v})+Z(\bx,\boldsymbol{v}),
    \]
    where \(M(\bx,\boldsymbol{v}):=\mathbb{E}[H(\bx)\mid g(\bx)=0, \nabla g(\bx)=\boldsymbol{v}]\) is the conditional mean and \(Z(\bx,\boldsymbol{v})\) is centered Gaussian, each entry of \(Z(\bx,\boldsymbol{v})\) has conditional \(m\)-th moment bounded by \(C_m(d/q(\bx))^m\) for every integer \(m\ge 1\) where \(C_m\) is a constant depending only on \(m\).
    Since the operator norm is upper bounded by the Frobenius norm,
    \[
      \|Z(\bx,\boldsymbol{v})\|_{\mathrm{op}}\le \|Z(\bx,\boldsymbol{v})\|_F
    \]
    and \(n\) is fixed, this yields
    \[
      \mathbb E\!\left[\|Z(\bx,\boldsymbol{v})\|_{\mathrm{op}}^m \,\middle|\, g(\bx)=0,\nabla g(\bx)=\boldsymbol{v}\right]
      \lesssim_{n,m}
      \left(\frac{d}{q(\bx)}\right)^m.
    \]
    Recall that, from \cref{eq:conditional-Hessian-mean}, we have 
    \[
      \|M(\bx,\boldsymbol{v})\|_{\mathrm{op}}\lesssim_{n} q(\bx)^{-1/2}\|\boldsymbol{v}\|.
    \]
    Combining the above two inequalities and using the triangle inequality, we have 
    \begin{align}
      \mathbb E \left[\|H(\bx)\|_{\mathrm{op}}^m \,\middle|\, g(\bx)=0,\nabla g(\bx)=\boldsymbol{v}\right] &\leq \mathbb E \left[\left(\|M(\bx,\boldsymbol{v})\|_{\mathrm{op}} + \|Z(\bx,\boldsymbol{v})\|_{\mathrm{op}}\right)^m \,\middle|\, g(\bx)=0,\nabla g(\bx)=\boldsymbol{v}\right] \notag \\ 
      &\lesssim_{m} \mathbb E \left[\|M(\bx,\boldsymbol{v})\|_{\mathrm{op}}^m + \|Z(\bx,\boldsymbol{v})\|_{\mathrm{op}}^m \,\middle|\, g(\bx)=0,\nabla g(\bx)=\boldsymbol{v}\right] \notag \\
      &\lesssim_{n,m} q(\bx)^{-m} d^m + q(\bx)^{-m/2}\|\boldsymbol{v}\|^m. \label{eq:Hessian-moment}
    \end{align}

    According to \cref{thm:ed-kac-rice}, for the normalized field \(g\), the integrand \(\mathcal{I}(\bx,\lambda)\) is given by
    \[
      I(\bx,\lambda)
      =
      \mathbb{E}\!\left[
      \left|
      \boldsymbol{v}^\top \operatorname{adj}(I_n+\lambda H(\bx))\, \boldsymbol{v}
      \right|
      \,\middle|\,
      g(\bx)=0,\ \nabla g(\bx)=\boldsymbol{v}
      \right],
      \qquad
      \boldsymbol{v}=\frac{\boldsymbol{u}}{\lambda} = \frac{\bxi-\bx}{\lambda}.
    \]
    Using \cref{lem:adjoint-bound}, we have
    \[
      \|\operatorname{adj}(I_n+\lambda H(\bx))\|_{\mathrm{op}}
      \le
      (1+|\lambda|\|H(\bx)\|_{\mathrm{op}})^{n-1}.
    \]
    We can then obtain the following bound for the integrand:
    \begin{equation}
      I(\bx,\lambda)
      \le
      \|\boldsymbol{v}\|^2
      \mathbb{E}\!\left[
      (1+|\lambda|\|H(\bx)\|_{\mathrm{op}})^{n-1}
      \,\middle|\,
      g(\bx)=0,\ \nabla g(\bx)=\boldsymbol{v}
      \right].
    \end{equation}
    Expanding the binomial and using \cref{eq:Hessian-moment}, we have
    \begin{align}
      I(\bx,\lambda)
      &\le
      \|\boldsymbol{v}\|^2
      \mathbb{E}\!\left[
      (1+|\lambda|\|H(\bx)\|_{\mathrm{op}})^{n-1}
      \,\middle|\,
      g(\bx)=0,\ \nabla g(\bx)=\boldsymbol{v}
      \right] \notag \\
      &= \|\boldsymbol{v}\|^2 \sum_{m=0}^{n-1} \binom{n-1}{m} |\lambda|^m \mathbb E\!\left[\|H(\bx)\|_{\mathrm{op}}^m \,\middle|\, g(\bx)=0,\nabla g(\bx)=\boldsymbol{v}\right] \notag \\
      &\lesssim_{n,m} \|\boldsymbol{v}\|^2 \sum_{m=0}^{n-1} \binom{n-1}{m} |\lambda|^m \left[q(\bx)^{-m} d^m + q(\bx)^{-m/2}\|\boldsymbol{v}\|^m\right] \notag \\
      &\lesssim_{n} \sum_{m=0}^{n-1} \|\boldsymbol{v}\|^2 |\lambda|^m \left[q(\bx)^{-m} d^m + q(\bx)^{-m/2}\|\boldsymbol{v}\|^m\right] \notag \\ 
      &= \sum_{m=0}^{n-1} \left[q(\bx)^{-m} \|\boldsymbol{u}\|^2 |\lambda|^{m-2} d^m + q(\bx)^{-m/2}\|\boldsymbol{u}\|^{m+2} |\lambda|^{-2}\right] .
      \label{eq:integrand-bound}
    \end{align}

    According to \cref{thm:ed-kac-rice}, the expected number of real critical points is given by
    \[
      \mathbb{E}[\mathcal{N}_d]
      =
      \int_{\mathbb{R}^n}\int_{\mathbb{R}\setminus\{0\}}
      |\lambda|^{-n}
      p_{(g(\bx),\nabla g(\bx))}\!\left(0,\frac{\boldsymbol{u}}{\lambda}\right)
      I(\bx,\lambda)\, d\lambda\, d \bx.
    \]
    Combining \cref{eq:joint-density} and \cref{eq:integrand-bound}, we have
    \begin{equation} \label{eq:expected-number-bound}
      \mathbb{E}[\mathcal{N}_d]
      \lesssim_{n}
      \sum_{m=0}^{n-1} (A_{m,d}+B_{m,d}),
    \end{equation}
    where
    \[
      A_{m,d}
      =
      \alpha_d^{-\frac{n}{2}}
      \int_{\mathbb{R}^n}
      q(\bx)^{\frac{n+1-2m}{2}}
      d^m
      \|\bxi-\bx\|^2
      \left(
      \int_{\mathbb{R}\setminus\{0\}}
      |\lambda|^{m-n-2} e^{-\beta_d(\bx)/\lambda^2}\, d\lambda
      \right)d \bx,
    \]
    and
    \[
      B_{m,d}
      =
      \alpha_d^{-\frac{n}{2}}
      \int_{\mathbb{R}^n}
      q(\bx)^{\frac{n+1-m}{2}} \|\bxi-\bx\|^{m+2}
      \left(
      \int_{\mathbb{R}\setminus\{0\}}
      |\lambda|^{-n-2} e^{-\beta_d(\bx)/\lambda^2}\, d\lambda
      \right)d\bx.
    \]
    We now use the elementary identity
    \begin{equation}
      \int_{\mathbb{R}\setminus\{0\}}
      |\lambda|^{-r} e^{-\beta/\lambda^2}\, d\lambda
      =
      \Gamma\!\left(\frac{r-1}{2}\right)\beta^{-(r-1)/2},
      \qquad r>1,\ \beta>0.
    \end{equation}
    To see this, since the integrand is even, we have 
    \[
      \int_{\mathbb{R}\setminus\{0\}} |\lambda|^{-r} e^{-\beta/\lambda^2}\, d\lambda
      = 2\int_0^\infty \lambda^{-r} e^{-\beta/\lambda^2}\, d\lambda.
    \]
    Let \( t = \beta/\lambda^2 \). Then \( \lambda = (\beta/t)^{1/2} \) and
    \[
      d\lambda = -\tfrac{1}{2}\beta^{1/2} t^{-3/2}\, dt.
    \]
    Hence
    \(
      \lambda^{-r} = \beta^{-r/2} t^{r/2},
    \)
    and
    \begin{align*}
      2\int_0^\infty \lambda^{-r} e^{-\beta/\lambda^2}\, d\lambda
      &= 2\int_\infty^0 \beta^{-r/2} t^{r/2} e^{-t}
      \left(-\tfrac{1}{2}\beta^{1/2} t^{-3/2}\, dt\right) \\
      &= \beta^{-(r-1)/2} \int_0^\infty t^{(r-3)/2} e^{-t}\, dt \\
      &= \beta^{-(r-1)/2}\Gamma\!\left(\dfrac{r-1}{2}\right).
    \end{align*}
    This is finite iff \(r>1\), which ensures integrability at infinity.
    Applying this with $r=n-m+2$ and $r=n+2$ gives
    \begin{align*}
      A_{m,d}
      & \lesssim_{n,m}
      \alpha_d^{-\frac{n}{2}}
      d^m
      \int_{\mathbb{R}^n}
      q(\bx)^{\frac{n+1-2m}{2}}
      \|\bxi-\bx\|^2
      \beta_d(\bx)^{-\frac{n-m+1}{2}}\, d \bx, \\
      B_{m,d}
      &\lesssim_{n,m}
      \alpha_d^{-\frac{n}{2}}
      \int_{\mathbb{R}^n}
      q(\bx)^{\frac{n+1-m}{2}}
      \|\bxi-\bx\|^{m+2}
      \beta_d(\bx)^{-\frac{n+1}{2}}\, d \bx.
    \end{align*}
    Recall that
    \[
      \beta_d(\bx)=\frac{nq(\bx)}{2\alpha_d}D_{\bx},
      \qquad
      D_{\bx}=\|\bxi-\bx\|^2+\frac{\left(\bx^\top \left(\bxi-\bx\right)\right)^2}{n},
    \]
    we obtain
    \begin{align}
      A_{m,d}
      &\lesssim_{n,m}
      \alpha_d^{-\frac{m-1}{2}}
      d^{m}
      \int_{\mathbb{R}^n}
      q(\bx)^{-\frac{m}{2}}
      \|\bxi-\bx\|^2 D_{\bx}^{-\frac{n-m+1}{2}}
      \, d \bx, \label{eq:Amd-bound}\\ 
      B_{m,d}
      &\lesssim_{n,m}
      \alpha_d^{\frac{1}{2}}
      \int_{\mathbb{R}^n}
      q(\bx)^{-\frac{m}{2}}
      \|\bxi-\bx\|^{m+2} D_{\bx}^{-\frac{n+1}{2}}
      \, d\bx. \label{eq:Bmd-bound}
    \end{align}

    Thus it remains only to prove that the $\bx$-integrals are finite.
    Define
    \[
      J_m(\bxi):=
      \int_{\mathbb{R}^n}
      q(\bx)^{-\frac{m}{2}}
      \|\bxi-\bx\|^2 D_{\bx}^{-\frac{n-m+1}{2}}
      \, d \bx,
      \qquad
      K_m(\bxi):=
      \int_{\mathbb{R}^n}
      q(\bx)^{-\frac{m}{2}}
      \|\bxi-\bx\|^{m+2} D_{\bx}^{-\frac{n+1}{2}}
      \, d\bx.
    \]
    We show that $J_m(\bxi),K_m(\bxi)<\infty$ for every $m=0,\dots,n-1$.
    
    We first show that the integrands are locally integrable near $x=\xi$.
    Since \(q(\bx)=1+\|\bx\|^2/n\) is continuous and strictly positive, for every \(\delta>0\) there exists \(C_{\bxi,\delta,m}>0\) such that
    \[
      q(\bx)^{-m/2}\le C_{\bxi,\delta,m}
      \qquad\text{for all } \bx\in \mathbb{B}(\bxi,\delta).
    \]
    Also, we have
    \[
      D_{\bx}=\|\bxi-\bx\|^2+\frac{(\bx^\top(\bxi-\bx))^2}{n}\ge \|\bxi-\bx\|^2.
    \]
    Therefore, for \(\bx\in \mathbb{B}(\bxi,\delta)\), and $m \in\{0,\dots,n-1\}$, we have
    \[
      q(\bx)^{-\frac{m}{2}} \|\bxi-\bx\|^2 D_{\bx}^{-\frac{n-m+1}{2}}
      \le
      C_{\bxi,\delta,m}\,\|\bxi-\bx\|^{m-n+1},
    \]
    and
    \[
      q(\bx)^{-\frac{m}{2}} \|\bxi-\bx\|^{m+2} D_{\bx}^{-\frac{n+1}{2}}
      \le
      C_{\bxi,\delta,m}\,\|\bxi-\bx\|^{m-n+1}.
    \]
    Hence it suffices to show that \(\|\bxi-\bx\|^{m-n+1}\) is locally integrable near \(\bx=\bxi\).
    After the change of variables \(\boldsymbol{u}=\bxi-\bx\), this reduces to
    \[
      \int_{\|\boldsymbol{u}\|<\delta}\|\boldsymbol{u}\|^{m-n+1}\,d\boldsymbol{u}.
    \]
    Using polar coordinates, we can write $\boldsymbol{u} = r \boldsymbol{\theta}$ where $r=\|\boldsymbol{u}\|$ and $\boldsymbol{\theta}\in \mathbb{S}^{n-1}$.
    The Jacobian of this transformation is \(r^{n-1}\), and the surface area of \(\mathbb{S}^{n-1}\) is \(|\mathbb{S}^{n-1}|\). 
    Hence the above integral can be expressed as
    \[
      \int_{\|\boldsymbol{u}\|<\delta}\|\boldsymbol{u}\|^{m-n+1}\,d\boldsymbol{u} = |\mathbb S^{n-1}|\int_0^\delta r^{(m - n + 1) + (n-1)}\,dr \lesssim_{n,m} \delta^{m+1} <\infty.
    \]

    Now we consider the case when $\|\bx\|$ is large.
    Set \(R_0:=2\|\bxi\|+2\). 
    If \(\|\bx\|\ge R_0\), then
    \[
      \|\bxi-\bx\|\le \|\bx\|+\|\bxi\|\le \frac{3}{2}\|\bx\|,
    \]
    and
    \[
      |\bx^\top(\bxi-\bx)|
      =
      |\bx^\top\bxi-\|\bx\|^2|
      \ge \|\bx\|^2-\|\bx\|\,\|\bxi\|
      \ge \frac{\|\bx\|^2}{2}.
    \]
    Therefore, we have
    \[
      D_{\bx}
      =
      \|\bxi-\bx\|^2+\frac{(\bx^\top(\bxi-\bx))^2}{n}
      \ge \frac{\|\bx\|^4}{4n},
      \qquad
      q(\bx)=1+\frac{\|\bx\|^2}{n}\ge \frac{\|\bx\|^2}{n}.
    \]
    It follows that
    \begin{align}
      q(\bx)^{-\frac{m}{2}} \|\bxi-\bx\|^2 D_{\bx}^{-\frac{n-m+1}{2}}
      & \le C_{n,\bxi,m} ^{(1)}\,\|\bx\|^{-m + 2 - 2(n-m+1)} = C_{n,\bxi,m} ^{(1)}\,\|\bx\|^{-2n+m}, \\ 
      q(\bx)^{-\frac{m}{2}} \|\bxi-\bx\|^{m+2} D_{\bx}^{-\frac{n+1}{2}}
      & \le C_{n,\bxi,m} ^{(2)}\,\|\bx\|^{-m + m + 2 - 2(n+1)} = C_{n,\bxi,m} ^{(2)}\,\|\bx\|^{-2n}.
    \end{align}
    Since \(m\le n-1\), we have \(2n-m \ge n+1 > n\), and \(2n>n\).
    Hence, by polar coordinates,
    \[
      \int_{\{\|\bx\|\ge R_0\}} \|\bx\|^{-2n+m}\,d\bx
      =
      |\mathbb S^{n-1}|\int_{R_0}^\infty r^{-2n+m}r^{n-1}\,dr
      =
      |\mathbb S^{n-1}|\int_{R_0}^\infty r^{-n+m-1}\,dr<\infty,
    \]
    because \(-n+m-1\le -2\). 
    Likewise, we have 
    \[
      \int_{\{\|\bx\|\ge R_0\}} \|\bx\|^{-2n}\,d\bx
      =
      |\mathbb S^{n-1}|\int_{R_0}^\infty r^{-2n}r^{n-1}\,dr
      =
      |\mathbb S^{n-1}|\int_{R_0}^\infty r^{-n-1}\,dr<\infty.
    \]
    Therefore both right-hand sides are integrable on \(\{\|\bx\|\ge R_0\}\).
    We conclude that $J_m(\bxi),K_m(\bxi)<\infty$ for every $m=0,\dots,n-1$.

    Now, we can finally find the desired bound for $\mathbb{E}[\mathcal{N}_d]$.
    Combining \cref{eq:expected-number-bound,eq:Amd-bound,eq:Bmd-bound}, we have
    \begin{align*}
      \mathbb{E}[\mathcal{N}_d]
      &\lesssim_{n}
      \sum_{m=0}^{n-1} (A_{m,d}+B_{m,d}) \\ 
      &\lesssim_{n,m}
      \sum_{m=0}^{n-1}
      \left[
      \alpha_d^{-\frac{m-1}{2}}
      d^{m}
      \int_{\mathbb{R}^n}
      q(\bx)^{-\frac{m}{2}}
      \|\bxi-\bx\|^2 D_{\bx}^{-\frac{n-m+1}{2}}
      \, d \bx \right. \\ 
      &\qquad \qquad \qquad \left. + 
      \alpha_d^{\frac{1}{2}}
      \int_{\mathbb{R}^n}
      q(\bx)^{-\frac{m}{2}}
      \|\bxi-\bx\|^2 D_{\bx}^{-\frac{n-m+1}{2}}
      \, d \bx
      \right] \\ 
      & \lesssim_{n,m, \bxi}
      \sum_{m=0}^{n-1} \alpha_d^{-\frac{m-1}{2}} d^m
      +
      \sum_{m=0}^{n-1} \alpha_d^{\frac{1}{2}} < \sum_{m=0}^{n-1} (\dfrac{d}{2})^{-\frac{m-1}{2}} d^m
      +
      \sum_{m=0}^{n-1} d^{\frac{1}{2}}
      \lesssim_{n} d^{\frac{n}{2}},
    \end{align*}
    where second-to-last inequality uses the fact \(\dfrac{d}{2} <  \alpha_d = \dfrac{d^2}{2d-1} \leq d\).
    The proof is complete.
  \end{proof}

  \subsection{Proof of \cref{prop:ed-kac-rice}}
  \begingroup
  \renewcommand{\thetheorem}{\ref{prop:ed-kac-rice}}
  \renewcommand{\theHtheorem}{restated.\ref{prop:ed-kac-rice}}
  %
  \coredkacrice* \endgroup
  \begin{proof}
    When the input dimension of $f$ is $n=1$, the hypersurface
    $\{x \in \R: f(x) = 0\}$ degenerates to a finite set of points.
    In this the case, the number of real ED critical points coincides with the
    number of real zeros of the random function $f(x) = 0$.
    By \cref{prop:nngp-shallow-poly}, $f$ is a centered Gaussian process with
    kernel $K(x,x')$ defined as
    \[
      K(x,x') := \sum_{s=0}^{\left\lfloor \frac{d}{2} \right\rfloor}c(s,d) \left(
      x^{2}+ 1\right)^{s}\left((x')^{2}+ 1\right)^{s}(xx' + 1)^{d-2s},
    \]
    where $c(s,d) = \dfrac{(d!)^{2}}{2^{2s}(s!)^{2}(d-2s)!}$.
    When $x' = x$, the kernel simplifies to
    \[
      K(x,x) = \operatorname{Var}[f(x)] = (2d - 1)!! (x^{2}+ 1)^{d}.
    \]
    According to Kac--Rice formula, the expected number of real zeros of $f$ in
    an interval $U \subseteq \R$ is given by
    \[
      \mathbb{E}\Big[\#\{x \in U:\ f(x)=0\}\Big] = \int_{U}\mathbb{E}\left[ |f'(x
      )| \ \middle|\ f(x)=0 \right] p_{f(x)}(0)\, dx.
    \]
    We define the Kac--Rice density as
    \[
      \rho(x) := \mathbb{E}\left[|f'(x)| \ \middle|\ f(x)=0 \right] p_{f(x)}(0).
    \]
    Expressing the conditional expectation using the joint density of $(f(x), f'(
    x))$, we have
    \begin{align*}
      \rho(x) & = \int_{\mathbb{R}}|y| \, p_{(f(x), f'(x))}(0, y)\, dy                                   \\
              & = \int_{\mathbb{R}}|y| \, \frac{p_{(f(x), f'(x))}(0, y)}{p_{f(x)}(0)}\, p_{f(x)}(0)\, dy \\
              & = \int_{\mathbb{R}}|y| \, p_{(f(x), f'(x))}(0, y) \, dy.
    \end{align*}
    The joint density $p_{(f(x), f'(x))}(0, y)$ is a bivariate Gaussian density
    with mean zero and covariance
    \[
      K_{F}(x,x) :=
      \begin{bmatrix}
        K(x,x)                             & \partial_{x'}K(x,x') \big|_{x' = x}             \\
        \partial_{x}K(x,x') \big|_{x' = x} & \partial_{x}\partial_{x'}K(x,x') \big|_{x' = x}
      \end{bmatrix}.
    \]
    Then we have
    \begin{align*}
      \rho(x) & = \int_{\mathbb{R}}|y| \, \frac{1}{2\pi\sqrt{\det K_{F}(x,x)}}\exp\left(-\frac{1}{2}\begin{bmatrix}0 \\ y\end{bmatrix}^{\top}K_{F}(x,x)^{-1}\begin{bmatrix}0 \\ y\end{bmatrix}\right) dy \\
              & = \int_{\mathbb{R}}|y| \, \frac{1}{2\pi\sqrt{\det K_{F}(x,x)}}\exp\left(-\frac{y^{2}}{2}\cdot \frac{K(x,x)}{\det K_{F}(x,x)}\right) dy                                                   \\
              & = \frac{1}{2\pi\sqrt{\det K_{F}(x,x)}}\cdot \frac{2 \det K_{F}(x,x)}{K(x,x)}                                                                                                             \\
              & = \frac{\sqrt{\det K_{F}(x,x)}}{\pi K(x,x)}.
    \end{align*}
    By \cref{prop:poly-nngp-kernel}, we have
    \begin{align*}
      \partial_{x}K(x,x') \big|_{x' = x}              & = d(2d - 1)!! \, x (x^{2}+ 1)^{d-1},                                             \\
      \partial_{x}\partial_{x'}K(x,x') \big|_{x' = x} & = 2d^{2}(d-1)(2d - 3)!! \, x^{2}(x^{2}+ 1)^{d-2}+ d^{2}(2d-3)!!(x^{2}+ 1)^{d-1}.
    \end{align*}
    The determinant of $K_{F}(x,x)$ can be computed as
    \begin{align*}
      \det K_{F}(x,x) & = K(x,x) \cdot \partial_{x}\partial_{x'}K(x,x') \big|_{x' = x}- \left(\partial_{x'}K(x,x') \big|_{x' = x}\right)^{2}        \\
                      & = (2d - 1)!! (x^{2}+ 1)^{d}\cdot \left[2d^{2}(d-1)(2d - 3)!! \, x^{2}(x^{2}+ 1)^{d-2}+ d^{2}(2d-3)!!(x^{2}+ 1)^{d-1}\right] \\
                      & \quad - \left[d(2d - 1)!! \, x (x^{2}+ 1)^{d-1}\right]^{2}                                                                  \\
                      & = d^{2}(2d-1){((2d - 3)!!)}^{2}(x^{2}+ 1)^{2d - 2}.
    \end{align*}
    The Kac--Rice density can then be found by
    \[
      \rho(x) = \frac{\sqrt{\det K_{F}(x,x)}}{\pi K(x,x)}= \frac{d\sqrt{(2d-1)}{(2d - 3)!!}(x^{2}+
      1)^{d - 1}}{\pi (2d-1)!!(x^{2}+1)^{d}}= \frac{d}{\sqrt{2d-1}}\cdot \frac{1}{\pi(x^{2}+1)} . 
    \]
    Therefore, the expected number of real zeros of $f$ in an $\R$ is given by
    \[
      \mathbb{E}\Big[\#\{x \in U:\ f(x)=0\}\Big] = \int_{\R}\rho(x) dx = \int_{\R}
      \frac{d}{\sqrt{2d-1}}\cdot \frac{1}{\pi(x^{2}+1)}dx = \frac{d}{\sqrt{2d-1}} . 
    \]

    \cref{fig:kac-rice-real-roots} shows a simulation study that verifies the
    above theoretical result.
    \begin{figure}[ht]
      \vskip 0.2in
      \begin{center}
        \centerline{\includegraphics[width=0.6\columnwidth]{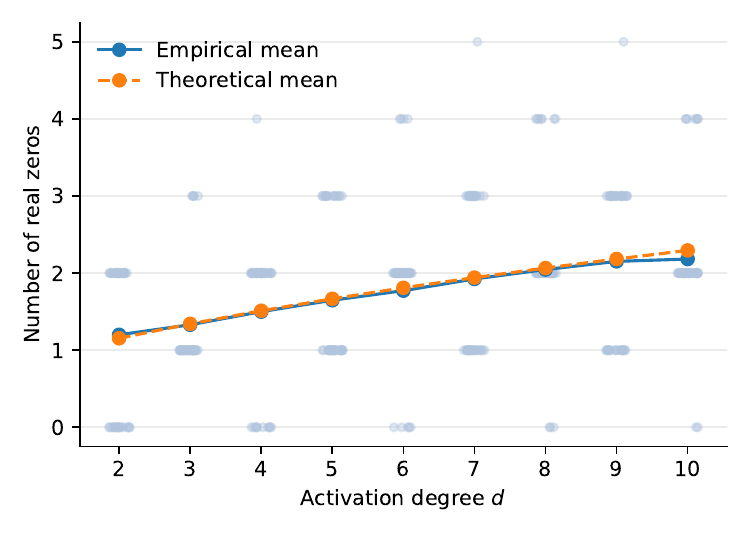}}
        \caption{ Number of real zeros vs.\ activation degree.
        Light points show individual simulations; solid and dashed lines denote
        empirical and theoretical means ($n=1,000$ samples, $m=20,000$ hidden
        units). }
        \label{fig:kac-rice-real-roots}
      \end{center}
    \end{figure}
  \end{proof}

  \section{Auxiliary Lemmas and Proofs for Expected Real ED Degree}  
  \subsection{Remark of \cref{thm:ed-kac-rice}}
  \begin{remark}
    When the random field is a Gaussian process, we provide some simplification
    of the above result.
    Let $f : \mathbb{R}^{n}\to \mathbb{R}$ be a Gaussian process, i.e.,
    $f \sim \mathcal{GP}(0, K),$ with a covariance kernel
    $K_{f}: \mathbb{R}^{n}\times \mathbb{R}^{n}\to \mathbb{R}$ that is $C^{4}$.
    Then the joint density term $p_{\left(f(\bx), \nabla f(\bx)\right)}(0, \frac{\bxi
    - \bx}{\lambda})$ can be found by
    \begin{equation*}
      p_{(f(\bx), \nabla f(\mathbf{x}))}\left(0, \frac{\bxi - \bx}{\lambda}\right
      ) = \left((2\pi)^{n+1}K_{f}(\bx, \bx) \det S(\bx)\right)^{-\frac{1}{2}}\exp
      {\left(-\frac{\left(\bxi - \bx\right)^{\top}S(\bx)^{-1}\left(\bxi - \bx\right)}{2\lambda^{2}} \right)}
      ,
    \end{equation*}
    where $S(\bx) \in \R^{n \times n}$ is defined as
    \begin{equation*}
      S(\bx) := \nabla_{\bx}\nabla_{\bx'}^{\top}K_{f}(\bx,\bx') \big|_{\bx'=\bx}-
      \frac{\nabla_{\bx}K_{f}(\bx,\bx') \big|_{\bx'=\bx}\left(\nabla_{\bx'}K_{f}(\bx,\bx')
      \big|_{\bx'=\bx}\right)^{\top}}{K_{f}(\bx,\bx)}.
    \end{equation*}

    By \cref{lem:augmented-gp}, the joint distribution of $(f(\bx), \nabla f( \bx
    ))$ is Gaussian with mean zero and covariance
    \[
      K_{F}(\bx,\bx) :=
      \begin{bmatrix}
        K_{f}(\bx,\bx)                               & \left(\nabla_{\bx'}K_{f}(\bx,\bx') \big|_{\bx'=\bx}\right)^{\top} \\
        \nabla_{\bx}K_{f}(\bx,\bx') \big|_{\bx'=\bx} & \nabla_{\bx}\nabla_{\bx'}^{\top}K_{f}(\bx,\bx') \big|_{\bx'=\bx}
      \end{bmatrix}.
    \]
    Therefore, the density
    $p_{\left(f(\bx), \nabla f(\bx)\right)}(0, \frac{\bxi - \bx}{\lambda})$ is
    given by
    \begin{align}
      p_{\left(f(\bx), \nabla f(\bx)\right)}(0, \frac{\bxi - \bx}{\lambda}) & = \frac{1}{(2\pi)^{(n+1)/2}\sqrt{\det K_F(\bx,\bx)}}\exp\left(-\frac{1}{2}\begin{bmatrix}0 \\ \frac{\bxi - \bx}{\lambda}\end{bmatrix}^{\top}K_{F}(\bx,\bx)^{-1}\begin{bmatrix}0 \\ \frac{\bxi - \bx}{\lambda}\end{bmatrix}\right) \nonumber \\
                                                                            & = \frac{1}{\sqrt{(2\pi)^{n+1}K_{f}(\bx, \bx) \det S(\bx)}}\exp\left(-\frac{(\bxi - \bx)^{\top}S(\bx)^{-1}(\bxi - \bx)}{2\lambda^{2}}\right),
    \end{align}
    where the last equality follows from the block matrix inversion formula via
    Schur complement.
    To find $\mathcal{I}(\bx,\lambda)$, we need to find the conditional
    distribution of $\nabla^{2}f(\bx)$ given $f(\bx) = 0$ and
    $\nabla f(\bx) = \frac{\bxi - \bx}{\lambda}$.
    Since the kernel $K_{f}$ is $C^{4}$, by a similar argument as in
    \cref{lem:augmented-gp}, the joint distribution of $\left(f(\bx), \nabla f(\bx
    ), \operatorname{vech}{(\nabla^{2} f(\bx))}\right)$ is still Gaussian.
    The covariance matrix can be computed using the derivatives of the kernel
    $K_{f}$, i.e.,
    \begin{align*}
      \Cov \bigl(\partial_{ij}f(\bx),\, f(\bx)\bigr)                 & = \left.\partial_{x_i}\partial_{x_j}K(\bx,\bx')\right|_{\bx'=\bx},                                  \\[0.5em]
      \Cov \bigl(\partial_{ij}f(\bx),\, \partial_{k}f(\bx)\bigr)     & = \left.\partial_{x_i}\partial_{x_j}\partial_{x'_k}K(\bx,\bx')\right|_{\bx'=\bx},                   \\[0.5em]
      \Cov \bigl(\partial_{ij}f(\bx),\, \partial_{k\ell}f(\bx)\bigr) & = \left.\partial_{x_i}\partial_{x_j}\partial_{x'_k}\partial_{x'_\ell}K(\bx,\bx')\right|_{\bx'=\bx}.
    \end{align*}
    We don't provide the explicit formula here due to its complexity, but it can
    be derived using standard results on conditional distributions of
    multivariate Gaussians.
    Giving closed-form expressions for the expected number of real ED critical points
    would be challenging in general, but it may be possible for some specific
    kernels due to their special structures.
    For example, when $K_{f}$ is a stationary kernel, i.e., $K_{f}(\bx,\bx') = \kappa
    (\bx - \bx')$ for some function $\kappa:\R^{n}\to\R$, $f(\bx)$ and
    $\nabla f(\bx)$ are independent random variables, which simplifies the expression
    of $p_{(f(\bx), \nabla f(\bx))}(0, \frac{\bxi - \bx}{\lambda})$ significantly.
  \end{remark}

  \subsection{Joint Gaussianity of the Augmented Process}
  \begin{lemma}
        \label{lem:augmented-gp} Let $f: \R^{n}\to \R$ be a centered Gaussian process
    with kernel $K_{f}\in C^{2}$.
    Define the augmented process $g(\bx) =
    \begin{bmatrix}
      f(\bx)        \\
      \nabla f(\bx)
    \end{bmatrix}
    \in \R^{n+1}$, where $\nabla f(\bx)$ is the gradient of $f$ at $\bx$.
    Then, $g: \R^{n}\to \R^{n+1}$ is a vector-valued Gaussian process with mean zero
    and kernel
    \[
      K_{g}(\bx,\bx') =
      \begin{bmatrix}
        K_{f}(\bx,\bx')             & (\nabla_{\bx'}K_{f}(\bx,\bx'))^{\top}           \\
        \nabla_{\bx}K_{f}(\bx,\bx') & \nabla_{\bx}\nabla_{\bx'}^{\top}K_{f}(\bx,\bx')
      \end{bmatrix}.
    \]
  \end{lemma}

  \begin{proof}
    We first show that $g(\bx)$ is a Gaussian process.
    We need to show that, if we fix any finite set of points $\bx^{(1)}, \dots ,
    \bx^{(m)}\in \R^{n}$, then any linear combination of the random variables
    $g(\bx^{(1)}), \dots, g(\bx^{(m)})$ is Gaussian.
    Let $\alpha^{(1)},\dots,\alpha^{(m)}\in\mathbb{R}$ and $\boldsymbol{\beta}^{(1)}
    ,\dots,\boldsymbol{\beta}^{(m)}\in\mathbb{R}^{n}$ be arbitrary.
    For each $i\in[m]$ and $j\in[n]$, define the difference quotient random variables
    \[
      D_{h}^{(i,j)}\;:=\; \frac{f\!\left(\bx^{(i)}+ h\,\mathbf{e}_{j}\right) - f\!\left(\bx^{(i)}\right)}{h}
      , \qquad h\neq 0,
    \]
    where $\mathbf{e}_{j}$ is the $j$-th standard basis vector in $\mathbb{R}^{n}$.
    Since $\bigl(f(\bx^{(i)}+ h\mathbf{e}_{j}), f(\bx^{(i)})\bigr)$ is jointly
    Gaussian, each $D_{h}^{(i,j)}$ is Gaussian, and hence any finite linear
    combination
    \[
      S_{h}\;:=\; \sum_{i=1}^{m}\alpha^{(i)}f(\bx^{(i)}) \;+\; \sum_{i=1}^{m}\sum
      _{j=1}^{n}\beta^{(i)}_{j}\, D_{h}^{(i,j)}
    \]
    is Gaussian for every fixed $h\neq 0$.

    Because $K_{f}$ is $C^{2}$, the mean-square derivatives of $f$ exist and coincide
    with the gradient components $\partial_{j}f(\bx^{(i)})$, and moreover
    \[
      D_{h}^{(i,j)}\xrightarrow[h\to 0]{}\partial_{j}f(\bx^{(i)}) \quad \text{in }
      L^{2}.
    \]
    Consequently, $S_{h}\to S$ in $L^{2}$, where
    \[
      S \;=\; \sum_{i=1}^{m}\alpha^{(i)}f(\bx^{(i)}) \;+\; \sum_{i=1}^{m}\sum_{j=1}
      ^{n}\beta^{(i)}_{j}\, \partial_{j}f(\bx^{(i)}).
    \]
    Since an $L^{2}$-limit of Gaussian random variables is Gaussian, it follows that
    $S$ is Gaussian.
    This proves that $\bigl(g(\bx^{(1)}),\dots,g(\bx^{(m)})\bigr)$ is jointly
    Gaussian, hence $g$ is a vector-valued Gaussian process.

    Next, we compute the mean and kernel of $g$.
    Since $f$ is centered, we have
    $\mathbb{E}[\nabla f(\bx)] = \nabla_{\bx}\mathbb{E}[f(\bx)] = 0$ due to the
    mean-square differentiability of $f$.
    Therefore, $\mathbb{E}[g(\bx)] = 0$.

    It remains to compute the covariance blocks.
    We have $\mathrm{Cov}(f(\bx),f(\bx')) = K_{f}(\bx,\bx')$ by definition.
    For the cross-covariance, using mean-square differentiability, we have
    \[
      \mathrm{Cov}\bigl(f(\bx),\nabla f(\bx')\bigr) = \mathbb{E}\bigl[f(\bx )\,(\nabla
      f(\bx'))^{\top}\bigr] = \nabla_{\bx'}\mathbb{E}\bigl[f(\bx )\,f(\bx')\bigr]
      = \nabla_{\bx'}K_{f}(\bx,\bx').
    \]
    which gives the top-right block as $(\nabla_{\bx'}K_{f}(\bx,\bx'))^{\top}$.
    Similarly,
    \[
      \mathrm{Cov}\bigl(\nabla f(\bx), f(\bx')\bigr) = \nabla_{\bx}K_{f}(\bx ,\bx
      '),
    \]
    and for the bottom-right block,
    \[
      \mathrm{Cov}\bigl(\nabla f(\bx), \nabla f(\bx')\bigr) = \mathbb{E}\bigl [\nabla
      _{\bx}f(\bx)\,(\nabla_{\bx'}f(\bx'))^{\top}\bigr] = \nabla_{\bx}\nabla_{\bx'}
      ^{\top}\mathbb{E}\bigl[f(\bx)\,f(\bx')\bigr] = \nabla_{\bx}\nabla_{\bx'}^{\top}
      K_{f}(\bx,\bx').
    \]
    Stacking these blocks yields the stated matrix-valued kernel $K_{g}$.
  \end{proof}

  \subsection{Operator Norm Bound for the Adjoint Matrix}
  \begin{lemma} \label{lem:adjoint-bound}
    For any $H \in \mathbb{R}^{n \times n}$ and $\lambda \in \mathbb{R}$,
    \[
    \|\operatorname{adj}(I_n+\lambda H)\|_{\mathrm{op}}
    \leq
    \bigl(1+|\lambda|\,\|H\|_{\mathrm{op}}\bigr)^{n-1}.
    \]
  \end{lemma}

  \begin{proof}
    Let $A := I_n+\lambda H$, and let $\sigma_1(A)\ge \cdots \ge \sigma_n(A)\ge 0$
    be its singular values. 
    The singular values of $\operatorname{adj}(A)$ are
    \[
      \left\{\prod_{j\neq i}\sigma_j(A):\, i=1,\dots,n\right\},
    \]
    hence
    \[
      \|\operatorname{adj}(A)\|_{\mathrm{op}}
      =
      \max_{1\le i\le n}\prod_{j\neq i}\sigma_j(A)
      \le
      \sigma_1(A)^{n-1}
      =
      \|A\|_{\mathrm{op}}^{\,n-1}.
    \]
    By the triangle inequality,
    \[
      \|A\|_{\mathrm{op}}
      =
      \|I_n+\lambda H\|_{\mathrm{op}}
      \le
      \|I_n\|_{\mathrm{op}}+|\lambda|\,\|H\|_{\mathrm{op}}
      =
      1+|\lambda|\,\|H\|_{\mathrm{op}}.
    \]
    Combining the two inequalities gives
    \[
      \|\operatorname{adj}(I_n+\lambda H)\|_{\mathrm{op}}
      \le
      \|I_n+\lambda H\|_{\mathrm{op}}^{\,n-1}
      \le
      \bigl(1+|\lambda|\,\|H\|_{\mathrm{op}}\bigr)^{n-1}.
    \]
    \qedhere
  \end{proof}

  \subsection{NNGP of Shallow Polynomial Networks}
  \begin{proposition}[NNGP of shallow polynomial networks]
    \label{prop:nngp-shallow-poly} Consider a two-layer network of width $m$ with
    element-wise polynomial activation $\sigma:\mathbb{R}\to\mathbb{R}$ given by
    $\sigma(z)=z^{d}$ for some integer $d\ge 1$.
    For an input $\bx\in\mathbb{R}^{n}$, the network output $f_{m}:\mathbb{R}^{n}
    \to\mathbb{R}$ is
    \[
      f_{m}(\bx) \;=\; \frac{1}{\sqrt{m}}\sum_{i=1}^{m}a_{i}\,\sigma(\bw_{i}^{\top}
      \bx + b_{i}) \;=\; \frac{1}{\sqrt{m}}\sum_{i=1}^{m}a_{i}\,(\bw_{i}^{\top}\bx
      + b_{i})^{d},
    \]
    where $\bw_{i}\in\mathbb{R}^{n}$ and $b_{i}$ are the hidden-layer weights and
    biases, respectively, and $a_{i}\in\mathbb{R}$ are the output-layer weights.
    Let the parameters $\{a_{i},\bw_{i},b_{i}\}_{i=1}^{m}$ be initialized i.i.d.\ from
    \[
      a_{i}\sim \mathcal{N}(0, 1),\quad b_{i}\sim \mathcal{N}(0, 1),\quad \bw_{i}
      \sim \mathcal{N}(0, \frac{1}{n}I_{n}).
    \]
    Then, as $m\to\infty$, the network output converges in distribution to a Gaussian
    process
    \[
      f(\bx) \sim \mathcal{GP}(0, K),
    \]
    where the kernel function is given by
    \begin{equation}
      K(\bx,\bx') := \sum_{s=0}^{\left\lfloor \frac{d}{2} \right\rfloor}c(s ,d) \left
      (\frac{1}{n}\|\bx\|_{2}^{2}+ 1\right)^{s}\left(\frac{1}{n}\| \bx' \|_{2}^{2}
      + 1\right)^{s}\left(\frac{1}{n}\bx^{\top}\bx' + 1\right )^{d-2s}, \label{eq:poly-nngp-kernel}
    \end{equation}
    with $c(s,d) = \dfrac{(d!)^{2}}{2^{2s}(s!)^{2}(d-2s)!}$.
    When $\bx = \bx'$, the kernel simplifies to
    \[
      K(\bx,\bx) = \operatorname{Var}[f(\bx)] = (2d - 1)!! \left(\frac{1}{n}\|\bx
      \|_{2}^{2}+ 1\right)^{d}.
    \]
  \end{proposition}
  \begin{proof}
    For fixed $\bx$, by the central limit theorem (CLT), as $m\to\infty$,
    \[
      f_{m}(\bx) \;\xrightarrow{d}\; \mathcal{N}\left(0, \mathbb{E}_{a,\bw,b}\left
      [a^{2}(\bw^{\top}\bx + b)^{2d}\right]\right).
    \]
    Since $a$, $\bw$, and $b$ are independent, we have
    \[
      \mathbb{E}_{a,\bw,b}\left[a^{2}(\bw^{\top}\bx + b)^{2d}\right] = \mathbb{E}
      _{a}[a^{2}] \cdot \mathbb{E}_{\bw,b}[(\bw^{\top}\bx + b)^{2d}] = \mathbb{E}
      _{\bw,b}[(\bw^{\top}\bx + b)^{2d}].
    \]
    Next, we compute $\mathbb{E}_{\bw,b}[(\bw^{\top}\bx + b)^{2d}]$.
    Note that $\bw^{\top}\bx + b$ is a Gaussian random variable with mean zero and
    variance $\frac{1}{n}\|\bx\|_{2}^{2}+ 1$.
    Therefore, using the moment formula for Gaussian random variables, we have
    \[
      \mathbb{E}_{\bw,b}[(\bw^{\top}\bx + b)^{2d}] = (2d - 1)!! \left(\frac{1}{n}
      \|\bx\|_{2}^{2}+ 1\right)^{d}.
    \]
    Thus, the variance of $f(\bx)$ is
    \[
      \operatorname{Var}[f(\bx)] = (2d - 1)!! \left(\frac{1}{n}\|\bx\|_{2}^{2}+ 1
      \right)^{d}.
    \]
    The kernel function $K(\bx,\bx')$ is given by
    \[
      K(\bx,\bx') = \operatorname{Cov}(f(\bx), f(\bx')) = \mathbb{E}_{a,\bw,b}[a^{2}
      (\bw^{\top}\bx + b)^{d}(\bw^{\top}\bx' + b)^{d}] = \mathbb{E}_{\bw,b}[(\bw^{\top}
      \bx + b)^{d}(\bw^{\top}\bx' + b)^{d}].
    \]
    Recall Isserlis' theorem \citep{isserlis_1918,janson_1997}, which states
    that for zero-mean Gaussian random variables $X_{1}, X_{2}, \ldots, X_{n}$,
    \[
      \mathbb{E}[X_{1}X_{2}\cdots X_{n}] = \sum_{p \in P_n^2}\prod_{\{i,j\} \in
      p}\mathbb{E}[X_{i}X_{j}],
    \]
    where $P_{n}^{2}$ is the set of all pairwise partitions of $\{1, 2, \ldots ,
    n\}$.
    If $(X,Y)$ is a bivariate zero-mean Gaussian vector, to find $\mathbb{E}[ X^{d}
    Y^{d}]$, we consider all pairwise partitions of the multiset
    \[
      \{\underbrace{X, X, \ldots, X}_{d \text{ times}}, \quad \underbrace{Y, Y,
      \ldots, Y}_{d \text{ times}}\}.
    \]
    The partition that contributes to $\mathbb{E}[X^{d}Y^{d}]$ must pair each
    $X$ with either another $X$ or a $Y$, and similarly for each $Y$.
    Let $s$ be the number of pairs of the form $(X,X)$ (or equivalently $(Y,Y )$)
    in the partition.
    Then there are $d - 2s$ pairs of the form $(X,Y)$.
    The number of such partitions is
    \[
      \binom{d}{d-2s}^{2}(d-2s)! \left((2s-1)!!\right)^{2},
    \]
    where $\binom{d}{d-2s}^{2}$ chooses which $X$'s and $Y$'s are paired
    together, $(d-2s)!$ counts the number of ways to pair the selected $X$'s with
    $Y$'s, and $\left((2s-1)!!\right)^{2}$ counts the number of ways to pair the
    remaining $X$'s and $Y$'s among themselves.
    Therefore, by Isserlis' theorem, we have
    \begin{align*}
      \mathbb{E}[X^{d}Y^{d}] & = \sum_{s=0}^{\left\lfloor \frac{d}{2} \right\rfloor}\binom{d}{d-2s}^{2}(d-2s)! \left((2s-1)!!\right)^{2}\mathbb{E}[X^{2}]^{s}\mathbb{E}[Y^{2}]^{s}\mathbb{E}[XY]^{d - 2s}                                                                                        \\
                             & = \sum_{s=0}^{\left\lfloor \frac{d}{2} \right\rfloor}\frac{(d!)^{2}}{\left((d-2s)!\right)^{2}\left((2s)!\right)^{2}}\cdot (d-2s)! \cdot \frac{\left((2s)!\right)^{2}}{2^{2s}\left(s!\right)^{2}}\mathbb{E}[X^{2}]^{s}\mathbb{E}[Y^{2}]^{s}\mathbb{E}[XY]^{d - 2s} \\
                             & = \sum_{s=0}^{\left\lfloor \frac{d}{2} \right\rfloor}\frac{(d!)^{2}}{2^{2s}(s!)^{2}(d-2s)!}\mathbb{E}[X^{2}]^{s}\mathbb{E}[Y^{2}]^{s}\mathbb{E}[XY]^{d - 2s}.
    \end{align*}
    Let $X = \bw^{\top}\bx + b$ and $Y = \bw^{\top}\bx' + b$.
    Then, $\mathbb{E}[X^{2}] = \frac{1}{n}\|\bx\|_{2}^{2}+ 1$, $\mathbb{E}[Y^{2}]
    = \frac{1}{n}\|\bx'\|_{2}^{2}+ 1$, and
    $\mathbb{E}[XY ] = \frac{1}{n}\bx^{\top}\bx' + 1$.
    Substituting these into the expression for $\mathbb{E}[X^{d}Y^{d}]$ yields the
    desired kernel formula.
  \end{proof}

  \begin{remark} \label{rmk:identities}
    We provide two identities about the coefficients $c(s,d)$ appearing in \cref{eq:poly-nngp-kernel}.
    These identites can be useful in simplifying certain expressions involving
    the kernel.
    \begin{align}
      \sum_{s=0}^{\left\lfloor \frac{d}{2} \right\rfloor}c_{s,d}        & = (2d - 1)!!,                    \\
      \sum_{s=0}^{\left\lfloor \frac{d}{2} \right\rfloor}c_{s,d}\cdot s & = \frac{d (d - 1)(2d - 3)!!}{2}.
    \end{align}
    The first identity follows from the expression of $K_{f}(\bx,\bx)$.
    To show the second identity, we note that
    \[
      \dfrac{c_{s,d}}{c_{s,d-1}}= \frac{d^{2}}{d-2s}.
    \]
    Therefore, we have
    \[
      \sum_{s=0}^{\left\lfloor \frac{d}{2} \right\rfloor}c_{s,d}\cdot (d-2s ) = d
      ^{2}\sum_{s=0}^{\left\lfloor \frac{d-1}{2} \right\rfloor}c_{s,d-1}.
    \]
    Re-arranging the above equation gives
    \[
      \sum_{s=0}^{\left\lfloor \frac{d}{2} \right\rfloor}c_{s,d}\cdot s = \frac{d}{2}
      \left( \sum_{s=0}^{\left\lfloor \frac{d}{2} \right\rfloor}c_{s,d}- d \sum_{s=0}
      ^{\left\lfloor \frac{d-1}{2} \right\rfloor}c_{s,d-1}\right) = \frac{d}{2}\left
      ( (2d - 1)!! - d (2d - 3)!! \right) = \frac{d (d - 1)(2d - 3)!!}{2}.
    \]
  \end{remark}

  \subsection{Derivatives of the NNGP Kernel}
  \begin{proposition}
    \label{prop:poly-nngp-kernel} Given the kernel function $K: \R^{n}\times \R^{n}
    \rightarrow \R$ defined as in \cref{eq:poly-nngp-kernel}, its first and
    second derivatives evaluated at the same point $\bx' = \bx$ are given by
    \begin{align*}
      \nabla_{\bx}K(\bx,\bx') \big|_{\bx'=\bx}                           & = \frac{d (2d - 1)!!}{n}\left(\frac{1}{n}\|\bx\|_{2}^{2}+ 1\right)^{d-1}\bx                                                                                                              \\
      \quad \nabla_{\bx}\nabla_{\bx'}^{\top}K(\bx,\bx') \big|_{\bx'=\bx} & = \dfrac{2d^2(d-1)(2d-3)!!}{n^2}\left(\frac{1}{n}\|\bx\|_{2}^{2}+ 1\right)^{d-2}\cdot \bx \bx^{\top}\\
      &\quad+\dfrac{d^2(2d-3)!!}{n}\left(\frac{1}{n}\|\bx\|_{2}^{2}+ 1\right)^{d-1}\cdot I_{n}.
    \end{align*}
  \end{proposition}
  \begin{proof}
    To compute the gradient block $\nabla_{\bx}K(\bx,\bx') \big|_{\bx'=\bx}$, we
    first compute the total derivative of $K(\bx,\bx)$ with respect to $\bx$.
    By the chain rule and symmetry of the kernel, we have
    \[
      \frac{d \, K(\bx, \bx)}{d \, \bx}= \left(\nabla_{\bx}K(\bx,\bx')+\nabla_{\bx'}
      K(\bx,\bx')\right) \Big |_{\bx'=\bx}= 2 \cdot \nabla_{\bx}K(\bx ,\bx') \Big
      |_{\bx'=\bx}.
    \]

    Therefore, we have
    \[
      \nabla_{\bx}K(\bx,\bx') \big|_{\bx'=\bx}= \frac{1}{2}\cdot \frac{d \, K(\bx,
      \bx)}{d \, \bx}= \frac{d (2d - 1)!!}{n}\left(\frac{1}{n}\|\bx \|_{2}^{2}+ 1
      \right)^{d-1}\bx.
    \]

    Next, we compute the Hessian block
    $\nabla_{\bx}\nabla_{\bx'}^{\top}K(\bx,\bx') \big|_{\bx'=\bx}$.
    $\nabla_{\bx}K_{f}(\bx,\bx')$ can be computed as
    \begin{align*}
      \nabla_{\bx}K_{f}(\bx,\bx') & = \sum_{s=0}^{\left\lfloor \frac{d}{2} \right\rfloor}c_{s,d}\left(\frac{1}{n}\|\bx'\|_{2}^{2}+ 1\right)^{s}\left( \frac{2s}{n}\left(\frac{1}{n}\|\bx\|_{2}^{2}+ 1\right)^{s-1}\left(\frac{1}{n}\bx^{\top}\bx' + 1\right)^{d-2s}\bx \right. \\
                                  & \quad \left. + \frac{d-2s}{n}\left(\frac{1}{n}\|\bx\|_{2}^{2}+ 1\right)^{s}\left(\frac{1}{n}\bx^{\top}\bx' + 1\right)^{d-2s-1}\bx' \right)                                                                                                 \\
                                  & = \sum_{s=0}^{\left\lfloor \frac{d}{2} \right\rfloor}c_{s,d}\left(\frac{1}{n}\|\bx'\|_{2}^{2}+ 1\right)^{s}\left(\frac{1}{n}\|\bx\|_{2}^{2}+ 1\right)^{s-1}\left(\frac{1}{n}\bx^{\top}\bx' + 1\right)^{d-2s-1}                             \\
                                  & \quad \cdot \left( \frac{2s}{n}\left(\frac{1}{n}\bx^{\top}\bx' + 1\right) \bx + \frac{d-2s}{n}\left(\frac{1}{n}\|\bx\|_{2}^{2}+ 1\right) \bx' \right).
    \end{align*}
    Furthermore, we can compute
    $\nabla_{\bx}\nabla_{\bx'}^{\top}K_{f}(\bx,\bx')$ by differentiating
    $\nabla_{\bx}K_{f}(\bx,\bx')$ with respect to $\bx'$.
    Applying the product rule and chain rule, we obtain
    \begin{align*}
      &\quad \nabla_{\bx}\nabla_{\bx'}^{\top}K_{f}(\bx,\bx') \\
      & = \sum_{s=0}^{\left\lfloor \frac{d}{2} \right\rfloor}c_{s,d}\left(\frac{1}{n}\|\bx\|_{2}^{2}+ 1\right)^{s-1}\left(\frac{1}{n}\|\bx'\|_{2}^{2}+ 1\right)^{s-1}\left(\frac{1}{n}\bx^{\top}\bx' + 1\right)^{d-2s-2}                                      \\
                                                      & \quad \cdot \left[\frac{2s}{n}\left(\frac{1}{n}\bx^{\top}\bx' + 1\right) \left(\frac{2s}{n}\left(\frac{1}{n}\bx^{\top}\bx' + 1\right) \bx + \frac{d-2s}{n}\left(\frac{1}{n}\|\bx\|_{2}^{2}+ 1\right) \bx' \right){\bx'}^{\top}\right.                 \\
                                                      & \quad + \frac{d-2s-1}{n}\left(\frac{1}{n}\|\bx'\|_{2}^{2}+ 1\right) \left( \frac{2s}{n}\left(\frac{1}{n}\bx^{\top}\bx' + 1\right) \bx + \frac{d-2s}{n}\left(\frac{1}{n}\|\bx\|_{2}^{2}+ 1\right) \bx' \right) \bx^{\top}                              \\
                                                      & \quad + \left. \left(\frac{1}{n}\bx^{\top}\bx' + 1\right) \left(\frac{1}{n}\|\bx'\|_{2}^{2}+ 1\right) \left(\frac{2s}{n}\cdot \frac{1}{n}\cdot \bx{\bx'}^{\top}+ \frac{d-2s}{n}\left(\frac{1}{n}\|\bx'\|_{2}^{2}+ 1\right) \cdot I_{n}\right)\right].
    \end{align*}
    Evaluating the above expression at $\bx' = \bx$ and simplifying using the
    identities for $c_{s,d}$ yields
    \begin{align}
      & \quad \nabla_{\bx}\nabla_{\bx'}^{\top}K_{f}(\bx,\bx') \big|_{\bx'=\bx} \nonumber \\
      & = \sum_{s=0}^{\left\lfloor \frac{d}{2} \right\rfloor}c_{s,d}\left(\frac{d(d-1) + 2s}{n^{2}}\cdot \left(\frac{1}{n}\|\bx\|_{2}^{2}+ 1\right)^{d-2}\cdot \bx \bx^{\top}+ \frac{d-2s}{n}\cdot \left(\frac{1}{n}\|\bx\|_{2}^{2}+ 1\right)^{d-1}\cdot I_{n}\right) \nonumber \\
                                                                             & = \dfrac{2d^2(d-1)(2d-3)!!}{n^2}\left(\frac{1}{n}\|\bx\|_{2}^{2}+ 1\right)^{d-2}\cdot \bx \bx^{\top}+ \dfrac{d^2(2d-3)!!}{n}\left(\frac{1}{n}\|\bx\|_{2}^{2}+ 1\right)^{d-1}\cdot I_{n}.
    \end{align}
  \end{proof}

\section{Parameter Discriminant}

  \subsection{Proof of \cref{prop:quadratic-ed-degree}}
  \begingroup
  \renewcommand{\thetheorem}{\ref{prop:quadratic-ed-degree}}
  \renewcommand{\theHtheorem}{restated.\ref{prop:quadratic-ed-degree}}
  %
  \propquadraticeddegree* \endgroup

  \begin{proof}
    The hypersurface $\mathcal{V}$ is the union of two hyperplanes $H_{1}= \{ \bx
    : \ell_{1}(\bx)=0\}$ and $H_{2}= \{\bx : \ell_{2}(\bx)=0\}$. For a generic
    data point $\bu$, the ED critical points on $\mathcal{V}$ are simply the orthogonal
    projections of $\bu$ onto $H_{1}$ and $H_{2}$. Let these projections be
    $p_{1}(\bu)$ and $p_{2}(\bu)$. Thus, there are exactly two critical points. These
    two solutions collide precisely when $p_{1}(\bu ) = p_{2}(\bu)$. Let this
    common point be $\bx^{*}$. Since $\bx^{*}\in H_{1}$ and $\bx^{*}\in H_{2}$,
    it must lie in the intersection $H_{1}\cap H_{2}$. The condition that $\bx^{*}$
    is the orthogonal projection of $\bu$ onto $H_{1}$ implies that the vector $\bu
    - \bx^{*}$ is normal to $H_{1}$. Similarly, $\bu - \bx^{*}$ must be normal to
    $H_{2}$.

    Since the normals of $H_{1}$ and $H_{2}$ are not parallel by assumption, the
    only vector simultaneously normal to two transverse hyperplanes is the zero
    vector. Therefore, we must have $\bu - \bx^{*}= \mathbf{0}$, which implies $\bu
    = \bx^{*}$. Thus, the critical points coincide if and only if $\bu \in H_{1}\cap
    H_{2}\subseteq \mathcal{V}$.
  \end{proof}

  \subsection{Proof of \cref{thm:ed-degree-conics-n-variables}}
  \begingroup
  \renewcommand{\thetheorem}{\ref{thm:ed-degree-conics-n-variables}}
  \renewcommand{\theHtheorem}{restated.\ref{thm:ed-degree-conics-n-variables}}
  %
  \eddegreeconicsnvariables* \endgroup

  \begin{proof}
    The ED degree of $\mathcal{V}_{c,c'}^{\bt}$ is the number of complex critical
    points of the squared distance function $d_{\bu}(\bx)$ restricted to the regular
    part $(\mathcal{V}_{c,c'}^{\bt})_{\text{reg}}$ for a generic data point $\bu
    \in \mathbb{C}^{n}$. The critical points satisfy the Lagrange multiplier
    condition, which states that the vector $\bu-\bx$ is normal to the tangent space
    of $\mathcal{V}_{c,c'}^{\bt}$ at $\bx$. Algebraically, this implies $\bu - \bx
    = \lambda \nabla B_{\bt}(\bx)$ for some scalar $\lambda \in \mathbb{C}^{*}$.
    Substituting the gradient
    $\nabla B_{\bt}(\bx) = 2A_{\bt}\bx + \boldsymbol{b}_{\bt}$, we obtain the system:
    \[
      \bu - \bx = \lambda (2A_{\bt}\bx + \boldsymbol{b}_{\bt}).
    \]
    Rearranging terms yields a linear system for $\bx$ parameterized by
    $\lambda$:
    \begin{equation}
      \label{eq:linear_system}(I + 2\lambda A_{\bt})\bx = \bu - \lambda \boldsymbol
      {b}_{\bt}.
    \end{equation}
    For a generic $\bu$, the matrix $(I + 2\lambda A_{\bt})$ is invertible for
    all solutions $\lambda$. We solve for $\bx$ by changing the coordinate system
    (rotating) to the eigenbasis of~$A_{\bt}$.

    Since $A_{\bt}$ is symmetric, it is orthogonally diagonalizable over
    $\mathbb{R}$. Let $A_{\bt}= U \Sigma U^{\top}$ be its spectral decomposition,
    where $U$ is an orthogonal matrix ($U^{\top}U = I$) and
    $\Sigma = \text{diag}(\sigma_{1}, \dots, \sigma_{n})$ contains the
    eigenvalues. We rotate the coordinate system to the eigenbasis of $A_{\bt}$ by
    defining
    \[
      \tilde{\bx}= U^{\top}\bx, \quad \tilde{\bu}= U^{\top}\bu, \quad \tilde{\boldsymbol{b}}
      = U^{\top}\boldsymbol{b}_{\bt}.
    \]
    In these coordinates $A_{\bt}$ is diagonal, and the linear system decouples
    into scalar equations for each component~$i \in [n]$:
    \[
      (1 + 2\lambda \sigma_{i})\tilde{x}_{i}= \tilde{u}_{i}- \lambda \tilde{b}_{i}
      \implies \tilde{x}_{i}(\lambda) = \frac{\tilde{u}_{i}- \lambda \tilde{b}_{i}}{1
      + 2\lambda \sigma_{i}}.
    \]
    Substituting these parameterized coordinates back into the defining equation
    $B_{\bt}(\bx) = \tilde{\bx}^{\top}\Sigma \tilde{\bx}+ \tilde{\boldsymbol{b}}^{\top}
    \tilde{\bx}+ c_{\bt}= 0$ yields a univariate rational equation in $\lambda$,
    which we denote by $E (\lambda) = 0$:
    \begin{equation}
      \label{eq:rational_E}E(\lambda) := \sum_{i=1}^{n}\sigma_{i}\left( \frac{\tilde{u}_{i}-
      \lambda \tilde{b}_{i}}{1 + 2\lambda \sigma_{i}}\right)^{2}+ \sum_{i=1}^{n}\tilde
      {b}_{i}\left( \frac{\tilde{u}_{i}- \lambda \tilde{b}_{i}}{1 + 2\lambda \sigma_{i}}
      \right) + c_{\bt}= 0.
    \end{equation}
    Let $r$ be the number of distinct non-zero eigenvalues of $A_{\bt}$, denoted
    $\mu_{1}, \dots, \mu_{r}$. The common denominator of $E(\lambda)$ is $D(\lambda
    ) = \prod_{j=1}^{r}(1 + 2\lambda \mu_{j})^{2}$, which has degree $2r$. The ED
    degree of $\mathcal{V}_{c,c'}^{\bt}$ is given by the degree of the polynomial
    $P(\lambda)$, which is obtained from $E(\lambda)$ by clearing the
    denominators:
    \[
      P(\lambda) = E(\lambda) \cdot D(\lambda).
    \]
    Note that for generic $\bu$, the correspondence between roots $\lambda$ and
    critical points $\bx$ is bijective, as $D(\lambda)$ does not vanish at the roots
    of $P(\lambda)$ and the linear system uniquely determines $\bx$.

    To determine $\deg(P)$, we analyze the behavior of the rational function at
    infinity. Since $E(\lambda) ={P(\lambda)}/{D(\lambda)}$, the function scales
    as $E(\lambda) \sim \lambda^{\deg(P) - \deg(D)}$ as $\lambda \to \infty$. Letting
    $\text{ord}_{\infty}(E)$ denote this exponent of growth or decay, we obtain
    the relation:
    \[
      \deg(P) - \deg(D) = \text{ord}_{\infty}(E) \implies \text{EDdegree}(\mathcal{V}
      _{c,c'}^{\bt}) = 2r + \text{ord}_{\infty}(E).
    \]

    To determine the asymptotic behavior of $E(\lambda)$, we analyze the rank of
    the augmented matrix $M_{\bt}$ associated with the quadric. We partition the
    indices $\{1, \dots, n\}$ into two sets: $I_{\text{range}}= \{i \mid \sigma_{i}
    \neq 0\}$ and $I_{\text{ker}}= \{i \mid \sigma_{i}= 0\}$. The augmented
    matrix is given by:
    \[
      M_{\bt}=
      \begin{pmatrix}
        A_{\bt}                                & \frac{1}{2}\boldsymbol{b}_{\bt} \\
        \frac{1}{2}\boldsymbol{b}_{\bt}^{\top} & c_{\bt}
      \end{pmatrix}.
    \]

    First, we apply the orthogonal change of basis $U$ that diagonalizes $A_{\bt}$.
    Second, we permute the coordinates to group the indices into the two sets $I_{\text{range}}$
    and $I_{\text{ker}}$. This transformation yields the partitioned matrix:

    \[
      M'_{\bt}= \left(
      \begin{array}{ccc}
        \Sigma_{\text{range}}                                 & 0                                                   & \frac{1}{2}\tilde{\boldsymbol{b}}_{\text{range}} \\
        0                                                     & 0                                                   & \frac{1}{2}\tilde{\boldsymbol{b}}_{\text{ker}}   \\
        \frac{1}{2}\tilde{\boldsymbol{b}}_{\text{range}}^\top & \frac{1}{2}\tilde{\boldsymbol{b}}_{\text{ker}}^\top & c_{\bt}
      \end{array}
      \right).
    \]

    We perform symmetric block Gaussian elimination to eliminate the couplings between
    the range components and the affine part. Specifically, we subtract the projection
    of the first block-row onto the last row using the invertible pivot $\Sigma_{\text{range}}$.
    This operation zeroes out $\frac{1}{2}\tilde{\boldsymbol{b}}_{\text{range}}$
    and transforms the constant $c_{\bt}$ into the Schur complement $S$:
    \[
      S = c_{\bt}- (\frac{1}{2}\tilde{\boldsymbol{b}}_{\text{range}}^{\top}) \Sigma
      _{\text{range}}^{-1}(\frac{1}{2}\tilde{\boldsymbol{b}}_{\text{range}}) = c_{\bt}
      - \sum_{j \in I_{\text{range}}}\frac{\tilde{b}_{j}^{2}}{4\sigma_{j}}.
    \]The resulting matrix is block-diagonal, decoupling the range variables from
    the kernel and constant terms:
    \[
      M_{\bt}\cong
      \begin{pmatrix}
        \Sigma_{\text{range}} & 0                                                     & 0                                              \\
        0                     & 0                                                     & \frac{1}{2}\tilde{\boldsymbol{b}}_{\text{ker}} \\
        0                     & \frac{1}{2}\tilde{\boldsymbol{b}}_{\text{ker}}^{\top} & S
      \end{pmatrix}.
    \]
    The rank of $M_{\bt}$ is the sum of
    $\text{rank}(\Sigma_{\text{range}}) = \rank A_{\bt}$ and the rank of the
    residual block $K =
    \begin{pmatrix}
      0                                            & \tilde{\boldsymbol{b}}_{\text{ker}}/2 \\
      \tilde{\boldsymbol{b}}_{\text{ker}}^{\top}/2 & S
    \end{pmatrix}$.

    We now analyze the cases.

    \textbf{Case 1:} $\rank M_{\bt}= \rank A_{\bt}+ 1$.

    Since $\rank M_{\bt}= \rank A_{\bt}+ \rank K$, this condition implies that
    $\rank K = 1$. This, in turn, forces $\tilde{\boldsymbol{b}}_{\text{ker}}$ to
    be zero. Indeed, if $\tilde{\boldsymbol{b}}_{\text{ker}}$ were not zero, we
    could pick a non-zero component $(\tilde{\boldsymbol{b}}_{\text{ker}})_{i}$.
    The $2 \times 2$ submatrix formed by this component and $S$ would have determinant
    $-(\tilde{\boldsymbol{b}}_{\text{ker}})_{i}^{2}/4 \neq 0$, implying
    $\rank (K) \ge 2$. This contradicts our assumption. Therefore, $\tilde{\boldsymbol{b}}
    _{\text{ker}}= 0$. With $\tilde{\boldsymbol{b}}_{\text{ker}}= 0$, the matrix
    becomes diagonal $K = \text{diag}(0, \dots, 0, S)$. Its rank is 1 if and
    only if $S \neq 0$.

    We use these facts to find the degree of $E(\lambda)$. The terms in the sum (\ref{eq:rational_E})
    behave differently depending on whether the eigenvalue $\sigma_{i}$ is zero or
    not. We formally split the sum into two parts:
    \[
      E(\lambda) = E_{\text{range}}(\lambda) + E_{\text{ker}}(\lambda) + c_{\bt},
    \]
    where $E_{\text{range}}$ sums over indices with $\sigma_{i}\neq 0$, and
    $E_{\text{ker}}$ sums over indices with $\sigma_{i}= 0$:
    \[
      E_{\text{range}}(\lambda) := \sum_{i \in I_{\text{range}}}\left[ \sigma_{i}
      \left( \frac{\tilde{u}_{i}- \lambda \tilde{b}_{i}}{1 + 2\lambda \sigma_{i}}
      \right)^{2}+ \tilde{b}_{i}\left( \frac{\tilde{u}_{i}- \lambda \tilde{b}_{i}}{1
      + 2\lambda \sigma_{i}}\right) \right],
    \]
    \[
      E_{\text{ker}}(\lambda) := \sum_{i \in I_{\text{ker}}}\tilde{b}_{i}( \tilde
      {u}_{i}- \lambda \tilde{b}_{i}).
    \]

    Since $\tilde{\boldsymbol{b}}_{\text{ker}}= 0$, the coefficients in the kernel
    sum are all zero, so $E_{\text{ker}}(\lambda) = 0$. We just need the limit of
    the remaining parts as $\lambda \to \infty$:
    \[
      \lim_{\lambda \to \infty}E(\lambda) = \lim_{\lambda \to \infty}(E_{\text{range}}
      (\lambda) + c_{\bt}) = c_{\bt}-\sum_{j\in I_{\text{range}}}\frac{\tilde{b}_{j}^{2}}{4\sigma_{j}}
      = S.
    \]
    Because $S \neq 0$, $E(\lambda)$ approaches a non-zero constant. This means
    the numerator polynomial $P(\lambda)$ has the same degree as the denominator
    $D(\lambda)$:
    \[
      \text{EDdegree}(\mathcal{V}_{c,c'}^{\bt}) = \text{deg}(D) = 2r.
    \]

    \textbf{Case 2:} $\rank M_{\bt}= \rank A_{\bt}+ 2$.

    In this case, we get that $\rank K = 2$. This implies $\tilde{\boldsymbol{b}}
    _{\text{ker}}\neq 0$, as vanishing kernel components would restrict the rank
    of $K$ to at most 1. We analyze the behavior of
    $E (\lambda) = E_{\text{range}}(\lambda) + E_{\text{ker}}(\lambda) + c_{\bt}$
    as $\lambda \to \infty$. While the range terms converge to a constant, the
    non-zero kernel vector grows linearly:
    \[
      E(\lambda) = -\lambda \sum_{j \in I_{\text{ker}}}\tilde{b}_{j}^{2}+ O (1).
    \]

    Since the coefficient
    $-\sum \tilde{b}_{j}^{2}= -\|\tilde{\boldsymbol{b}}_{\text{ker}}\|^{2}$ is
    non-zero, this linear term dominates the behavior at infinity, implying
    $\text{ord}_{\infty}(E) = 1$. Thus, the degree of the numerator is one
    higher than the denominator:
    \[
      \text{EDdegree}(\mathcal{V}_{c,c'}^{\bt}) = \deg(D) + 1 = 2r + 1.
    \]

    \textbf{Case 3:} $\rank M_{\bt}= \rank A_{\bt}$.

    The condition $\rank M_{\bt}= \rank A_{\bt}+ \rank K$ implies that the residual
    block $K$ must be the zero matrix. This forces both
    $\tilde{\boldsymbol{b}}_{\text{ker}}= 0$ and $S = 0$. We analyze the asymptotic
    behavior of the range terms $E(\lambda) = c_{\bt}+ \sum_{i \in I_{\text{range}}}
    Q_{i}(T_{i}(\lambda))$, where $Q_{i}(t) = \sigma_{i}t^{2}+ \tilde{b}_{i}t$
    and
    $T_{i}(\lambda) = \frac{\tilde{u}_{i}- \lambda \tilde{b}_{i}}{1 + 2\lambda \sigma_{i}}$.

    We rewrite the numerator of $T_{i}(\lambda)$ to match the denominator:
    \[
      T_{i}(\lambda) = \frac{-\frac{\tilde{b}_i}{2\sigma_i}(1 + 2\lambda \sigma_{i})
      + \frac{\tilde{b}_i}{2\sigma_i}+ \tilde{u}_{i}}{1 + 2\lambda \sigma_{i}}= -
      \frac{\tilde{b}_{i}}{2\sigma_{i}}+ \frac{C_{i}}{1 + 2\lambda \sigma_{i}},
    \]
    where $C_{i}= \frac{\tilde{b}_{i}}{2\sigma_{i}}+ \tilde{u}_{i}$ is a constant.
    Let $t^{*}_{i}= -\frac{\tilde{b}_{i}}{2\sigma_{i}}$. Since the denominator $(
    1 + 2\lambda \sigma_{i})$ grows linearly with $\lambda$, the second term
    decays as $\lambda^{-1}$:
    \[
      T_{i}(\lambda) = t^{*}_{i}+ O(\lambda^{-1}).
    \]
    We substitute this into the quadratic $Q_{i}$. Note that $t^{*}_{i}$ is a
    critical point of $Q_{i}$, satisfying $Q_{i}'(t^{*}_{i}) = 2\sigma_{i}t^{*}_{i}
    + \tilde{b}_{i}= 0$. Using the Taylor expansion around $t^{*}_{i}$:
    \[
      Q_{i}(T_{i}(\lambda)) = Q_{i}(t^{*}_{i}) + \underbrace{Q_i'(t^*_i)}_{0}\cdot
      O(\lambda^{-1}) + O(\lambda^{-2}).
    \]

    Summing these terms:
    \[
      E(\lambda) = c_{\bt}+ \sum_{i \in I_{\text{range}}}Q_{i}(t^{*}_{i}) + O(\lambda
      ^{-2}) = S + O(\lambda^{-2}).
    \]
    Since $S=0$, the rational function decays as $O(\lambda^{-2})$. This implies
    that the degree of the numerator $P(\lambda)$ is exactly two less than the denominator:
    \[
      \text{EDdegree}(\mathcal{V}_{c,c'}^{\bt}) = \deg(D) - 2 = 2r - 2.
    \]
  \end{proof}

  \subsection{Proof of \cref{cor:conic-generic}}
  \begingroup
  \renewcommand{\thetheorem}{\ref{cor:conic-generic}}
  \renewcommand{\theHtheorem}{restated.\ref{cor:conic-generic}}
  %
  \conicgeneric* \endgroup

  \begin{proof}
    We examine the three cases from Theorem \ref{thm:ed-degree-conics-n-variables}.
    First, consider the case when $\text{rank}(M_{\bt}) = \text{rank}(A_{\bt}) +
    2$. The degree is given by $2r+1$. Since this value is always odd, it cannot
    equal the even integer $2n$. Next, consider the case when
    $\text{rank}(M_{\bt}) = \text{rank}(A_{\bt})$. The degree is given by $2r -2$.
    Because the number of distinct eigenvalues $r$ cannot exceed the dimension
    $n$, the maximum possible degree in this case is $2n-2$, which is strictly
    less than~$2n$. Therefore, the degree can equal $2n$ only in the case when
    $\text{rank}(M_{\bt}) = \text{rank}(A_{\bt}) + 1$. In this scenario, the
    degree is $2r$. Setting $2r = 2n$ implies that $r=n$. This means that the
    matrix $A_{\bt}$ must have exactly $n$ distinct non-zero eigenvalues. Hence,
    $A_{\bt}$ has full rank with no zero or repeated eigenvalues. Conversely, if
    $A_{\bt}$ is non-singular with distinct eigenvalues and
    $\mathcal{V}_{c,c'}^{\bt}$ is non-singular, we are in the case with $r=n$,
    yielding the ED degree of $2n$.
  \end{proof}

\section{ReLU networks}\label{sec:relu}
Although our results primarily concern algebraic classifiers, the metric-projection
formulation of Section~\ref{sec:verification-as-algebraic-problem} also gives a natural
geometric interpretation of ReLU networks. By
Proposition~\ref{prop:relaxation-equivalence}, which applies to arbitrary continuous
classifiers, the robustness margin of a test point $\bxi$ with a unique predicted class
may be computed by projecting $\bxi$ onto the relaxed boundary
$\V_{c,c'}^{\bt}$. For a ReLU network, the input space $\R^n$ is partitioned into
activation regions, and on each such region the map $f_{\bt}$ is affine linear. Thus, for
every full-dimensional activation region $P$ and every pair of classes $c \neq c'$, there
exist $a_{P,c,c'} \in \R^n$ and $b_{P,c,c'} \in \R$ such that
$$
f_{\bt,c'}(\bx)-f_{\bt,c}(\bx)=a_{P,c,c'}^\top \bx+b_{P,c,c'}
\qquad \text{for all } \bx \in P.
$$
Hence
$$
\V_{c,c'}^{\bt}\cap P
=
P\cap \{\bx \in \R^n : a_{P,c,c'}^\top \bx+b_{P,c,c'}=0\}.
$$
In particular, on each activation region, the relaxed boundary is an affine linear
variety, and the corresponding local metric-projection problem has ED degree $1$.

Thus the ReLU case separates local and global complexity in a particularly transparent
way. Locally, every piece of the relaxed boundary has ED degree 1, so there is no
nontrivial algebraic complexity on a fixed activation region. Globally, however, the
relaxed boundary is a union of such pieces across all activation regions, glued along
lower-dimensional faces; the difficulty of verification is therefore governed by the
combinatorics of this polyhedral stratification, namely by the number of relevant regions
and how their boundary pieces fit together. 
For algebraic classifiers, by contrast, the
relaxed boundary is already genuinely nonlinear on a single piece, and the ED degree
measures this intrinsic local complexity.

This separation of local and global complexity also suggests the correct analog of the ED
discriminant in the ReLU setting. On each fixed activation region, the relaxed boundary is an affine linear variety, so the local ED problem has a unique critical point and hence
empty ED discriminant. Thus any change in the verification landscape must come from the
global piecewise-linear geometry, not from local collisions of critical points: it occurs
when the nearest point switches from one affine stratum to another, or moves onto a
lower-dimensional face where strata meet. In this sense, the polyhedral separation of the
activation regions plays, for ReLU networks, the same global role that the ED discriminant
plays for algebraic decision boundaries. This also points to a natural intermediate class
for future study, namely piecewise-algebraic classifiers such as networks with polynomial
spline activations, where one expects both phenomena to coexist: nontrivial local ED
geometry on each piece together with global complexity coming from how the pieces fit
together.

\end{document}

%% file: math_commands.tex
\usepackage{amsmath,amsfonts,bm, amsthm}

\def\eqref#1{equation~\ref{#1}}

\def\1{\bm{1}}

\def\vtheta{{\bm{\theta}}}

\DeclareMathAlphabet{\mathsfit}{\encodingdefault}{\sfdefault}{m}{sl}
\SetMathAlphabet{\mathsfit}{bold}{\encodingdefault}{\sfdefault}{bx}{n}

\newcommand{\R}{\mathbb{R}}

\newcommand{\Cov}{\mathrm{Cov}}

\DeclareMathOperator*{\argmax}{arg\,max}

\newcommand{\rank}{\operatorname{rank}}

\newcommand{\bb}{\boldsymbol{b}}

\newcommand{\bx}{\boldsymbol{x}}
\newcommand{\bw}{\boldsymbol{w}}
\newcommand{\bu}{\boldsymbol{u}}
\newcommand{\bz}{\boldsymbol{z}}
\newcommand{\bt}{\boldsymbol{\theta}}
\newcommand{\bxi}{\boldsymbol{\xi}}
\newcommand{\V}{\mathcal{V}}
\newcommand{\B}{\mathcal{B}}